\documentclass{article} \usepackage{iclr2027_conference,times}

\usepackage{hyperref}
\usepackage{url}

\usepackage{mymacro}
\usepackage{mycommand}
\usepackage{enumitem} 

\allowdisplaybreaks[3]

\title{Reinforcement Learning for Hierarchical Reasoning Rewards: \\ Minimax-Optimal Rates with Transformers}

\author{Naoki Nishikawa \\
The University of Tokyo, RIKEN AIP \\
\texttt{nishikawa-naoki259@g.ecc.u-tokyo.ac.jp}
\And
Taiji Suzuki \\
The University of Tokyo, RIKEN AIP \\
\texttt{taiji@mist.i.u-tokyo.ac.jp}
}

\iclrfinalcopy \begin{document}

\maketitle
\lhead{Preprint.} 
\begin{abstract}
    Reinforcement learning (RL) has become a standard tool for post-training language models on reasoning tasks,
    where the policy is updated by reward feedback while exploring the space of responses.
    Despite its empirical success, theoretical understanding of RL post-training remains limited,
    in particular of why on-policy exploration combined with a neural reward model is effective.
    In this paper, we address this question by modeling the reward as a hierarchical function on the response space:
    the reward consists of infinitely many local components, each of which becomes relevant only after the preceding ones have been resolved.
    We show that a natural Transformer-based actor--critic algorithm,
    which alternates between sampling from the current KL-regularized policy, fitting a Transformer critic to the observed rewards, and updating the policy,
    achieves the minimax optimal rates in the query budget and in the regularization strength up to logarithmic factors, and is minimax optimal for a fixed number of prompts.
    In contrast, we prove that sampling from the fixed reference distribution, as in offline reward modeling,
    can limit regret decay to a logarithmic rate.
    These results show that on-policy exploration progressively zooms in on the region where the reward is concentrated,
    and quantify its benefit for RL post-training.
\end{abstract}

\section{Introduction}
\label{sec:intro}
Reinforcement learning (RL) has emerged as a central ingredient in post-training large language models (LLMs) for complex reasoning. Recent systems have demonstrated substantial gains on mathematical reasoning, coding, and other verifiable tasks by optimizing model-generated responses against outcome-based rewards \citep{shao2024deepseekmath,guo2025deepseek,kimi2025k15}. In contrast to supervised fine-tuning from a fixed collection of demonstrations, RL post-training repeatedly generates responses from the current policy, scores them with a reward model or verifier, and updates the policy with the resulting feedback. This creates a feedback loop: as the model improves, the distribution of responses from which subsequent training data are collected changes as well.

Despite the empirical success of this paradigm, the statistical role of such adaptive data collection is not yet fully understood. Recent theoretical work has identified several benefits of exploration in LLM reasoning and alignment, through the dynamics of reasoning, the generalization of the learned policy, or the implicit exploration of a KL-regularized policy~\citep{foster2025good,kim2025metastable,wang2026planning,wang2026greedy}, and suggests that \emph{where} the training responses come from can be as important as the learning rule applied to them. Yet a basic statistical question remains open: \emph{when and why does adapting the response distribution improve the rate of learning}, in particular for the reasoning tasks that have driven recent progress in RL post-training, in which a correct solution must clear several difficult steps in succession?

To study this question, we formulate LLM post-training as KL-regularized reward maximization over a response space and introduce a hierarchical model of reasoning rewards. The motivation is that complex reasoning rarely reveals all of its difficulties at once: only after one hard step has been cleared does it become apparent that another, again nontrivial, problem must be solved next. In a mathematical proof, for example, one may first need to find the right theorem to invoke; only once this is settled does it become clear which transformation is needed to apply it, and only then do the finer algebraic and logical decisions come into view. We model this by a hierarchy of smooth reward components supported on nested regions of the response space, with progressively more localized reward structure appearing closer to optimal responses.

This hierarchical structure creates a sharp statistical distinction between fixed and adaptive response collection. Under the uniform reference distribution, the nested regions carrying deeper reward components occupy exponentially small fractions of the response space, and observations outside a region carry no information about the component inside it: one cannot tell that a deeper step is needed before getting that far. Learning a deep component is therefore a rare-event problem. On-policy RL removes this barrier: once a coarse reward component has been learned, the policy shifts its mass toward the corresponding promising region, so that the next batch is collected precisely where the next, finer component becomes visible. Repeating this step progressively localizes the data distribution around high-reward responses.

\paragraph{Contributions.}
We make this intuition precise through a minimax analysis. Let \(n\) denote the total number of reward observations, \(M\) the number of prompts, \(\rho\ge0\) the strength of the KL regularization, \(\gamma>1\) the rate at which deeper reward components decay, and \(\alpha\) the smoothness of the components in a \(d\)-dimensional response embedding space. Our contributions are threefold.
\begin{enumerate}[leftmargin=*]
\item \textbf{Hierarchical reward model and minimax characterization.}
We introduce the nested hierarchical reward model and characterize the minimax regret over all (adaptive) algorithms. For fixed $M$, when the regularization is absent or weak, the minimax rate is
$\widetilde{\Theta}\qty(n^{-(\gamma-1)/(2\gamma+1)})$:
the difficulty is how deep in the hierarchy the learner can explore. In the large-budget regime, the rate is
$\widetilde{\Theta}\qty(\rho\,(n/n_0)^{-2\alpha/(2\alpha+d)})$ with $n_0:=M\rho^{-(2+1/\gamma)}$:
the difficulty is the local nonparametric estimation of the reward near the optimum. The upper bounds match the lower bounds up to logarithmic factors and a factor polynomial in $M$, and up to logarithmic factors for fixed $M$.
\item \textbf{A minimax-optimal Transformer actor--critic.}
We give an on-policy actor--critic whose critic is a causal Transformer, refitted at every iteration to fresh on-policy observations, and whose actor is the resulting Gibbs policy. Such a procedure, motivated by standard RL post-training, attains both rates above up to logarithmic factors and is thus minimax optimal for fixed $M$.
\item \textbf{A statistical separation from fixed response collection.}
We establish lower bounds for algorithms whose responses are drawn from the fixed reference distribution $\mu$, as in offline reward modeling, while the estimator and the output policy are unrestricted. Because deeper reward components are exponentially rare under fixed sampling, the regret can decay only logarithmically in $n/M$ when the regularization is weak. The benefit of on-policy RL is therefore statistical, not merely an optimization effect of the policy update.
\end{enumerate}

\paragraph{Related work.}
A first line of work adapts the distribution of tasks or prompts during training: easy-to-hard schedules \citep{li2023understanding,parashar2026curriculum}, adaptive prompt selection from the current performance \citep{rajaraman2026autocurriculum}, and depth- or hint-based curricula \citep{bu2026provable}.
The last makes RL post-training with outcome-only rewards polynomial instead of exponential on a discrete reasoning tree.
There the adaptivity is over prompts or tasks and the intermediate signal is supplied by the curriculum; here the prompt distribution is fixed and the learner adapts only its response distribution, which is possible because the reward is graded rather than outcome-only.

Technically, the closest algorithms are those for continuum-armed bandits, which partition the action space hierarchically and refine it around the maximizer \citep{kleinberg2008multi,bubeck2011x,munos2011optimistic,valko2013stochastic}.
We depart from that line in three ways, all motivated by LLM reasoning: the reward is hierarchical, with a new difficulty appearing at every level; the target is a diverse policy, the Gibbs policy at temperature $\rho>0$, rather than the maximizer alone; and the optimal rates are attained by fitting an LLM-compatible Transformer critic and updating a Gibbs actor, rather than by optimistic tree search.
\citet{nishikawa2026inference} likewise alternate between updating a policy and refitting a neural reward estimator, but for a non-hierarchical reward with a fully connected estimator that must be smoothed before it drives the policy, and their rate is not matched by a lower bound.
\Cref{app:related} discusses these two lines in more detail, together with search over reasoning trajectories, on-policy exploration under KL regularization, and neural function approximation in estimation and RL.

\begin{figure*}[t]
  \centering
  \includegraphics[width=0.96\linewidth]{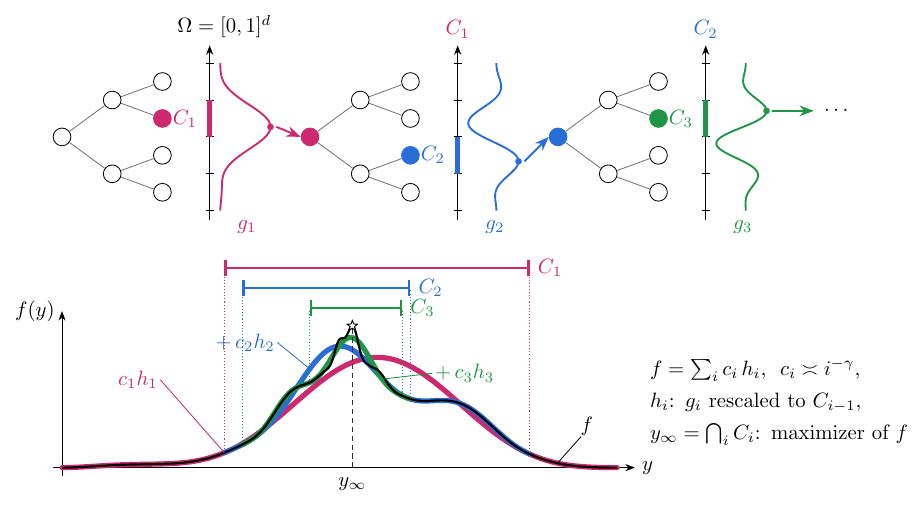}
  \vspace{-2mm}
  \caption{
    A hierarchically nested reward (\cref{def:nested}), for $d=1$ and zoom depths $\ell_i=2$.
    (Top) A response is a path in a tree with $2^d$ children per node, and a prefix of depth $h$ is a dyadic cube of that depth.
    The primitive $g_1$ lives on $\Omega$; inside the cell $C_1$ around its maximizer, $g_2$ is added, rescaled to $C_1$, and so on.
    (Bottom) The resulting reward $f=\sum_ic_ih_i$: bumps of decreasing weight $c_i\asymp i^{-\gamma}$ and decreasing width accumulate around the maximizer $\ttrue=\bigcap_iC_i$.
  }
  \label{fig:concept}
\end{figure*}

\section{Problem Setting}
\label{sec:setting}

We first introduce the class of hierarchically nested rewards.
We then formulate RL post-training as KL-regularized maximization of such a reward.
Finally, we describe how it can be queried.

\subsection{Hierarchically Nested Rewards}
\label{subsec:reward}

We consider multi-step reasoning tasks with a hierarchical reward: at first only the coarse structure of the problem is visible, and solving one step reveals a further problem nested inside it.
We model this structure by functions on the response embedding space $\Omega=[0,1]^d$ in which a new term is added each time the argument enters a small region around the maximizer of the previous term; how a response of a language model is mapped to a point of $\Omega$ is explained in \cref{subsec:regret}.

\paragraph{Hierarchically nested functions.}
Let $\mu$ be the uniform distribution on $\Omega$, $\mu(S)$ being the Lebesgue measure of $S\subset\bR^d$ and $\ind{S}$ its indicator.
A \emph{dyadic cube} of depth $h\ge0$ is a cube of side $2^{-h}$ of the form $\prod_{j\le d}[k_j2^{-h},(k_j+1)2^{-h})$ with integers $0\le k_j<2^h$, and $D_h(y)$ is the depth-$h$ dyadic cube containing $y$, so that $\Omega=D_0(y)\supset D_1(y)\supset\cdots$.
The idea is that a first function $g_1$ is defined on the whole cube; once the argument enters the small dyadic cube around the maximizer of $g_1$, a second function $g_2$ is added on that cube (rescaled to $\Omega$), and so on.
\begin{definition}[Hierarchically nested function class]\label{def:nested}
    Let $\scG=(g_i)_{i\ge1}$ be a sequence of uniformly bounded functions $g_i:\Omega\to\bR$, let $c=(c_i)_{i\ge1}$ be positive numbers with $\sum_ic_i<\infty$, and let $\ell=(\ell_i)_{i\ge1}$ be positive integers (\emph{zoom depths}); write $L_i=\sum_{a\le i}\ell_a$, $L_0=0$.
    For each $i$, fix a maximizer $z^\star_i\in\argmax_{u\in\Omega}g_i(u)$ (ties broken by a fixed rule) and let $Q_i$ be the depth-$\ell_i$ dyadic cube containing $z^\star_i$.
    Define the \emph{active cells} recursively by $C_0=\Omega$ and $C_i=b_{i-1}+2^{-L_{i-1}}Q_i$, of depth $L_i$, where $b_{i-1}$ is the lower-left corner of $C_{i-1}$, so that $C_{i-1}=b_{i-1}+2^{-L_{i-1}}\Omega$.
    The \emph{hierarchically nested function} generated by $(\scG,c,\ell)$ is
    \begin{align}
        f(y)=\sum_{i=1}^\infty c_i\,h_i(y),
        \qquad
        h_i(y)=g^0_i\bigl(2^{L_{i-1}}(y-b_{i-1})\bigr),
        \label{eq:reward}
    \end{align}
    where $g^0_i$ is the extension of $g_i$ by zero to $\bR^d$, so that $h_i$ is a rescaled copy of $g_i$ on $C_{i-1}$ and vanishes outside.
    We write $\scH(\scG,c,\ell)$ for $f$, and $\scH$ for the class of all such functions.
\end{definition}
The cells are nested, $\Omega=C_0\supset C_1\supset\cdots$, and the $i$-th level $c_ih_i$ is visible only inside $C_{i-1}$: solving level $i$ (entering $C_i$) is a prerequisite for collecting the rewards of the deeper levels.
\Cref{fig:concept} illustrates the construction.

We make the following assumption; see \cref{app:notation} for $\alpha$-H\"older smoothness.
\begin{assumption}\label{assmp:main}
    Let $\alpha>0$, $\lambda>0$, $\cmin\in(0,1)$, $\gamma>1$ and integers $1\le\lstar\le\lbar<\infty$ be known constants, and let $c_i=i^{-\gamma}/\zeta(\gamma)$ with $\zeta(\gamma)=\sum_{i\ge1}i^{-\gamma}$.
    The function $f=\scH(\scG,c,\ell)$ has unknown primitives $\scG=(g_i)_{i\ge1}$ and unknown zoom depths $\ell=(\ell_i)_{i\ge1}$ that satisfy the following conditions:
    \assmp{A1} (Smoothness and boundedness) $g_1\in[0,1]$, $g_i\in[-1,1]$, and $g_i^0$ is $\alpha$-H\"older smooth with constant $\lambda$ for all $i$.
    \assmp{A2} (Positive maximum) $g_i(z^\star_i)\ge\cmin$ for all $i$.
    \assmp{A3} (Deep enough zoom) $\ell_i\in[\lstar,\lbar]$ for all $i$, and $3^\gamma2^{-\beta\lstar}\le\min\{1/2,\cmin/(8\lambda)\}$.
    \assmp{A4} (Zero mean) $\int_\Omega g_i\,d\mu=0$ ($i\ge2$).
\end{assumption}
\vspace{-1mm}
We collect the \emph{structural parameters} as $\theta=(d,\alpha,\lambda,\cmin,\gamma,\lstar,\lbar)$, and write $T_k=\sum_{i>k}c_i\asymp k^{-(\gamma-1)}$ for the weight of the levels deeper than $k$.
Constants such as $C,c$ and those implicit in $a\lesssim b$, $a\asymp b$ and $\widetilde O(\cdot)$ depend only on $\theta$ and the noise level $\sigma$ of \cref{subsec:oracle}.

\paragraph{Remark.}
\assmp{A1} is a standard smoothness condition.
Requiring it of the zero extension $g^0_i$ makes a primitive vanish continuously at the boundary of its cell, so that $f$ is continuous although it has infinitely many levels.
\assmp{A2} and \assmp{A3} guarantee a genuine ``correct continuation'' at every level: $C_i$ is so small that the variation of the previous levels inside it is dominated by the gain $\cmin c_i$ of solving level $i$, uniformly in $i$.
\assmp{A4} makes the deeper levels cancel when averaged over a cell containing them, so that the average of $f$ over a region (\cref{subsec:oracle}) reveals only the levels that this region has resolved; by analogy with mathematical research, \emph{the deeper one digs into a problem, the deeper the results that come into view}.
Under \cref{assmp:main}, $f\in[-1,1]$ and its unique maximizer is the common point $\ttrue$ of the nested cells (\cref{lem:geometry}).

\subsection{RL Post-Training as KL-Regularized Reward Maximization}
\label{subsec:regret}

Let $\scX$ be the set of prompts and $\scY$ the set of responses.
A \emph{policy} $P$ specifies, for every prompt $x$, a conditional distribution $P(\cdot\mid x)$ over responses; a language model is a policy, and generating a response means sampling from it.
The quality of a response is measured by an unknown \emph{reward function} $f:\scX\times\scY\to\bR$.
Given a reference policy $\pi_{\rm ref}$, typically the model before post-training, and a regularization strength $\rho\in[0,1]$, post-training maximizes the KL-regularized objective
\begin{align}
    J_\rho(P;x)=\bE_{y\sim P(\cdot\mid x)}[f(x,y)]-\rho\,\KL{P(\cdot\mid x)}{\pi_{\rm ref}(\cdot\mid x)}
    \label{eq:objective}
\end{align}
over prompts $x\sim\nu$, where $\KL\cdot\cdot$ is the Kullback--Leibler divergence; this is the objective of KL-regularized post-training such as PPO and GRPO~\citep{ziegler2019fine,ouyang2022training,schulman2017proximal,shao2024deepseekmath}.
The performance of a policy is its \emph{regret}
\begin{align}
    J^\ast_\rho(x):=\sup\nolimits_PJ_\rho(P;x),\qquad
    \Reg_\rho(P)=\bE_{x\sim\nu}\bigl[J^\ast_\rho(x)-J_\rho(P;x)\bigr].
    \label{eq:regret}
\end{align}
For $\rho=0$ we have $J^\ast_0(x)=\sup_yf(x,y)$, so $\Reg_0$ is the unregularized regret.
The learner does not know $f$ and can only query a noisy reward oracle a limited number of times (\cref{subsec:oracle}), which is the dominant cost when the reward comes from a verifier or a human.

For simplicity, we take the prompts to be $M$ labels, $\scX=[M]$, drawn uniformly, $\nu=\unif([M])$, and the responses to be infinite binary sequences, $\scY=\{0,1\}^{\{0,1,2,\ldots\}}$, generated one bit at a time; the reference policy $\pi_{\rm ref}$ generates the bits uniformly at random.
Prompts are exogenous: at every round of interaction, a fresh prompt is drawn from $\nu$ independently of the past, and the learner cannot choose or reject prompts.
The reward functions of different prompts are unrelated, and we write $f(y;m)$ for $f(m,y)$.

\paragraph{Embedding of responses.}
We embed responses into $\Omega=[0,1]^d$ so that the responses sharing a given prefix form a dyadic cube, and a longer prefix a smaller cube inside it.
Concretely, we interleave the bits across the $d$ coordinates in a round-robin fashion: $\phi:\scY\to\Omega$ with $\phi_j(y)=\sum_{k\ge0}y_{kd+j-1}2^{-(k+1)}$ for $j\in[d]$.
A prefix of length $dh$ then corresponds to the dyadic cube $D_h(\phi(y))$, and generating bits uniformly corresponds to $\phi(y)\sim\mu$.
We identify a response with $\phi(y)$ and a prefix with a cube, so that a policy is a distribution over $\Omega$.
The embedding can be read as part of the language model itself: it places responses that share a long prefix close to each other, so that the reward, as a function on $\Omega$, records how the quality of a response depends on its reasoning prefix.
Any embedding whose regions form a family of dyadic cubes would do; the results do not depend on the particular interleaving.

\paragraph{Hierarchically nested rewards of prompts.}
For every prompt $m\in[M]$ the reward is a hierarchically nested function $f(\cdot;m)=\scH(\scG_m,c,\ell)$ satisfying \cref{assmp:main}, with unknown primitives $\scG_m=(g_{i,m})_{i\ge1}$; the weights $c$ and the zoom depths $\ell$ are shared, and we write $C_i(m)$, $h_i(\cdot;m)$, $\ttrue(m)$ for the cells, levels and maximizer of prompt $m$, and $\scF_{M,\theta}$ for the resulting class of rewards.
Each level is one difficulty of the task: its contribution $c_i\asymp i^{-\gamma}$ decays with the depth at which it appears and is collected only once the earlier ones have been resolved.
The reward is thus a function of the whole response rather than of a final answer, as in post-training pipelines where a verifier or a generative reward model scores the entire trajectory~\citep{shao2025deepseekmathv2,deepseek2026v4}.

\subsection{Reward Oracle and Learning Protocol}
\label{subsec:oracle}

The learner interacts with the environment for $n$ rounds, each of which yields exactly one observation.
At round $t$ a prompt $X_t\sim\nu$ arrives, the learner generates a possibly partial response and queries the reward of \emph{one} prefix, i.e.\ one dyadic cube $D_t$ containing the response.
The oracle returns the average reward over that cube plus noise,
\begin{align}
    Z_t=Q(X_t,D_t)+\xi_t,\qquad
    Q(m,D):=\bE_{y\sim\mu}\qty[f(y;m)\relmiddle| y\in D],
    \label{eq:oracle}
\end{align}
where, conditionally on the history and on $(X_t,D_t)$, the noise is mean zero and $\sigma$-sub-Gaussian, with $\sigma$ (or an upper bound on it) known to the learner.

Equivalently, $Q(m,D)$ is the expected terminal reward of a response with prefix $D$ whose remaining bits are generated by the reference policy, the randomness of the completion being absorbed into the noise.
This value can be estimated by Monte Carlo completions, as in methods that derive process supervision from terminal outcomes~\citep{wang2024mathshepherd,setlur2024rewarding}.
By \assmp{A3} and \assmp{A4}, a prefix of depth $h$ resolves at most the first $\max\{1,\lceil h/\lstar\rceil\}$ levels, and deeper levels do not contribute to $Q(m,D)$.

After the $n$ rounds the learner outputs a policy $\Phat(\cdot\mid m)$ for every prompt.
It knows $\theta$ and $\sigma$, but neither the primitives, nor the zoom depths, nor the cells, nor the maximizers.

\section{Transformer Actor--Critic Attains the Minimax Rates}
\label{sec:upper}

In this section, we present an actor--critic algorithm in which the critic is a Transformer trained by empirical risk minimization on the full class and the actor is a KL-regularized (Gibbs) policy, and prove its regret bound.

\subsection{Transformer Critic}
\label{subsec:transformer}

The critic takes a prompt and a prefix and outputs a scalar.
Its input is a token sequence $X=(x_1,\ldots,x_T)$ over the vocabulary $\scV=[M]\cup\{0,1,\STOP,\PAD,\RO\}$: the prompt token $m$, the $dh$ bits of the prefix, a $\STOP$ token, $\PAD$ tokens, and a readout token $\RO$ at the last position, the $\PAD$ tokens bringing every input to the same length $T$.
Tokens are represented by one-hot vectors, $X\in\bR^{\abs{\scV}\times T}$.
Following \citet{takakura2023approximation}, the Transformer architecture has three components: (i) a position-wise FNN layer, (ii) a self-attention layer, and (iii) an embedding layer.

\textbf{(i)} An FNN of depth $L$ and width $W$ is $f(x):=(A_L\eta(\cdot)+b_L)\circ\cdots\circ(A_1x+b_1)$ with $A_i\in\bR^{d_{i+1}\times d_i}$, $b_i\in\bR^{d_{i+1}}$, $\max_id_i\le W$ and the ReLU activation $\eta(x)=\max\{x,0\}$ applied element-wise; $\Psi(L,W,S,B)$ is the class of such FNNs with $\max_i\{\norm{A_i}_\infty,\norm{b_i}_\infty\}\le B$ and $\sum_{i\le L}\norm{A_i}_0+\norm{b_i}_0\le S$.
\textbf{(ii)} A causal self-attention layer with embedding dimension $D$ and parameters $K,Q,V\in\bR^{D\times D}$ maps $X\in\bR^{D\times T}$ to $g(X)_i:=x_i+VX_{1:i}\Softmax((KX_{1:i})^\top(Qx_i))$, $i\in[T]$, where $X_{1:i}=(x_1,\ldots,x_i)$ and $\Softmax(z)=[e^{z_j}/\sum_{j'}e^{z_{j'}}]_j$, so that each position attends only to itself and the preceding ones; $\scA(D,B)$ denotes the class of such layers with $\max\{\norm{K}_\infty,\norm{Q}_\infty,\norm{V}_\infty\}\le B$.
\textbf{(iii)} The embedding layer is $\enc_P(X)=EX+P$ with a learnable $E\in\bR^{D\times\abs{\scV}}$ and the fixed sinusoidal positional encoding $P=[p_t]_{t=1}^T\in\bR^{D\times T}$, $p_{t,2r}=\sin(t/10000^{2r/D})$, $p_{t,2r+1}=\cos(t/10000^{2r/D})$.

We define the class of Transformers $\Tcls=\Tcls(\Nblk,D,L,W,S,B)$ by
\begin{align*}
    \Tcls
    :=\qty{
        \clip\circ f_{\Nblk}\circ g_{\Nblk}\circ\cdots\circ f_1\circ g_1\circ\enc_P
        \relmiddle|
        \begin{array}{@{}l@{}}
            \norm{E}_\infty\le B,\ g_b\in\scA(D,B),\\[1.5mm] f_b\in\Psi(L,W,S,B)\ \ (b\in[\Nblk])
        \end{array}
    },
\end{align*}
where $\Nblk$ is the number of blocks (the depth of the Transformer; a block consists of a self-attention layer $g_b$ followed by an FNN layer $f_b$), the FNNs are applied column-wise, the last FNN $f_{\Nblk}$ has a one-dimensional output, and $\clip=\clip_{[-1,1]}$ is applied element-wise.
A network $F\in\Tcls$ thus maps an input sequence to a sequence of scalars in $[-1,1]$, and the critic's prediction is the entry at the readout position, $F(X)_T$; we identify $F$ with this scalar function.
The class is that of \citet{takakura2023approximation}, with causal attention in place of sliding-window attention, a single head, and the output read at the last position.

\subsection{Algorithm}
\label{subsec:algorithm}

\begin{algorithm}[t]
    \caption{Transformer actor--critic with prefix-level exploration}
    \label{alg:actor-critic}
    \begin{algorithmic}[1]
        \Statex {\bfseries Input:} budget $n$, prompts $M$, regularization $\rho$, and the schedule of \cref{subsec:algorithm}.
        \State Initialize the score $W_0\equiv0$.
        \For{$h=0,1,\ldots,\Hall-1$}
            \State \textbf{Actor:} $P_h(dy\mid m)\propto\exp\qty(W_h(m,y)/\vartheta_h)\,\mu(dy)$,
            \Statex\qquad i.e., $P_h(\cdot\mid m)=\argmax_P\{\bE_{y\sim P}[W_h(m,y)]-\vartheta_h\KL{P}{\mu}\}$. \label{ln:actor}
            \State \textbf{Collect:} for $t=1,\ldots,b_h$: receive $X_t\sim\nu$; generate $Y_t\sim P_h(\cdot\mid X_t)$ prefix by prefix, stopping at the prefix $D_t=\Query_h(X_t,Y_t)$; observe $Z_t=Q(X_t,D_t)+\xi_t$. \label{ln:collect}
            \State \textbf{Critic:} $\That_h\in\argmin_{T\in\Tcls_h}\frac1{b_h}\sum_{t=1}^{b_h}\qty[T(X_t,D_t)-Z_t]^2$. \label{ln:erm}
            \State \textbf{Update:} $W_{h+1}=W_h+I_h\cdot\clip_{[-\Delta_h,\Delta_h]}\bigl(\That_h(\cdot,\Query_h(\cdot))-W_h\bigr)$ pointwise in $(m,y)$, \label{ln:score}
            \Statex\qquad with the mask $I_h(m,y)\in\{0,1\}$ and the clipping width $\Delta_h(m,y)$.
        \EndFor
        \State \textbf{return:} $\Phat(dy\mid m)\propto\exp\qty(W_{\Hall}(m,y)/\vartheta_{\Hall})\,\mu(dy)$.
    \end{algorithmic}
\end{algorithm}

\Cref{alg:actor-critic} is an on-policy actor--critic of the kind used in RL post-training.
It maintains a \emph{score} $W_h(m,y)$, the current estimate of the reward, and an \emph{actor} $P_h$, the Gibbs policy of that score at a temperature $\vartheta_h$; since the score is piecewise constant on dyadic cubes, the actor is sampled autoregressively, the probability of the next $d$ bits being the ratio of the Gibbs masses of the children.
Each iteration collects a fresh batch of on-policy responses and queries their rewards, fits the critic to them, and updates the actor.
Every decision about a sample is a function of its own prefix, the shared critics and the shared rules.
What distinguishes the algorithm from a vanilla actor--critic is the choice of the queried prefix.

\paragraph{Depth curriculum and adaptive resolution.}
Level $i$ of the reward is visible only inside $C_{i-1}(m)$, so the learner has to zoom in along the correct path, while deep bits of a response that is already wrong at a shallow level carry no information \assmp{A4}.
\Cref{alg:actor-critic} therefore uses a \emph{depth curriculum}: at iteration $h$ the actor's confidence is judged at depth $h$, which increases by one per iteration, and a response that is confident down to that depth is queried $\ell_{\rm look}$ depths deeper, the \emph{lookahead}.
Training on reasoning of increasing length or difficulty is standard practice in RL post-training~\citep{luo2025deepscaler,kimi2025k15} and an instance of curriculum learning~\citep{bengio2009curriculum}.
The critic is a \emph{process} reward model, trained on prefixes of all depths~\citep{lightman2023let,wang2024mathshepherd}, which allows a response to be stopped early and queried at a coarse prefix whenever the actor is not confident about it (\emph{adaptive resolution}).
The query rule is
\begin{align}
    \Query_h(m,y)=
    \begin{cases}
        D_k(y)&\text{at the first depth }1\le k\le h\text{ with }P_h(D_k(y)\mid m)<\kappa_0,\\
        D_{h+\ell_{\rm look}}(y)&\text{if there is no such depth},
    \end{cases}
    \label{eq:query-rule}
\end{align}
with a constant threshold $\kappa_0>0$: generation stops at the first prefix whose actor mass falls below $\kappa_0$ and that prefix is queried; if none does, the response is extended and queried $\ell_{\rm look}$ levels deeper.
The rule uses only the probability that the current actor assigns to a prefix, which an autoregressive model already computes while generating it.
The lookahead is $\ell_{\rm look}=1$ in the small-regularization setting, where the iterations simply continue as long as the budget allows and the output is the last actor, and $\ell_{\rm look}=R+1$ with $R=O(\log\Lf)$ and $\Lf:=\log(e+n+M+\rho^{-1})$ in the large-budget setting, where they are followed by the step described next.

\paragraph{Training the critic.}
The critic $\That_h$ is the least-squares fit over the entire class $\Tcls_h$ to all records of the batch, including those of the responses that stopped early; all parameters are fitted, and no quantity is computed or memorized per cube.
Its risk is controlled by the approximation power and the covering number of $\Tcls_h$, as in the nonparametric analysis of least squares over deep network classes~\citep{schmidt2020nonparametric,takakura2023approximation}.
The batch of iteration $h$ has size $b_h\asymp M^{\zeta_M}c_{h+\ell_{\rm look}}^{-2}$ up to logarithms, with $\zeta_M:=2+d/(2\alpha)$: an iteration must reach the accuracy $c_{h+\ell_{\rm look}}$ of the level it queries (the schedules are given in \cref{app:small-params,app:parameters}).

\paragraph{Actor update.}
Two devices standard in RL post-training are used in the update.
In the front phase, the score is replaced only on the cubes of actor mass at least $\kappa_0$, where the critic has just been trained on fresh on-policy data, and the mask $I_h(m,y)=\ind{P_h(D_h(y)\mid m)\ge\kappa_0}$ keeps the old score elsewhere: this is a trust region in the spirit of TRPO and PPO~\citep{schulman2015trust,schulman2017proximal}.
In the front phase no clipping is needed, $\Delta_h\equiv2$; the final update is clipped, as in PPO, since a Gibbs policy at temperature $\rho$ amplifies an overestimate of the score exponentially.

\paragraph{Final refinement step (large budgets only).}
After $H\asymp\rho^{-1/\gamma}+\log(n/\rho)$ exploration iterations the actor has resolved as many levels as $\rho$ can distinguish, and the remaining budget $b_H=n-\sum_{h<H}b_h$ is spent on one last iteration, which samples at temperature $\rho$, sets $I_H\equiv1$, and uses
\begin{align}
    \Query_H(m,y)=D_{q(m,y)+\ell_{\rm look}+r}(y),\qquad
    \Delta_H(m,y)=\rho+\eta\,c_{1+\lfloor q(m,y)/\lbar\rfloor},
    \label{eq:final-step}
\end{align}
where $q(m,y)$ is the last exploration iteration at which the region of $y$ was updated, $r$ is a refinement depth and $\eta$ is a structural constant.
Querying $r$ levels deeper than the last update balances approximation against estimation: the reward is H\"older in the local coordinates of each shell $C_{a-1}\setminus C_a$, so a finer query buys a better fit until the estimation error takes over.

\subsection{Regret Upper Bound}
\label{subsec:upper-bound}

We give two upper bounds, one for each regime of $\rho$, proved in \cref{app:small,app:upper}.
\begin{theorem}\label{thm:upper}
    Suppose that $f\in\scF_{M,\theta}$, that the noise is conditionally $\sigma$-sub-Gaussian, that $\rho\in[0,1]$ and that $n\ge M^{\zeta_M}$ with $\zeta_M=2+d/(2\alpha)$.
    Run \cref{alg:actor-critic} with the hyperparameters of \cref{app:small-params} in the first case below and of \cref{app:parameters} in the second; in both cases it uses at most $n$ queries, and its output satisfies, for a fixed exponent $c_{\log}>0$ and with $\widetilde O$ hiding logarithmic factors,
    \begin{align}
        \bE\qty[\Reg_\rho(\Phat)]\ \le\
        \begin{cases}
            \ \widetilde O\qty(\qty(n/M^{\zeta_M})^{-\frac{\gamma-1}{2\gamma+1}})
            &\text{if }\rho=0\text{ or }\rho\lesssim(n/M)^{-\frac\gamma{2\gamma+1}},
            \\[2mm]
            \ \widetilde O\qty(\rho\qty(n\rho^{2+1/\gamma}/M)^{-\frac{2\alpha}{2\alpha+d}})
            &\text{if }\rho\gtrsim\qty(n/M^{\zeta_M})^{-\frac\gamma{2\gamma+1}}\Lf^{c_{\log}},
        \end{cases}
        \label{eq:upper}
    \end{align}
\end{theorem}
\Cref{thm:upper} leaves a band between its two cases, namely $(n/M)^{-\gamma/(2\gamma+1)}\lesssim\rho\lesssim(n/M^{\zeta_M})^{-\gamma/(2\gamma+1)}\Lf^{c_{\log}}$, in which the algorithm is well defined but we prove no guarantee.
Up to logarithmic factors, the gap between the two thresholds is the price of using a single critic for all prompts: the guaranteed depth is $(n/M^{\zeta_M})^{1/(2\gamma+1)}$ rather than $(n/M)^{1/(2\gamma+1)}$, since estimation theory bounds the $L^2$ error under the joint query law, not the error on every prefix that a prompt updates (\cref{lem:joint}).
For a fixed number of prompts the gap is only polylogarithmic, and outside the band both rates are the minimax rates of \cref{sec:lower} up to logarithmic factors.

\paragraph{Proof sketch.}
We sketch the proofs through the three inequalities that produce the two rates.

First, the actor keeps exploring the correct path.
The two facts support each other by induction on $h$.
If the score is accurate on each cube at the scale of the depth at which that cube was created, then the reward gap between the correct cube of depth $h$ and any cube that left the path exceeds the errors made so far, and the Gibbs actor satisfies $P_h(D_h(\ttrue(m))\mid m)\ge p_\ast>0$ uniformly in $h$ and in the reward (\cref{lem:mass}).
The correct cube is therefore refined at every iteration, and the fresh batch drawn there restores the accuracy one level deeper; this is what keeps the queries concentrated where the next level is visible, and the same computation shows that the mass leaving the path $j$ levels earlier decays like $e^{-cj}$, so the queried prefixes stay in a window of logarithmic length and the critic only has to approximate finitely many H\"older functions there.

Second, the depth reached is fixed by the budget.
Resolving level $i$ requires an accuracy $\asymp c_i\asymp i^{-\gamma}$ on a cube average, so the batches needed to reach depth $H$ satisfy
\begin{align}
    \sum\nolimits_{h<H}b_h\ \asymp\ M^{\zeta_M}\sum\nolimits_{h<H}c_h^{-2}\ \asymp\ M^{\zeta_M}H^{2\gamma+1}\ \le\ n,
    \label{eq:sketch-cost}
\end{align}
up to logarithms, the power $\zeta_M$ being the price of converting the risk of the shared critic into accuracy at each updated prefix.
In the small-regularization case this is the whole argument: the output is the last actor, uniform on cubes of depth $H$, its regret is the weight $T_H\asymp H^{-(\gamma-1)}$ of the unresolved levels, and \eqref{eq:sketch-cost} gives $H\asymp(n/M^{\zeta_M})^{1/(2\gamma+1)}$, hence the first exponent of \eqref{eq:upper}.
The regularization only adds the penalty $\rho\,dH\log2$, which stays below $T_H$ in the first case.

Third, in the large-budget case the score has to be turned into a regret bound.
A Gibbs policy at temperature $\rho$ cannot distinguish levels of weight below $\rho$, i.e.\ levels deeper than $K\asymp\rho^{-1/\gamma}$; reaching that depth costs $M^{\zeta_M}\rho^{-(2+1/\gamma)}$ queries by \eqref{eq:sketch-cost}, up to logarithms.
What remains is the fine shape of the reward around the correct path, which the last iteration estimates with one more critic.
The regret is an integral of variances under the interpolated Gibbs policies.
Pointwise accuracy of the held score, clipping and exponential-moment control bound their density ratios relative to the collecting policy, giving schematically
\begin{align}
    \bE\bigl[\Reg_\rho(\Phat)\bigr]\ \lesssim\
    \rho^{-1}\bigl[(\text{critic error})+(\text{projection error})\bigr]
    +(\text{deep-core and failure errors})
    \label{eq:sketch-transfer}
\end{align}
(\cref{prop:final}).
Here the projection error is the squared error from replacing the reward by its cube averages; the core and failure terms are lower order.
Since the reward is $\alpha$-H\"older in $d$ local coordinates on each shell, the squared error obeys the nonparametric tradeoff, and balancing it against \eqref{eq:sketch-cost} gives the second case of \eqref{eq:upper}.

\section{Minimax Lower Bound}
\label{sec:lower}

We now show that the rates of \cref{thm:upper} cannot be improved by any algorithm, including algorithms that select prefixes adaptively without any restriction on computation, memory, or the use of past data.
Let $\scN_\sigma$ be the class of noise distributions that are conditionally mean zero and $\sigma$-sub-Gaussian as in \cref{subsec:oracle}, and define the minimax regret
\begin{align}
    \Rmin(n,M,\rho):=\inf\nolimits_{\scA}\sup\nolimits_{f\in\scF_{M,\theta},\,N\in\scN_\sigma}\bE_{f,N,\scA}\bigl[\Reg_\rho(\Phat_{\scA})\bigr],
    \label{eq:minimax}
\end{align}
where the infimum is over all learning algorithms that use at most $n$ queries in the protocol of \cref{subsec:oracle}.
\begin{theorem}\label{thm:lower}
    Fix $\sigma>0$, suppose that $\lambda\ge\Lambda_{\alpha,d}$ and let $\rho\in[0,1]$.
    For all $M\ge1$ and $n\ge M$,
    \begin{align}
        \Rmin(n,M,\rho)\ \gtrsim\
        \begin{cases}
            \ (n/M)^{-\frac{\gamma-1}{2\gamma+1}}
            &\text{if }\rho=0\text{ or }\rho\lesssim(n/M)^{-\frac\gamma{2\gamma+1}},
            \\[2mm]
            \ \rho\min\qty{1,\qty(n\rho^{2+1/\gamma}/M)^{-\frac{2\alpha}{2\alpha+d}}}
            &\text{for every }\rho>0,
        \end{cases}
        \label{eq:lower}
    \end{align}
\end{theorem}
The proof is given in \cref{app:lower}.
The condition $\lambda\ge\Lambda_{\alpha,d}$ only ensures that the class $\scF_{M,\theta}$ is rich enough to contain the hard instances, with the sufficient constant $\Lambda_{\alpha,d}$ defined in \eqref{eq:richness-threshold} for a bump of amplitude one and radius $1/8$; the bounds hold even if the zoom depths $\ell_i\equiv\lstar$ are revealed to the learner.

The table below compares \cref{thm:upper,thm:lower} regime by regime; all entries are up to polylogarithmic factors.
The first cases of \eqref{eq:upper} and \eqref{eq:lower} hold under the same condition on $\rho$, and the second case of \eqref{eq:lower} holds for every $\rho>0$.
In both regimes the two columns agree in the budget $n$ and in the regularization $\rho$, and for a fixed number of prompts they agree completely: there the algorithm is \emph{minimax optimal up to logarithmic factors}.
Since the lower bounds hold for every algorithm, no amount of computation or memory can improve on our on-policy scheme by more than logarithmic factors.
Apart from logarithmic factors, the two columns differ only through $M$, in the rate when the regularization is small and in the range of $\rho$ when the budget is large; both differences come from converting the joint $L^2$ error of the shared critic into accuracy at every updated prefix, which makes the cost of reaching depth $H$ in \eqref{eq:sketch-cost} grow as $M^{\zeta_M}$ rather than $M$ in the number of prompts.

\begin{center}
\begin{tabular}{lll}
\hline
 & Lower bound (\cref{thm:lower}) & \cref{alg:actor-critic} (\cref{thm:upper}) \\
\hline
Small regularization & $(n/M)^{-\frac{\gamma-1}{2\gamma+1}}$ & $\qty(n/M^{\zeta_M})^{-\frac{\gamma-1}{2\gamma+1}}$ \\
Large budget & $\rho\qty(n\rho^{2+1/\gamma}/M)^{-\frac{2\alpha}{2\alpha+d}}$ & $\rho\qty(n\rho^{2+1/\gamma}/M)^{-\frac{2\alpha}{2\alpha+d}}$ \\
\hline
\end{tabular}
\end{center}

\paragraph{Remark (the two bottlenecks).}
The two bounds describe two different obstructions.
\emph{Small regularization: the budget.} The budget runs out before the levels of weight $c_i>\rho$ are resolved, so the regret is dominated by the weight of the unresolved levels and neither $\alpha$ nor $d$ appears. The bound is the same for $\rho=0$ and for every $\rho\lesssim(n/M)^{-\gamma/(2\gamma+1)}$, because in that range the regularization changes the optimal value only by $O(\rho^{1-1/\gamma})$, a small fraction of $(n/M)^{-(\gamma-1)/(2\gamma+1)}$.
\emph{Large budget: the regularization.} The correct path can be resolved down to the last level of weight $c_i>\rho$, at depth $\rho^{-1/\gamma}$, and what is left is to estimate the fine shape of the reward around it.
The second bound holds for every $\rho>0$, and the first, where it applies, is the larger one.

\vspace{-1.5mm}
\section{Lower Bound for Fixed-Distribution Sampling: The Benefit of On-Policy Exploration}
\label{sec:fixed}

The algorithm of \cref{sec:upper} is \emph{on-policy}: the responses whose rewards are queried are generated by the current actor.
A common alternative in reward modeling is to collect responses from a fixed base distribution and to fit the reward model offline.
In this section, we show that this restriction is costly in the hierarchical setting.
Let $\scAunif$ be the class of algorithms that, at every round $t$, generate the response $Y_t\sim\mu$ uniformly at random (independently of the past), may choose the depth of the queried prefix $D_t\mathrel{\ni}Y_t$ after seeing $(X_t,Y_t)$ and the history, and output an arbitrary policy $\Phat$ at the end.
The only restriction is the sampling distribution: the estimator, the choice of the query depth and the final policy are unrestricted, and past data may be reused freely.
Let $\RC(n,M,\rho)$ be the minimax regret \eqref{eq:minimax} with the infimum restricted to $\scAunif$.

\begin{theorem}[The price of fixed sampling]\label{thm:fixed}
    Suppose that $\sigma>0$ and $\lambda\ge\Lambda_{\alpha,d}$.
    There are constants $\rho_0,c_0>0$ such that, for all $n\ge M\ge1$,
    \begin{align}
        \RC(n,M,\rho)\ \gtrsim\
        \begin{cases}
            \ \qty(\log\tfrac{en}M)^{-(\gamma-1)}
            &\text{if }\rho\lesssim\qty(\log\tfrac{en}M)^{-\gamma},
            \\[2mm]
            \ \rho\min\qty{1,\qty(n\rho^{2+1/\gamma}e^{-c_0\rho^{-1/\gamma}}/M)^{-\frac{2\alpha}{2\alpha+d}}}
            &\text{if }\rho\in(0,\rho_0].
        \end{cases}
        \label{eq:fixed}
    \end{align}
    The first bound holds for every $\sigma\ge0$, that is, even for a noiseless oracle.
\end{theorem}
The proofs are given in \cref{app:fixed}.
The two bounds describe two budget regimes.
While $\rho\lesssim(\log(en/M))^{-\gamma}$, the regret of any fixed-sampling method decays at most logarithmically in $n/M$, whereas in the large-budget regime of \cref{thm:upper}, the on-policy actor--critic achieves a polynomial rate; the two conditions are compatible for small $\rho$.
For larger budgets the fixed-sampling regret does decay polynomially, but it reaches that same rate only with the multiplicative price $\exp(c_0\rho^{-1/\gamma})$, whereas the on-policy bound depends on $\rho^{-1}$ polynomially.
Two complements are proved there: for a \emph{fixed} $\rho>0$ and $n\to\infty$ the exponent of $n$ is not worsened by fixed sampling (\cref{prop:fixed-rho}); and the first bound is tight in its regime (\cref{prop:uniform-grid}).

\paragraph{Why on-policy sampling escapes.}
Once the actor has learned that a prefix is promising it generates responses from that prefix, so the cell of the next level is visited with constant probability rather than with probability $2^{-d\lstar k}$.
Each iteration of \cref{alg:actor-critic} then resolves one level at a cost of $c_k^{-2}\asymp k^{2\gamma}$ queries instead of $2^{d\lstar k}$, which is the statistical benefit of RL post-training over offline reward modeling in this setting.

\vspace{-1.5mm}
\section{Conclusion}
\label{sec:conclusion}

We studied RL post-training with hierarchical rewards, in which the reward of a binary response is a sum of infinitely many H\"older primitives revealed level by level.
A Transformer actor--critic of the standard form, alternating on-policy sampling from a KL-regularized policy, least-squares training of a shared Transformer critic and a Gibbs policy update, matches the minimax rates in the budget and in the regularization up to logarithmic factors, and is minimax optimal for a fixed number of prompts.
The analysis identifies the role of each component: the prefix mass of the current policy selects where to query, the temperature schedule keeps the correct path confident, the full-class ERM adapts to the unknown cells, and a calibrated, clipped final update transfers the estimation error to the regret.
Any method that keeps sampling from the fixed reference distribution $\mu$, in contrast, suffers a logarithmic rate or an exponential price in $\rho^{-1/\gamma}$, which quantifies the benefit of on-policy exploration.
\paragraph{Limitations and future work.}
Our upper bounds are not optimal in the number of prompts $M$, for the reason given in \cref{sec:lower}; removing the extra factor, for instance through a prompt-wise estimation guarantee or an analysis of the exploration in expectation over the prompts, is left open.
The two upper bounds are also separated by the band described after \cref{thm:upper}, where the same algorithm is well defined but we prove no guarantee, and the two parameter settings differ in the lookahead and in the presence of the refinement step.
We assumed exact ERM and exact Gibbs sampling; the optimization remains to be analyzed.
An empirical study of the prefix-mass exploration rule on synthetic hierarchical rewards and reasoning benchmarks is left for future work.

\subsection*{Acknowledgments}
NN was partially supported by JST ACT-X (JPMJAX24CK) and JST BOOST (JPMJBS2418).
TS was partially supported by JST CREST (JPMJCR2115) and JSPS KAKENHI (24K02905). This work was partially supported by JST ERATO (JPMJER2601). This research is supported by the National Research Foundation, Singapore and the Ministry of Digital Development and Information under the AI Visiting Professorship Programme (award number AIVP-2024-004). Any opinions, findings and conclusions or recommendations expressed in this material are those of the author(s) and do not reflect the views of National Research Foundation, Singapore and the Ministry of Digital Development and Information.

\subsection*{AI use statement}
In this work, generative AI tools assisted with developing the theoretical model and formulating mathematical claims, provided key ingredients for proofs, and assisted with revising mathematical statements, developing and checking proof arguments, and drafting and editing proofs.
Additionally, we used generative AI tools for drafting and revising parts of the manuscript, for the code that generates the figure, and for identifying and formatting references.
We have reviewed all AI-assisted work.
Every theorem, lemma and proof was verified by the authors, and all AI-drafted text was revised by them.
We take responsibility for the final content of this work, including text, claims or artifacts produced with the aid of generative AI.

\begingroup
\appto\UrlBreaks{\do\-}

\endgroup

\appendix
\newpage
\section{Additional Related Work}
\label{app:related}

This appendix expands the discussion of related work in \cref{sec:intro}.

\paragraph{Prompt-side curricula.}
Prior work provides theoretical analyses of how easy-to-hard task schedules \citep{li2023understanding,parashar2026curriculum}, performance-based prompt selection \citep{rajaraman2026autocurriculum}, and depth- or hint-based curricula \citep{bu2026provable} improve learning efficiency. These curricula adjust the tasks or prompts used for training. In our setting, the prompt distribution is fixed; reward feedback instead guides the learner to adapt its response distribution toward promising regions. Accordingly, the baseline in our separation is sampling from the fixed reference distribution, not RL without a curriculum, and the two forms of adaptivity can be combined.

\paragraph{Search and exploration over reasoning trajectories.}
\citet{kim2025metastable} model chain-of-thought as a metastable Markov process and analyze how search discovers rare transitions between clusters of reasoning states. \citet{wang2026planning} study language models that generate paths between given vertices of a graph, analyzing how exploration and policy updates affect planning accuracy, convergence, and output diversity. In contrast, we ask how much adapting the response distribution can reduce the number of noisy reward observations needed for policy learning, relative to sampling from the fixed reference distribution.

\paragraph{On-policy exploration under KL regularization.}
Exploration and iterative learning under KL regularization have been analyzed in linear-softmax and preference-feedback frameworks \citep{foster2025good,xiong2024iterative}. \citet{zhao2025logarithmic} bound the suboptimality of Gibbs policies through squared reward-estimation errors, a mechanism that our analysis also uses (\cref{app:clipping}), while \citet{wang2026greedy} show that sampling from the Gibbs policy of an estimated reward can itself provide sufficient exploration under sufficiently strong regularization. In contrast, we study Gibbs-policy learning with noisy prefix-reward observations and hierarchically nested nonparametric rewards.

\paragraph{Continuum-armed bandits.}
Continuum-armed bandit methods concentrate sampling on promising regions by refining the action space, as in zooming and HOO \citep{kleinberg2008multi,bubeck2011x}; SOO and StoSOO handle unknown local smoothness \citep{munos2011optimistic,valko2013stochastic}. Our algorithm shares this coarse-to-fine exploration principle, but our problem formulation reflects two features of LLM post-training. Specifically, we use KL regularization to retain response diversity and hierarchically nested rewards to model reasoning tasks in which resolving one difficulty reveals the next. In addition, our rates are attained not by optimistic tree search with explicit confidence bounds but by an LLM-compatible actor--critic, which fits a Transformer critic to on-policy reward feedback and updates the actor as a Gibbs policy.

\paragraph{Neural function approximation in estimation and RL.}
Approximation and estimation theory for deep networks provides statistical guarantees for nonparametric regression \citep{schmidt2020nonparametric,suzuki2019adaptivity,hayakawa2020minimax}, and \citet{takakura2023approximation} establish corresponding results for Transformers with sequence inputs. These results underpin our bounds on the Transformer critic's estimation error. In neural value-function learning for MDPs, \citet{fan2020theoretical} use a fixed sampling distribution, while \citet{yang2020provably} use optimism-based exploration. Both analyses link approximation and estimation errors to policy-performance guarantees. In contrast, under a hierarchical reward structure modeling reasoning, we study how much adapting the response distribution can improve learning from reward feedback relative to sampling from the fixed reference distribution.

\section{Common Preliminaries and Tools}
\label{app:common}

This section collects the material that is used by more than one of the proofs that follow: the conventions and notation of the whole appendix (\cref{app:notation}), the geometry of hierarchically nested rewards (\cref{app:geometry}), the Gibbs mass of the correct cube under a birth-accurate score (\cref{app:gibbs-mass}), and the approximation and estimation theory of the Transformer critic on the relevant window (\cref{app:transformer}).
The geometry of \cref{app:geometry} is used by the upper bounds of \cref{app:upper,app:small} and again by the lower bounds of \cref{app:lower,app:fixed}; the remaining subsections are common to the two upper bounds.
The results specific to one theorem, together with the hyperparameter schedules that they use, are stated in the section devoted to that theorem.

\subsection{Conventions and Notation}
\label{app:notation}

\paragraph{$\alpha$-H\"older smoothness.}
For $\alpha>0$, put $k=\lceil\alpha\rceil-1$, $\tau=\alpha-k\in(0,1]$ and $\beta=\min\{\alpha,1\}$.
For $v\in C^k(\bR^d)$ define
\begin{align}
    |v|_\alpha:=\max\left\{\sum_{1\le|s|\le k}\|D^s v\|_\infty,\ \max_{|s|=k}[D^s v]_\tau\right\},
    \qquad [w]_\tau:=\sup_{u\ne v}\frac{|w(u)-w(v)|}{\|u-v\|_\infty^\tau},
    \label{eq:holder-convention}
\end{align}
where $s$ is a multi-index and the empty sum is zero. On a cube, $\|v\|_{C^\alpha}:=\|v\|_\infty+|v|_\alpha$ uses the same restricted norms.
We call $v:\bR^d\to\bR$ \emph{$\alpha$-H\"older smooth with constant $\lambda$} if $v\in C^k(\bR^d)$ and $|v|_\alpha\le\lambda$.
Thus integer smoothness means that derivatives of order $\alpha-1$ are Lipschitz; for $\alpha\le1$, $|v|_\alpha=[v]_\alpha$.
Condition \assmp{A1} preserves the original H\"older condition for $\alpha\le1$ and implies $|g_i^0(u)-g_i^0(v)|\le\lambda\|u-v\|_\infty^\beta$ for every $\alpha>0$: when $\alpha>1$, integrate the gradient along the segment and use $\sum_j\|\partial_j g_i^0\|_\infty\le\lambda$. We use $C^\alpha$ (or $\alpha$-H\"older) for the full smoothness of \eqref{eq:holder-convention}, and $\beta$ for function-value oscillations.

\paragraph{Conventions.}
Throughout the appendix, $C,C',c,c',\ldots$ denote positive constants that depend only on the structural parameters $\theta=(d,\alpha,\lambda,\cmin,\gamma,\lstar,\lbar)$ and on $\sigma$, and never on $n$, $M$, $\rho$; their values may change from line to line.
We fix a reward $f\in\scF_{M,\theta}$ and suppress the prompt index $m$ when a single prompt is considered.
We use the following notation:
\begin{align*}
    &c_i=\frac{i^{-\gamma}}{\zeta(\gamma)},\qquad
    T_k=\sum_{i>k}c_i,\qquad
    L_i=\sum_{a\le i}\ell_a,\qquad
    E_a(m)=C_{a-1}(m)\setminus C_a(m),\\
    &f^{[k]}=\sum_{i\le k}c_ih_i,\qquad
    G_m=\max_yf(y;m),\qquad
    r_\star=2^{-\beta\lstar},\qquad
    \chi=3^\gamma r_\star,\qquad
    B_0=2^{d\lbar},\\
    &\kappa=\frac{\cmin}2,\qquad
    K=\rho^{-1/\gamma},\qquad
    \Lf=\log(e+n+M+\rho^{-1}),\\
    &p=\frac{2\alpha}{2\alpha+d},\qquad
    p_0=\frac{\gamma-1}{2\gamma+1}.
\end{align*}
The abbreviations $p$ and $p_0$ for the two exponents are used in the appendix only (they are recalled where needed).
Both upper bounds measure the accuracy of a score at the depth of the cube that carries it, and both pay a power of $M$ in the batch sizes; we write
\begin{align}
    a_j=c_{1+\lfloor j/\lbar\rfloor}\ \ (j\ge0),\qquad
    \zeta_M=2+\frac d{2\alpha},\qquad
    \ell_N=1+\log(N+1),
    \label{eq:a_j}
\end{align}
for the known scale attached to a depth $j$ (so that $a_j\asymp c_{k(j)}\asymp(j+1)^{-\gamma}$ for $k(j)=\min\{i:j\le L_i\}$, by \assmp{A3}), the exponent of $M$ in the budget, and a logarithmic factor of an approximation size $N$.

\Cref{tab:notation-global} collects the symbols of this list that recur throughout the appendix, and \cref{tab:notation-upper} the quantities that the two upper bounds share.
The schedules that belong to one proof only are listed in \cref{tab:notation-large} for the large-budget bound of \cref{thm:upper} and in \cref{app:small-params} for the small-regularization bound of \cref{thm:upper}.

\begin{table}[h]
\caption{Notation used throughout the appendix.}
\label{tab:notation-global}
\begin{center}
\begin{tabular}{ll}
\hline
Symbol & Meaning \\
\hline
$c_i=i^{-\gamma}/\zeta(\gamma)$ & weight of level $i$ of the hierarchy \\
$T_k=\sum_{i>k}c_i$ & weight of the levels deeper than $k$ \\
$L_i=\sum_{a\le i}\ell_a$ & depth at which level $i$ is fully resolved \\
$C_i(m)$, $E_a(m)$ & active cell of level $i$; shell $C_{a-1}\setminus C_a$ \\
$f^{[k]}=\sum_{i\le k}c_ih_i$ & reward truncated at level $k$ \\
$G_m=\max_yf(y;m)$ & optimal reward of prompt $m$ \\
$Q(m,D)$ & cube average \eqref{eq:oracle} queried by the oracle \\
$\kappa=\cmin/2$ & gap constant of \cref{lem:gap} \\
$r_\star$, $\chi$, $B_0$ & zoom contraction, its $\gamma$-weighted version, branching \\
$p=\frac{2\alpha}{2\alpha+d}$, $p_0=\frac{\gamma-1}{2\gamma+1}$ & the two exponents of \cref{thm:upper} \\
\hline
\end{tabular}
\end{center}
\end{table}

\begin{table}[h]
\caption{Notation shared by the two upper bounds.}
\label{tab:notation-upper}
\begin{center}
\begin{tabular}{lll}
\hline
Symbol & Meaning & Defined in \\
\hline
$a_j=c_{1+\lfloor j/\lbar\rfloor}$ & accuracy scale attached to depth $j$ & \eqref{eq:a_j} \\
$\zeta_M=2+\frac d{2\alpha}$ & exponent of $M$ in the batch sizes & \eqref{eq:a_j} \\
$W_h$, $\scD_h(m)$ & score at phase $h$; partition on which it is constant & \cref{lem:mask} \\
$\eta$ & birth-accuracy level & \cref{def:birth} \\
$P^\ast_h(m)$, $k(h)$ & correct depth-$h$ cube; the level it has reached & \cref{app:gibbs-mass} \\
$p_\ast$ & Gibbs mass of the correct cube & \cref{lem:mass} \\
$\vartheta_h$, $\kappa_0$ & temperature schedule; mass threshold & \cref{app:gibbs-mass} \\
$S$ & lookahead: query depth minus phase & \cref{subsec:algorithm} \\
$H$ & number of front-phase iterations & \cref{app:parameters,app:small-params} \\
$\varepsilon_h=\frac\eta4a_{h+S}$ & accuracy required of the critic of phase $h$ & \eqref{eq:event-E} \\
$\scW_h$, $K_w$, $J$, $N$ & query window; its size, length, approximation size & \cref{lem:window} \\
$p_S$ & actor mass of a single fine query & \cref{lem:fine} \\
$b_h$, $\delta_h$ & batch size and failure budget of phase $h$ & \cref{app:parameters,app:small-params} \\
$\scE$ & event that all front-phase critics are accurate & \eqref{eq:event-E} \\
\hline
\end{tabular}
\end{center}
\end{table}
By \assmp{A3},
\begin{align}
    \chi\le\frac12,\qquad \lambda\chi\le\frac{\cmin}8,
    \label{eq:chi}
\end{align}
and for $k\ge1$,
\begin{align}
    c_Tk^{-(\gamma-1)}\le T_k\le C_Tk^{-(\gamma-1)},\qquad
    c_T=\frac{2^{-(\gamma-1)}}{(\gamma-1)\zeta(\gamma)},\quad
    C_T=\frac1{(\gamma-1)\zeta(\gamma)}.
    \label{eq:tail}
\end{align}
For a cube $D$ we write $Q(D)=Q(m,D)$ for the cube average \eqref{eq:oracle}; the cube average always carries its arguments, and the cubes $Q_0,Q_1,Q_\flat$ of the lower-bound constructions never do.
The letters $D,L,W,S,T,B$ that name the parameters of the Transformer class in \cref{subsec:transformer} all denote other objects in this appendix ($D$ is a cube, $L_i$ a depth, $W_h$ a score, $S$ the lookahead, $T_k$ a tail weight and $B_0$ a branching number), so the class is written here as
\begin{align*}
    \Tcls\qty(\Nblk,D_{\rm emb},L_{\rm FFN},W_{\rm FFN},S_{\rm FFN},B_{\rm par})
    \quad\text{on inputs of length }\Tin,
\end{align*}
where $D_{\rm emb}$ is the embedding dimension, $L_{\rm FFN},W_{\rm FFN},S_{\rm FFN}$ the depth, width and sparsity of the FNN layers (also called FFN layers here), $B_{\rm par}$ the bound on the parameters and $\Tin$ the input length; $\Nblk$ is the number of blocks, as in the main text.
Boundaries of dyadic cubes are $\mu$-null sets and are assigned by the fixed convention of \cref{subsec:reward}; all statements below are understood up to such null sets.

\subsection{Geometry of Hierarchical Rewards}
\label{app:geometry}

The loss of a response is comparable to the tail weight of the level at which it leaves the correct path: \eqref{eq:geo-lower} makes exploration necessary and \eqref{eq:geo} makes reaching a deep cell sufficient.
\begin{lemma}[Basic geometry]\label{lem:geometry}
    Under \cref{assmp:main}, for every $m$, the series \eqref{eq:reward} converges uniformly, $f(\cdot;m)\in[-1,1]$, and $\bigcap_{i\ge0}C_i(m)=\{\ttrue(m)\}$ is the unique maximizer of $f(\cdot;m)$.
    Moreover, with $C_{\rm near}=2+\lambda2^\gamma r_\star/(1-\chi)$,
    \begin{align}
        G_m-f(y;m)&\ge\kappa T_i&&(y\in E_i(m)),
        \label{eq:geo-lower}\\
        0\le G_m-f(y;m)&\le C_{\rm near}T_k&&(y\in C_k(m)).
        \label{eq:geo}
    \end{align}
\end{lemma}
\begin{proof}
    Uniform convergence and $f\in[-1,1]$ follow from $\abs{h_i}\le1$ and $\sum_ic_i=1$.

    Convergence and boundedness are immediate from $\abs{h_i}\le1$ and $\sum_ic_i=1$.
    For the two gaps we compare an arbitrary response with the maximizer: the levels above the exit contribute nothing, the levels below it contribute at least a fixed fraction of their weight by \assmp{A2}, and the levels already resolved contribute an error that the zoom condition \assmp{A3} keeps below a fraction of the gain; summing the resulting geometric series gives \eqref{eq:geo-lower}, and the same comparison inside a cell gives \eqref{eq:geo}.

    \emph{Lower bound on the loss outside the cells.}
    Let $z=\ttrue(m)\in\bigcap_iC_i(m)$ and $y\in E_a(m)$.
    For every $i$ we have $z\in C_i(m)$, i.e., $z$ lies in the cube of relative depth $\ell_i$ that contains the maximizer of $g_{i,m}$, so by \assmp{A1}, \assmp{A2} and \eqref{eq:chi},
    \begin{align}
        h_i(z;m)\ge g_{i,m}(z^\star_{i,m})-\lambda r_\star\ge\cmin-\frac{\cmin}8=\frac{7\cmin}8\qquad(i\ge1);
        \label{eq:geo-proof-1}
    \end{align}
    the same bound holds for $i\le k$ and any $z\in C_k(m)$, which is the form used in \cref{lem:gap}.
    For $i>a$ we have $h_i(y;m)=0$.
    For $i<a$, both $y$ and $z$ lie in $C_{a-1}(m)$, which is a cube of relative side at most $2^{-\lstar(a-i)}$ in the coordinates of $g_{i,m}$, so that
    \begin{align*}
        \abs{h_i(z;m)-h_i(y;m)}\le\lambda r_\star^{a-i}\qquad(i<a),
    \end{align*}
    while for $i=a$ we have $h_a(y;m)\le g^\star_{a,m}\le h_a(z;m)+\lambda r_\star$.
    Since $(a+1)/i\le3^{a-i}$ for $i<a$, we obtain
    \begin{align}
        \lambda\qty(c_ar_\star+\sum_{i<a}c_ir_\star^{a-i})
        \le\lambda c_{a+1}\qty(2^\gamma r_\star+\sum_{b\ge1}3^{\gamma b}r_\star^b)
        =\lambda c_{a+1}\qty(2^\gamma r_\star+\frac{\chi}{1-\chi})
        \le\frac{3\cmin}8c_{a+1}.
        \label{eq:geo-proof-2}
    \end{align}
    Combining \eqref{eq:geo-proof-1} and \eqref{eq:geo-proof-2},
    \begin{align*}
        f(z;m)-f(y;m)
        \ge\frac{7\cmin}8\sum_{i=a+1}^\infty c_i-\frac{3\cmin}8c_{a+1}
        \ge\frac{\cmin}2T_a=\kappa T_a.
    \end{align*}
    Since $G_m\ge f(z;m)$, this gives \eqref{eq:geo-lower}, and shows that no $y\notin\bigcap_iC_i(m)$ is a maximizer.

    \emph{Upper bound on the loss inside the cells.}
    If $y\in C_k(m)$ and $z\in C_{k'}(m)$ with $k'\ge k$, then by the same geometric series and $\abs{h_i}\le1$,
    \begin{align*}
        f(z;m)-f(y;m)
        \le\sum_{i\le k}c_i\lambda r_\star^{k-i+1}+2T_k
        \le C_{\rm near}T_k,
    \end{align*}
    which gives \eqref{eq:geo}; the maximum is attained at the common point of the cells by continuity.
\end{proof}

A query at depth $h$ sees only the first $\max\{1,\lceil h/\lstar\rceil\}$ levels: a deeper primitive lies either inside the queried cube, where it has mean zero \assmp{A4}, or outside it.
\begin{lemma}[Cube averages resolve only finitely many levels]\label{lem:cube-average}
    Let $D$ be a dyadic cube of depth $h$ and $K_h=\max\{1,\lceil h/\lstar\rceil\}$. Then
    \begin{align}
        Q(m,D)=\frac1{\mu(D)}\int_Df^{[K_h]}(y;m)\dd{\mu}(y).
        \label{eq:cube-average}
    \end{align}
    More generally, for any $i\ge2$ with $L_{i-1}\ge h$, either $C_{i-1}(m)\subset D$ or the interiors of $C_{i-1}(m)$ and $D$ are disjoint, and in both cases $\int_Dh_i(\cdot;m)\dd{\mu}=0$.
\end{lemma}
\begin{proof}
    If $L_{i-1}\ge h$, then $C_{i-1}(m)$ is at least as deep as $D$; since two dyadic cubes are either nested or have disjoint interiors, it is contained in $D$ or disjoint from it.
    In the former case, by \assmp{A4},
    \begin{align*}
        \int_Dh_i(\cdot;m)\dd{\mu}=\mu(C_{i-1}(m))\int_\Omega g_{i,m}\dd{\mu}=0.
    \end{align*}
    Since $L_{i-1}\ge\lstar(i-1)\ge h$ for $i>K_h$, \eqref{eq:cube-average} follows.
\end{proof}

Leaving the correct path at level $a$ costs the whole weight $\sum_{i>a}c_i$ of the levels below $a$, up to the resolution $c_k$ of the current depth.
\begin{lemma}[Gap between the correct cube and the others]\label{lem:gap}
    Fix $m$ and let $\scD$ be a finite partition of $\Omega$ into dyadic cubes.
    Let $D_\ast\in\scD$ be the cube containing $\ttrue(m)$, let $j\ge1$ be its depth and
    \begin{align*}
        k=\min\{i\ge1:j\le L_i\},\qquad\text{so that}\qquad C_k(m)\subset D_\ast\subset C_{k-1}(m).
    \end{align*}
    Then every $D\in\scD\setminus\{D_\ast\}$ satisfies either $D\subset E_a(m)$ for a unique $a\le k-1$, or $D\subset C_{k-1}(m)\setminus D_\ast$, and
    \begin{align}
        Q(D_\ast)-Q(D)&\ge\kappa\sum_{i=a+1}^kc_i-3c_k&&(D\subset E_a(m),\ a<k),
        \label{eq:gap}\\
        Q(D_\ast)-Q(D)&\ge-3c_k&&(D\subset C_{k-1}(m)\setminus D_\ast).
        \label{eq:gap-inner}
    \end{align}
    Moreover,
    \begin{align}
        \sup_{u,v\in C_{k-1}(m)}\abs{f^{[k]}(u;m)-f^{[k]}(v;m)}\le3c_k.
        \label{eq:osc-cell}
    \end{align}
\end{lemma}
\begin{proof}
    If $D\ne D_\ast$ contained $C_a(m)$ for some $a\le k-1$, it would contain $D_\ast$, contradicting disjointness; hence, for each $a\le k-1$, $D$ is inside $C_a(m)$ or disjoint from it, which gives the classification.
    By \cref{lem:cube-average}, $Q(D_\ast)$ is the average of $f^{[k]}$ over $D_\ast$; moreover $f=f^{[a]}=f^{[k]}$ on $E_a(m)$ and $f=f^{[k]}$ on $C_{k-1}(m)\setminus C_k(m)$.
    The proof of \cref{lem:geometry} gives, for $z\in C_k(m)$ and $y\in E_a(m)$,
    \begin{align*}
        f^{[k]}(z;m)-f^{[k]}(y;m)\ge\kappa\sum_{i=a+1}^kc_i,
    \end{align*}
    and, for $u,v\in C_{k-1}(m)$,
    \begin{align*}
        \abs{f^{[k]}(u;m)-f^{[k]}(v;m)}\le2c_k+\lambda\sum_{i<k}c_ir_\star^{k-i}\le3c_k,
    \end{align*}
    which is \eqref{eq:osc-cell}.
    Since $Q(D_\ast)$ is an average over a subset of $C_{k-1}(m)$, \eqref{eq:osc-cell} gives $Q(D_\ast)\ge f^{[k]}(z;m)-3c_k$ for every $z\in C_k(m)$, and averaging the first display over $y\in D$ yields \eqref{eq:gap} and \eqref{eq:gap-inner}.
\end{proof}

Inside a cube that the hierarchy has not yet split, cube averages at a fixed depth differ by at most the weight of the current level; \eqref{eq:osc-shell} is the pointwise version on a shell.
\begin{lemma}[Oscillation of cube averages at finite depth]\label{lem:oscillation}
    Let $A$ be a dyadic cube of depth $h$, let $k(h)=\min\{i\ge1:h\le L_i\}$ (so that $L_{k(h)-1}<h\le L_{k(h)}$ for $h\ge1$), and let $R\ge0$.
    For all dyadic cubes $D,D'\subset A$ of depth at most $h+R+1$,
    \begin{align}
        \abs{Q(D)-Q(D')}\le C(R+1)c_{k(h)}.
        \label{eq:osc}
    \end{align}
    Moreover, if $A\subset E_i(m)$ for some $i$, then for every dyadic cube $D\subset A$ and $y\in D$,
    \begin{align}
        \abs{f(y;m)-Q(D)}\le C_{\rm osc}c_i.
        \label{eq:osc-shell}
    \end{align}
\end{lemma}
\begin{proof}
    A level $i$ with $L_{i-1}<h$ varies on $A$ by at most $\lambda c_i2^{-\beta(h-L_{i-1})}$ by \assmp{A1}.
    Since $c_i/c_{k(h)}\le(k(h)-i+1)^\gamma$ and $h-L_{i-1}\ge\lstar(k(h)-i)$, the sum over such levels is at most
    \begin{align*}
        \lambda c_{k(h)}\sum_{r'\ge0}(r'+1)^\gamma2^{-\beta\lstar r'}\le Cc_{k(h)}.
    \end{align*}
    Levels with $h\le L_{i-1}<h+R+1$ number at most $(R+1)/\lstar+1$, and each has amplitude at most $c_i\le c_{k(h)}$.
    Deeper levels do not contribute to $Q(D)$ by \cref{lem:cube-average}.
    This proves \eqref{eq:osc}; \eqref{eq:osc-shell} follows in the same way, using that on $E_i(m)$ only the levels $\le i$ are active and $c_a/c_i\le(i-a+1)^\gamma$ for $a\le i$.
\end{proof}

On a shell the reward is a finite sum of primitives, hence $\alpha$-H\"older in the $d$ local coordinates of that shell, with a seminorm of the order of $a_j$.
\begin{lemma}[The reward on a shell is H\"older in local coordinates]\label{lem:holder-shell}
    Let $P^\ast_j=b_j(m)+2^{-j}\Omega$ be the depth-$j$ cube on the correct path, $A_j=P^\ast_j\setminus P^\ast_{j+1}$, and $i(j)=\min\{i:L_i>j\}$. On $A_j$, $f=f^{[i(j)]}$, and the function
    \begin{align*}
        g_{m,j}(z)=f^{[i(j)]}\qty(b_j(m)+2^{-j}z;m),\qquad z\in\bR^d,
    \end{align*}
    satisfies $\abs{g_{m,j}}\le1$ and
    \begin{align}
        \|g_{m,j}-g_{m,j}(0)\|_{C^\alpha(\Omega)}+[g_{m,j}]_\beta\le Cc_{i(j)}\le Ca_j,
        \label{eq:holder-shell}
    \end{align}
    with $a_j$ as in \eqref{eq:a_j}.
\end{lemma}
\begin{proof}
    Levels $i>i(j)$ are supported in $P^\ast_{j+1}$, hence vanish on $A_j$.
    Put $t_a=2^{-(j-L_{a-1})}\le1$ for $a\le i(j)$.
    By \assmp{A1}, a derivative of order $1\le|s|\le k$ of the rescaled primitive has sup norm at most $\lambda t_a^{|s|}$, its order-$k$ derivative has $\tau$-H\"older seminorm at most $\lambda t_a^{\alpha}$, and its function-value $\beta$-H\"older seminorm is at most $\lambda t_a^\beta$.
    Each exponent is at least $\beta$, and subtraction of its value at zero also bounds its sup norm on $\Omega$ by $\lambda t_a^\beta$.
    Thus, by the same geometric series as in \cref{lem:oscillation},
    \begin{align*}
        \|g_{m,j}-g_{m,j}(0)\|_{C^\alpha(\Omega)}+[g_{m,j}]_\beta
        \le C\lambda\sum_{a\le i(j)}c_a2^{-\beta(j-L_{a-1})}\le Cc_{i(j)}.
    \end{align*}
    Also, $c_{i(j)}\le Ca_j$ because the indices $i(j)\le1+\lfloor j/\lstar\rfloor$ and $1+\lfloor j/\lbar\rfloor$ differ by at most the bounded factor $\lceil2\lbar/\lstar\rceil+1$.
\end{proof}

\subsection{Gibbs Mass of the Correct Cube under Birth-Accurate Scores}
\label{app:gibbs-mass}

Throughout, $P^\ast_h(m)=D_h(\ttrue(m))$ denotes the correct depth-$h$ cube and $k(h)=\min\{i\ge1:h\le L_i\}$, so that $C_{k(h)}(m)\subset P^\ast_h(m)\subset C_{k(h)-1}(m)$; we write $k(0)=1$.
The scores of \cref{alg:actor-critic} are analyzed through the following notion.
A cube created early is only ever estimated at the coarse scale of its own depth, so the accuracy the analysis can require of a score is scale-dependent.
\begin{definition}[Birth-accurate score]\label{def:birth}
    Let $\scD$ be a finite partition of $\Omega$ into dyadic cubes and let $W$ be constant on the cubes of $\scD$.
    We say that $W$ is \emph{$\eta$-birth-accurate on $\scD$} if
    \begin{align}
        \abs{W(D)-Q(m,D)}\le\frac\eta4\,a_{\mathrm{depth}(D)}\qquad\text{for all }D\in\scD.
        \label{eq:birth}
    \end{align}
\end{definition}
The accuracy allowed on a cube is tied to the depth of the cube, i.e., to the scale of the level that a query at this depth resolves, and not to a common final accuracy.

\begin{lemma}[Gibbs mass and cold tail]\label{lem:mass}
    Fix a prompt $m$ and a phase $h\ge1$, and let $\scD_h$ be the dyadic leaf partition on which the current score $W_h$ is constant, assumed to be such that $P^\ast_h(m)$ is a union of cubes of $\scD_h$.
    Let $k=k(h)$ and suppose that $W_h$ is $\eta$-birth-accurate on $\scD_h$ in the sense of \cref{def:birth}.
    Suppose that
    \begin{align}
        0<t_-\le\frac{\vartheta}{c_k}\le t_+,\qquad B_0\exp\qty(-\frac\kappa{2t_+})\le\frac12,\qquad \eta\le\eta_0,
        \label{eq:temp-cond}
    \end{align}
    where $\eta_0$ is a structural constant.
    Then, uniformly in the partition, in $h$ and in the zoom sequence, the actor $P_h=P_{W_h,\vartheta}$ of \cref{alg:actor-critic}, line~\ref{ln:actor}, satisfies
    \begin{align}
        P_{W_h,\vartheta}\qty(P^\ast_h(m)\relmiddle|m)\ge p_\ast:=\qty[1+2B_0e^{C_{\rm gap}/t_-}]^{-1},
        \label{eq:mass}\\
        P_{W_h,\vartheta}\qty(\qty(P^\ast_{h-j}(m))^c\relmiddle|m)\le Ce^{-cj}\qquad(0\le j\le h),
        \label{eq:cold-tail}
    \end{align}
    with a structural constant $C_{\rm gap}$.
    For $h=0$, $P^\ast_0=\Omega$ has mass one.
\end{lemma}
\begin{proof}
    The proof compares the Gibbs weight inside the correct cube with the total Gibbs weight outside it, through the identity
    \begin{align*}
        \frac1{P_{W_h,\vartheta}(P^\ast_h\mid m)}
        =1+\frac{\int_{\Omega\setminus P^\ast_h}e^{W_h/\vartheta}\dd{\mu}}{\int_{P^\ast_h}e^{W_h/\vartheta}\dd{\mu}}.
    \end{align*}
    The step \emph{Inside} lower-bounds the denominator by Jensen's inequality, and the step \emph{Outside} upper-bounds the numerator shell by shell, using the reward gap of \cref{lem:gap} and the birth accuracy of the score; combining the two gives the constant lower bound \eqref{eq:mass}, and summing the same geometric series over the deep shells only gives the cold tail \eqref{eq:cold-tail}.

    \emph{Inside: lower bound on the Gibbs weight of the correct cube.}
    The cubes $D\in\scD_h$ inside $P^\ast_h$ have depth at least $h$, so $a_{\mathrm{depth}(D)}\le a_h\le Cc_k$ and $W_h(D)\ge Q(D)-C\eta c_k$ on them.
    Since $Q(P^\ast_h)$ is the volume-weighted average of the $Q(D)$, $D\subset P^\ast_h$ (tower identity), Jensen's inequality gives
    \begin{align}
        \int_{P^\ast_h}e^{W_h/\vartheta}\dd{\mu}
        =\sum_{D\subset P^\ast_h}\mu(D)e^{W_h(D)/\vartheta}
        \ge\mu(P^\ast_h)\exp\qty(\frac{Q(P^\ast_h)-C\eta c_k}{\vartheta});
        \label{eq:mass-inside}
    \end{align}
    no constancy of $W_h$ inside $P^\ast_h$ is needed.

    \emph{Outside: upper bound on the Gibbs weight of the outer shells.}
    Apply \cref{lem:gap} to the partition $\{P^\ast_h\}\cup\{D\in\scD_h:D\cap P^\ast_h=\emptyset\}$, in which $P^\ast_h$ is the cube containing $\ttrue(m)$: every $D\in\scD_h$ disjoint from $P^\ast_h$ satisfies either $D\subset E_a(m)$ for a unique $a<k$, or $D\subset C_{k-1}(m)\setminus P^\ast_h$, and \eqref{eq:gap}, \eqref{eq:gap-inner} hold with $D_\ast=P^\ast_h$.
    A cube $D\subset E_a(m)$ is strictly inside $C_{a-1}(m)$, so $\mathrm{depth}(D)\ge L_{a-1}+1$ and $a_{\mathrm{depth}(D)}\le Cc_a$ by \assmp{A3}; similarly $a_{\mathrm{depth}(D)}\le Cc_k$ for $D\subset C_{k-1}\setminus P^\ast_h$.
    Hence, by \eqref{eq:birth}, \eqref{eq:gap} and $c_a\le2^\gamma c_{a+1}\le2^\gamma\sum_{i=a+1}^kc_i$, for $D\subset E_a(m)$,
    \begin{align*}
        Q(P^\ast_h)-C\eta c_k-W_h(D)
        &\ge\kappa\sum_{i=a+1}^kc_i-3c_k-C\eta(c_a+c_k)\\
        &\ge\frac\kappa2\sum_{i=a+1}^kc_i-C_{\rm gap}c_k
        \ge\qty[\frac\kappa2(k-a)-C_{\rm gap}]c_k,
    \end{align*}
    once $\eta\le\eta_0$ with $\eta_0$ small (the old errors are absorbed into half of the gap), and the same bound with the right-hand side $-C_{\rm gap}c_k$ holds for $D\subset C_{k-1}\setminus P^\ast_h$.
    By $h\le L_k$ and \assmp{A3}, the volumes satisfy $\mu(E_a)/\mu(P^\ast_h)\le B_0^{k-a+1}$ and $\mu(C_{k-1})/\mu(P^\ast_h)\le B_0$.
    Summing the volumes of the cubes in each shell and using \eqref{eq:mass-inside} and \eqref{eq:temp-cond},
    \begin{align*}
        \frac1{P_{W_h,\vartheta}(P^\ast_h\mid m)}
        &\le1+\sum_{D\cap P^\ast_h=\emptyset}\frac{\mu(D)}{\mu(P^\ast_h)}\exp\qty(-\frac{Q(P^\ast_h)-C\eta c_k-W_h(D)}{\vartheta})\\
        &\le1+B_0e^{C_{\rm gap}/t_-}\qty[1+\sum_{b'=1}^{k-1}\qty(B_0e^{-\kappa/(2t_+)})^{b'}]\\
        &\le1+2B_0e^{C_{\rm gap}/t_-},
    \end{align*}
    which is \eqref{eq:mass}.
    For \eqref{eq:cold-tail}, a point outside $P^\ast_{h-j}$ lies in a shell $E_a$ with $L_{a-1}<h-j$, hence $L_a<h-j+\lbar$ and $k-a\ge j/\lbar-1$; summing the same geometric series only over these shells gives $P_{W_h,\vartheta}((P^\ast_{h-j})^c\mid m)\le C\sum_{b'\ge j/\lbar-1}2^{-b'}\le Ce^{-cj}$.
\end{proof}

\paragraph{The temperature schedule satisfies \eqref{eq:temp-cond}.}
Both hyperparameter settings use the same schedule $\vartheta_h=t_0\,c_{1+\lfloor h/\lstar\rfloor}$ (\cref{app:parameters,app:small-params}), and we check here, once for both, that it meets the condition \eqref{eq:temp-cond} of \cref{lem:mass}.
Let $D_0=\lceil2\lbar/\lstar\rceil+1$.
If depth $h$ corresponds to the level $k(h)$, then $1+\lfloor h/\lstar\rfloor$ and $k(h)$ differ by at most a factor $D_0$, so
\begin{align}
    \frac{\vartheta_h}{c_{k(h)}}\in\qty[t_02^{-\gamma}D_0^{-\gamma},\ 2^\gamma t_0].
    \label{eq:temp-range}
\end{align}
We fix $t_0$ so small that $B_0\exp(-\kappa/(2^{\gamma+1}t_0))\le1/2$; then \cref{lem:mass} applies at every step with $t_+=2^\gamma t_0$ and $t_-=t_02^{-\gamma}D_0^{-\gamma}$, and $p_\ast$ is a structural constant (independent of $S$, $n$, $M$, $\rho$).

\subsection{Window-Restricted Transformer Approximation and Estimation}
\label{app:transformer}

The critic is fitted on data collected by the actor, and the two upper bounds both need the same chain of arguments to turn this into a pointwise accuracy statement:
\begin{align*}
\text{windowed target decomposition}
\;\to\;
\text{Transformer realization}
\;\to\;
\text{covering number}
\;\to\;
\text{risk}.
\end{align*}
\Cref{lem:position} shows that a causal Transformer can read the token positions that identify a queried cube; \cref{lem:covering} bounds the covering number of the full class; \cref{lem:oracle} converts an approximation error and a covering number into a bound on the conditional risk of the empirical minimizer; \cref{lem:window-target} shows that, on the window that the actor actually queries, the infinite hierarchical target reduces to finitely many $d$-dimensional H\"older functions; and \cref{lem:window} realizes approximations of those functions inside the Transformer class.
How the resulting risk bound is turned into accuracy at every updated node is not a generic estimation question and is treated where it is used (\cref{lem:joint}).

\paragraph{Position separation by the sinusoidal encoding.}
Take $D_{\rm emb}$ to be a multiple of $16$.
The coordinates $r'=0$ and $r'=D_{\rm emb}/16$ of the encoding contain
\begin{align}
    \varphi(t)=\qty(\sin t,\ \cos t,\ \sin(t/\sqrt{10}),\ \cos(t/\sqrt{10})),
    \label{eq:varphi}
\end{align}
since $10000^{2(D_{\rm emb}/16)/D_{\rm emb}}=\sqrt{10}$.
The constructions below read the bits of a queried cube by attention, which requires the fixed sinusoidal encoding to separate positions well enough.
\begin{lemma}[Position separation]\label{lem:position}
    For integers $1\le r'\le \Tin$,
    \begin{align}
        \Gamma(r'):=(1-\cos r')+\qty(1-\cos(r'/\sqrt{10}))\ge\frac1{40\Tin^2}.
        \label{eq:position}
    \end{align}
    Consequently, for integer positions $s,t\in[1,\Tin]$,
    \begin{align*}
        \ev{\varphi(s),\varphi(t)}=2\ \ (s=t),\qquad
        \ev{\varphi(s),\varphi(t)}\le2-\frac1{40\Tin^2}\ \ (s\ne t),
    \end{align*}
    and a softmax attention with query $a_T\varphi(s)$, $a_T=40\Tin^2\log(8\Tin)$, and keys $\varphi(t)$ puts weight at least $7/8$ on position $s$.
\end{lemma}
\begin{proof}
    Let $k,l$ be the nearest integers to $r'/(2\pi)$ and $r'/(2\pi\sqrt{10})$, and put $u=r'-2\pi k$, $v=r'/\sqrt{10}-2\pi l$, so that $\abs{u},\abs{v}\le\pi$ and $\abs{k},\abs{l}\le r'$.
    If $(k,l)\ne(0,0)$, then $k^2-10l^2$ is a nonzero integer, hence
    \begin{align*}
        \abs{k-\sqrt{10}\,l}\ge\frac1{\abs{k}+\sqrt{10}\,\abs{l}}\ge\frac1{5r'}.
    \end{align*}
    By the Cauchy--Schwarz inequality, $u^2+v^2\ge4\pi^2/(275r'^2)$, and $1-\cos x\ge2x^2/\pi^2$ on $[-\pi,\pi]$ gives
    \begin{align*}
        \Gamma(r')\ge\frac{2(u^2+v^2)}{\pi^2}\ge\frac8{275r'^2}\ge\frac1{40\Tin^2}.
    \end{align*}
    If $k=l=0$, then $u=r'\ge1$ and the bound is immediate.
    The softmax claim follows because the logit gap is at least $a_T/(40\Tin^2)=\log(8\Tin)$ and there are at most $\Tin$ positions.
\end{proof}
A ReLU ``plateau'' $r(z)=2(z-1/4)_+-2(z-3/4)_+$ recovers a one-hot vector exactly from a vector whose correct coordinate is $\ge7/8$ and whose other coordinates are $\le1/8$.
We also use the exact ReLU gate
\begin{align}
    G(z,c)=(z+C_{\rm g}(c-1))_+-(-z+C_{\rm g}(c-1))_+=cz\qquad(c\in\{0,1\},\ \abs{z}\le C_{\rm g}),
    \label{eq:gate}
\end{align}
and the identity $h=(h)_+-(-h)_+$ to copy the coordinates that an FFN must preserve (the FFNs have no residual connection, so every FFN reproduces the part of the state that later blocks need; this at most doubles the width).

The complexity of the critic is measured by the covering entropy of the whole class of \cref{subsec:transformer}, all parameters being free; no sub-class adapted to the unknown reward is used.
\begin{lemma}[Covering number of the full class]\label{lem:covering}
    Let $\Tcls=\Tcls(\Nblk,D_{\rm emb},L_{\rm FFN},W_{\rm FFN},S_{\rm FFN},B_{\rm par})$ with $B_{\rm par}\ge1$ and $W_{\rm FFN}\ge D_{\rm emb}$, let $P$ be its number of parameters, and let $A_0=C(D_{\rm emb}+W_{\rm FFN}+1)(B_{\rm par}+1)$.
    For $0<\varepsilon\le1$, the $\varepsilon$-covering number of $\Tcls$ in the sup norm over all valid inputs of length $\Tin$ satisfies
    \begin{align}
        \log\scN_\infty(\varepsilon,\Tcls)
        \le\Nblk S_{\rm FFN}\log\qty(2L_{\rm FFN}(W_{\rm FFN}+1)^2)+P\log\qty(1+\frac{2B_{\rm par}A_0^{C\Nblk^2(L_{\rm FFN}+2)}}{\varepsilon}).
        \label{eq:covering}
    \end{align}
\end{lemma}
\begin{proof}
    There are at most $(2L_{\rm FFN}(W_{\rm FFN}+1)^2)^{S_{\rm FFN}}$ sparsity patterns per FFN.
    For a fixed pattern, quantize all parameters to a grid of step $\varrho$, which gives $(1+2B_{\rm par}/\varrho)^P$ networks.
    Since one-hot inputs and the positional encoding are bounded, the embedded states are bounded by $2B_{\rm par}$.
    If the states before a block are bounded by $R_0$, they are bounded by $(1+D_{\rm emb}B_{\rm par})R_0$ after the attention (its output is a convex combination of values) and by $A_0^{C(L_{\rm FFN}+2)}(R_0+1)$ after the FFN, so all states are bounded by $A_0^{C\Nblk(L_{\rm FFN}+2)}$.
    The softmax is $2$-Lipschitz from the sup norm of the logits to the $\ell_1$ norm of the weights, so the output of one block is $A_0^{C(L_{\rm FFN}+2)}(R_0+1)^C$-Lipschitz in the parameters and in the input; composing the blocks, the output is $A_0^{C\Nblk^2(L_{\rm FFN}+2)}$-Lipschitz in the parameters, and the output clipping is $1$-Lipschitz.
    The input length $\Tin$ does not enter, because the attention weights sum to one.
    Choosing $\varrho=\varepsilon A_0^{-C\Nblk^2(L_{\rm FFN}+2)}$ gives \eqref{eq:covering}.
\end{proof}
This is the argument of \citet[Theorem 5.3]{takakura2023approximation} and \citet[Lemma 5]{schmidt2020nonparametric}, adapted to causal attention, a learned token embedding of size $O(M)$ and the fixed positional encoding.

The batch is collected by the current actor, so the inputs are i.i.d.\ only given the history and the noise is only conditionally sub-Gaussian; the oracle inequality is stated in that form.
\begin{lemma}[ERM oracle inequality under conditionally sub-Gaussian noise]\label{lem:oracle}
    Fix the history at the beginning of a batch.
    Let $U_1,\ldots,U_{n_b}$ be i.i.d.\ inputs with law $\Pi$ and $Z_t=Q(U_t)+\xi_t$ with $\abs{Q}\le1$, where $\xi_t$ is conditionally mean-zero and $\sigma$-sub-Gaussian given the past and $U_t$.
    Let $\scF$ be a class of functions with values in $[-1,1]$ whose sup-norm covering entropy is $V(\varepsilon)$, and let $\hat f$ be any empirical risk minimizer of the squared loss over $\scF$.
    Then, for $n_b\ge2$ and $\delta\in(0,1)$, with probability at least $1-\delta$,
    \begin{align}
        \norm{\hat f-Q}^2_{L^2(\Pi)}
        \le C\inf_{f\in\scF}\norm{f-Q}^2_{L^2(\Pi)}+C(1+\sigma^2)\frac{V(n_b^{-2})+\log(2n_b/\delta)}{n_b},
        \label{eq:oracle-ineq}
    \end{align}
    and the same bound holds in expectation with $\log(2n_b)$ in place of $\log(2n_b/\delta)$.
\end{lemma}
\begin{proof}
    Let $f_0\in\scF$ be fixed, and write $g=f-f_0$, $b_0=f_0-Q$ and $\Pi_{n_b}$ for the empirical measure.
    The basic inequality of ERM reads
    \begin{align}
        \Pi_{n_b}(\hat f-Q)^2\le\Pi_{n_b}(f_0-Q)^2+2\Pi_{n_b}\qty[\xi(\hat f-f_0)].
        \label{eq:basic}
    \end{align}
    The proof proceeds in three steps, starting from the basic inequality for the empirical risk minimizer.
    That inequality for the empirical risk minimizer leaves a noise term, an empirical squared error and an approximation term.
    We bound the noise term uniformly over a covering of the class by a sub-Gaussian maximal inequality, transfer the empirical squared error to its population version by Bernstein's inequality on the same covering, and combine the two with the approximation term.

    \emph{Noise term.}
    For a fixed $g$, the process
    \begin{align*}
        \exp\qty(u'\sum_{t=1}^{n_b}\xi_tg(U_t)-\frac{\sigma^2u'^2}2\sum_{t=1}^{n_b}g(U_t)^2)
    \end{align*}
    is a supermartingale with expectation at most one, because $g(U_t)$ is determined before $\xi_t$ is revealed and the conditional moment generating function of $\xi_t$ is bounded; the noise is not treated as independent of the inputs.
    Applying this to an $n_b^{-2}$-net of $\scF$ with a union bound, and controlling the discretization error by $\max_t\abs{\xi_t}\le C\sigma\sqrt{\log(Cn_b/\delta)}$, we obtain simultaneously for all $f\in\scF$, with probability $1-\delta/2$,
    \begin{align}
        2\Pi_{n_b}[\xi g]\le\frac1{16}\Pi_{n_b}g^2+C\sigma^2\frac{V(n_b^{-2})+\log(C/\delta)}{n_b}.
        \label{eq:noise-term}
    \end{align}
    \emph{Empirical to population.}
    Bernstein's inequality \citep[Theorem 2.10]{boucheron2013concentration} for the i.i.d.\ inputs on the same net, with $g^2\le4$ and $\bE g^4\le4\bE g^2$, gives simultaneously
    \begin{align}
        \Pi g^2\le2\Pi_{n_b}g^2+C\frac{V(n_b^{-2})+\log(C/\delta)}{n_b},
        \qquad
        \Pi_{n_b}b_0^2\le2\Pi b_0^2+C\frac{\log(C/\delta)}{n_b}.
        \label{eq:bernstein}
    \end{align}
    \emph{Conclusion.}
    Combining \eqref{eq:basic}--\eqref{eq:bernstein} with $2b_0g\ge-\frac12g^2-2b_0^2$, absorbing the empirical squared term into the left-hand side and taking the infimum over $f_0$ yields \eqref{eq:oracle-ineq}; the expectation bound follows by integrating the tail, since the loss is bounded by $4$.
\end{proof}

\paragraph{Windows.}
Fix a step $h$ and a lookahead $S\ge1$, and let $J\ge0$ be a window length.
Put $s_h=\max\{0,h-J\}$ and define the \emph{window} of step $h$ as the set of inputs
\begin{align}
    \scW_h=\qty{(m,D):\ s_h\le\mathrm{depth}(D)\le h+S,\ D\subset P^\ast_{s_h}(m)},
    \label{eq:window}
\end{align}
i.e., the prefixes of depth between $s_h$ and $h+S$ that agree with the correct path of their prompt up to depth $s_h$.
The window is common to all prompts and is fixed by the step index; it depends on the unknown cells, and is a device of the proof.
On the window, the cube average is a sum of few continuous functions of local coordinates.

\begin{lemma}[Cube averages on a window are sums of few H\"older functions]\label{lem:window-target}
    Let $K_w\ge C_d(J+S+2)$ with $C_d\ge1+1/\lstar$.
    For $(m,D)\in\scW_h$ of depth $\mathsf q$, write $s=s_h$, $F^{\rm old}_{m,s}=c_1h_1(\cdot;m)+\sum_{i\ge2:\,L_{i-1}<s}c_ih_i(\cdot;m)$, and let $\scI_h(m)=\{i\ge2:\ s\le L_{i-1}<h+S\}$, a set of at most $K_w$ levels.
    Then
    \begin{align}
        Q(m,D)=\frac1{\mu(D)}\int_DF^{\rm old}_{m,s}\dd{\mu}
        +\sum_{i\in \scI_h(m):\ L_{i-1}<\mathsf q,\ D\subset C_{i-1}(m)}\frac{c_i}{\mu(D)}\int_Dh_i(\cdot;m)\dd{\mu}.
        \label{eq:window-target}
    \end{align}
    Moreover, the components of \eqref{eq:window-target} are smooth at the following scales.
    \begin{enumerate}[label=(\alph*),leftmargin=*,itemsep=0mm]
        \item \emph{(Old component.)} In the local coordinates of $P^\ast_s(m)=b_s+2^{-s}\Omega$, the map $z\mapsto F^{\rm old}_{m,s}(b_s+2^{-s}z)$ is $\alpha$-H\"older on $[0,1]^d$ with $|\cdot|_\alpha$ at most $Ca_s$ and $C^\alpha$ norm at most $Ca_s$ after subtracting its value at zero, and $a_s\le CK_w^\gamma a_{h+S}$.
        \item \emph{(Window components.)} In the local coordinates of $C_{i-1}(m)$, the primitive $h_i(\cdot;m)$ is $\alpha$-H\"older with $|\cdot|_\alpha$ at most $\lambda$, and $c_i\le Ca_s$ for every $i\in\scI_h(m)$.
        \item \emph{(Averaging preserves smoothness.)} Fix a depth $\mathsf q$ and a cube $A$ of depth $p\le\mathsf q$ in whose local coordinates a function $g$ is $\alpha$-H\"older with seminorm $|g|_\alpha$. Then the cube average
        \begin{align*}
            z\mapsto\int_{[0,1]^d}g\qty(z+2^{-(\mathsf q-p)}u)\dd{u}
        \end{align*}
        is $\alpha$-H\"older in the lower-left corner $z$ with seminorm at most $|g|_\alpha$. The number of depths $\mathsf q$ occurring in the window is at most $J+S+1\le K_w$.
    \end{enumerate}
\end{lemma}
\begin{proof}
    A level $i\ge2$ with $L_{i-1}\ge\mathsf q$ has support $C_{i-1}(m)$ of depth at least $\mathsf q$, which is either inside $D$ or disjoint from $D$; in both cases its contribution to $Q(m,D)$ vanishes by \assmp{A4} and \cref{lem:cube-average}.
    A level with $s\le L_{i-1}<\mathsf q$ has support of depth $L_{i-1}<\mathsf q$, which either contains $D$ or is disjoint from it; only the former levels contribute, which gives \eqref{eq:window-target}.
    The number of levels with $L_{i-1}\in[s,h+S)$ is at most $(h+S-s)/\lstar+1\le(J+S)/\lstar+1\le K_w$ by \assmp{A3}.
    For (a), the derivative-by-derivative scaling of \cref{lem:holder-shell}, \assmp{A1} and the geometric series in \cref{lem:oscillation} give a $C^\alpha$ norm after subtraction of a constant at most $C\lambda\qty(c_1 2^{-\beta s}+\sum_{i\ge2:\,L_{i-1}<s}c_i2^{-\beta(s-L_{i-1})})\le Cc_{k(s)}\le Ca_s$ (the first level is included with $L_0=0$), and $a_s/a_{h+S}\le C((h+S+1)/(s+1))^\gamma\le C(J+S+1)^\gamma\le CK_w^\gamma$ by \eqref{eq:a_j}.
    (b) is \assmp{A1} together with $c_i\le c_{k(s)}\le Ca_s$ for $L_{i-1}\ge s$.
    For (c), differentiation under the integral gives $D^s\int g(z+vu)\dd u=\int D^s g(z+vu)\dd u$ for $|s|\le k$; averaging contracts each derivative sup norm and each top-order $\tau$-H\"older seminorm. The same holds for the function-value $\beta$-H\"older seminorm, so subtraction of a constant again bounds the full local $C^\alpha$ norm.
\end{proof}
The decomposition uses one coordinate system per level: approximating all levels in the coordinates of a single anchor cube would produce H\"older seminorms of order $2^{\alpha J}$.

\begin{lemma}[Window class: approximation and covering]\label{lem:window}
    For integers $K_w,N,\Tin\ge1$, let $\Tcls_{\rm win}(K_w,N,\Tin)$ be the class $\Tcls(\Nblk,D_{\rm emb},L_{\rm FFN},W_{\rm FFN},S_{\rm FFN},B_{\rm par})$ of \cref{subsec:transformer} with input length $\Tin$ and the following architecture parameters.
    Then the class has the stated capacity, covering entropy and approximation power:
    \begin{enumerate}[label=(\alph*),leftmargin=*,itemsep=0mm]
        \item \emph{(Architecture.)} It suffices to take
    \begin{align}
        \Nblk&\le CK_w,\quad D_{\rm emb}\le CK_w,\quad L_{\rm FFN}\le C\ell_N,\quad
        W_{\rm FFN}\le C\qty[MK_w^2(N+1)+K_w^3],
        \label{eq:window-capacity}\\
        S_{\rm FFN}&\le C\qty[MK_w^2(N+1)\ell_N+K_w^3],\qquad
        B_{\rm par}\le C\qty[M+K_w+\Tin^2\log(8\Tin)+N+1].
        \notag
    \end{align}
        \item \emph{(Covering entropy.)} Its covering entropy satisfies
    \begin{align}
        V_N(u):=\log\scN_\infty(u,\Tcls_{\rm win}(K_w,N,\Tin))\le CMK_w^6(N+1)\ell_N^2\log\frac{CM\Tin K_w(N+1)}u.
        \label{eq:window-covering}
    \end{align}
        \item \emph{(Approximation on the window.)} For every step $h$ with $\Tin\ge d(h+S)+3$ and $K_w\ge C_d(J+S+2)$, there is a comparator $F_{0,h}\in\Tcls_{\rm win}(K_w,N,\Tin)$ (depending on the unknown reward) with
    \begin{align}
        \sup_{(m,D)\in\scW_h}\abs{F_{0,h}(m,D)-Q(m,D)}\le CK_w^{\gamma+1}a_{h+S}N^{-\alpha/d},
        \label{eq:window-bias}
    \end{align}
        \item \emph{($L^2$ form.)} Consequently, for every law $\Pi$ of inputs $(m,D)$ with $\Pi(\scW_h^c)=:\Lambda$,
    \begin{align}
        \norm{F_{0,h}-Q}^2_{L^2(\Pi)}\le CK_w^{2\gamma+2}a_{h+S}^2N^{-2\alpha/d}+4\Lambda.
        \label{eq:window-bias-L2}
    \end{align}
    \end{enumerate}
\end{lemma}
\begin{proof}
    The proof follows the chain announced at the beginning of this subsection, in three steps: first we approximate each of the finitely many H\"older components produced by \cref{lem:window-target} by a ReLU network of size $N$; then we realize the sum of these networks inside the Transformer class, using the fixed positional encoding and \cref{lem:position} to route each record to the component it needs; finally we read off the architecture parameters and apply \cref{lem:covering} to bound the covering entropy.

    \emph{Approximation of the components.}
    By \cref{lem:window-target}, on the window $Q(m,\cdot)$ is a sum of at most $K_w+1$ components, the old sum and those levels of $\scI_h(m)$ whose support contains the query cube, and for each of the at most $K_w$ query depths $\mathsf q$ every component is an $\alpha$-H\"older function of the $d$ local coordinates of the query cube inside its support cube, with $|\cdot|_\alpha$ at most $Ca_s$ and $C^\alpha$ norm at most $Ca_s$ after subtracting its value at zero; the side length $2^{-(\mathsf q-p)}$ is a discrete branch, not a continuous variable.
    Subtracting the value at the origin and dividing by $Ca_s$ sends each component into the unit $\alpha$-H\"older ball of $[0,1]^d$, and by \citet[Theorem 5]{schmidt2020nonparametric} every function of that ball is approximated in sup norm to accuracy $CN^{-\alpha/d}$ by a network of
    \begin{align*}
        \Psi\qty(C\ell_N,\ CN,\ CN\ell_N,\ C)
    \end{align*}
    (the internal precision integer of that theorem is set to $C\log(N+1)$, with constants depending on $\alpha,d$; the theorem covers every $\alpha>0$, including the integer $C^{k,1}$ cases).
    Per prompt there are $O(K_w^2)$ such functions, one for each component and query depth, of which at most $K_w+1$ are active on a given input, so the sum of the rescaled approximations, clipped to $[-1,1]$, satisfies \eqref{eq:window-bias} by (a)--(c) of \cref{lem:window-target}.
    The approximation holds on the whole continuous range of the local coordinates; no interpolation of observed query values is involved.

    \emph{Realization with the fixed positional encoding.}
    Let $\varphi$ be the four coordinates \eqref{eq:varphi} of the positional encoding and $a_T=40\Tin^2\log(8\Tin)$.
    The comparator is assembled from the following blocks.
    \begin{enumerate}[label=(\arabic*),leftmargin=*,itemsep=0mm]
        \item \emph{Token embedding.} The embedding of each token contains its category (prompt, $0$, $1$, $\STOP$, $\PAD$, $\RO$) as constants spaced by at least $4$ in coordinates separate from $\varphi$. Each positional shift has magnitude at most $1$, so the category intervals are disjoint and the first FFN recovers the category exactly by plateaus; it preserves $\varphi$ and initializes unused working coordinates to zero.
        By the gate \eqref{eq:gate} with the prompt category, the same FFN zeroes at all positions other than $1$ the \emph{prompt information}, namely the index $m$, the $d(h+S-s_h)$ bits of the correct path of prompt $m$ at the depths $s_h+1,\ldots,h+S$, and the level-boundary flags $\ind{L_{i-1}=j}$ for the depths $j$ of the window and $i\in\scI_h(m)$, that is, $O(K_w)$ numbers; the positional encoding of position $1$ is subtracted in this embedding.
        These unknown bits and flags are parameters of the comparator, not information available to the learner.
        \item \emph{Broadcasting.} The second block has zero attention logits, i.e.\ uniform causal attention, and a value map that copies the prompt information, multiplied by $\Tin$, to fresh coordinates.
        Since only position $1$ carries it, the readout position $\Tin$ receives exactly the stored values; other positions receive multiples $\Tin/t$, which are never read again.
        \item \emph{Reading the query cube.} Then $K_w$ blocks read the tokens at the absolute positions
        \begin{align*}
            ds_h+2,\ \ldots,\ d(h+S)+2,
        \end{align*}
        one block per position: the query of the readout position is $a_T\varphi(t)$ for the target position $t$, the keys are $\varphi(t')$ and the values are the token categories.
        By \cref{lem:position} the target receives weight at least $7/8$, and a plateau recovers its category exactly; the result is stored in a fresh coordinate and the temporary coordinates are reset, so that earlier reads do not contaminate later keys and values.
        Since the FFNs have no residual connection, each of them copies the coordinates that later blocks need, using $x=(x)_+-(-x)_+$; this at most doubles the width.
        \item \emph{Computing the arguments.} From the read tokens, an FFN of size $O(K_w^3)$ computes the query depth $\mathsf q$ from the position of $\STOP$; the support gates $\ind{D\subset C_{i-1}(m)}$ for $i\in\scI_h(m)$, as ReLU functions of the number of mismatches between the read bits and the stored path bits up to depth $L_{i-1}$; and the local coordinates
        \begin{align*}
            z_j=\sum_{r'=p+1}^{\mathsf q}b_{r',j}2^{-(r'-p)},\qquad j\in[d],
        \end{align*}
        of the query cube in the coordinates of each candidate support cube of depth $p$, namely $p=s_h$ for the old sum and $p=L_{i-1}$ for the level $i$.
        \item \emph{Evaluating and selecting.} A shared FFN of width $O(MK_w^2(N+1)+K_w^3)$ evaluates in parallel the $O(MK_w^2)$ approximating networks of all prompts, components and depths, and the exact gate \eqref{eq:gate} with the product of the indicators $\ind{m'=m}$, $\ind{\mathsf q'=\mathsf q}$ and the support gate selects the active ones and adds them, so that the approximation error is not multiplied by $M$ or $K_w$; the output is clipped to $[-1,1]$.
    \end{enumerate}
    All information flows to the readout position, in accordance with the causal mask.
    On inputs outside the window, that is, a prefix that disagrees with the correct path before depth $s_h$ or a depth below $s_h$, the comparator outputs some value in $[-1,1]$; no detection of such inputs is claimed, and their error, at most $2$ in absolute value, is charged to $\Lambda$ in \eqref{eq:window-bias-L2}.
    The construction fits \eqref{eq:window-capacity}: the number of blocks is $O(K_w)$, the state dimension $O(K_w)$, the FFN depth $O(\ell_N)$, and the weights are bounded by $O(M+K_w+\Tin^2\log(8\Tin)+N+1)$.

    \emph{Covering.}
    Counting the dense token embedding, the attention matrices, the readout and the nonzero FFN weights, the number of parameters is at most $D_{\rm emb}(M+5)+3\Nblk D_{\rm emb}^2+\Nblk S_{\rm FFN}\le CMK_w^4(N+1)\ell_N$, and the sparsity patterns are part of the class; \cref{lem:covering} with these capacities gives \eqref{eq:window-covering} (the input length enters only through the logarithm of the weight bound).
    \eqref{eq:window-bias-L2} follows from \eqref{eq:window-bias} and $\abs{F_{0,h}-Q}\le2$.
\end{proof}
The learner is not restricted to the comparator: all internal weights of the class $\Tcls_{\rm win}$ are fitted by ERM, and the comparator is used only to bound the approximation error in \cref{lem:oracle}.

\section{Proof of the Large-Budget Bound of \cref{thm:upper}}
\label{app:upper}

\paragraph{Proof strategy.}
The algorithm of the large-budget bound of \cref{thm:upper} consists of two stages, and the proof follows that division.
The \emph{front phase} is the sequence of exploration iterations $h<H$: it refines the score along the correct path and is analyzed in \cref{app:front}, whose main result is \cref{prop:front}, a high-probability guarantee that every update stays birth-accurate together with a bound on the total number of queries that the phase consumes.
The \emph{final refinement step} is the single iteration $h=H$: conditionally on the success event of the front phase, it keeps the score $W_H$ fixed, samples from the Gibbs policy at the target temperature $\rho$, fits one more critic at a variable query depth and outputs the clipped score.
It is analyzed in \cref{app:final}, whose main result is \cref{prop:final}, a conditional regret bound that still exhibits the bias--variance tradeoff in the refinement resolution.
\Cref{app:proof-upper} then assembles the two propositions: it balances the resolution, uses the remaining budget $b\ge n/2$, substitutes the bound on $H$ and adds the contribution of the front-phase failure event.
Schematically,
\begin{align*}
\text{\cref{prop:front} (front phase)}
\;\Longrightarrow\;
\text{\cref{prop:final} (final refinement)}
\;\Longrightarrow\;
\text{\cref{thm:upper}}.
\end{align*}

\subsection{Hyperparameters of the Large-Budget Setting}
\label{app:parameters}

This section concerns the large-budget setting of \cref{subsec:algorithm}: $\ell_{\rm look}=S:=R+1$, $\Hall=H+1$, and the last iteration $h=H$ is the refinement step \eqref{eq:final-step} with $I_H\equiv1$ and $\vartheta_H=\vartheta_{H+1}=\rho$; the small-regularization setting is treated in \cref{app:small}.
We call the exploration iterations $h<H$ the \emph{front phase} and the final iteration the \emph{final refinement step}, and use the following names: the batch size, class, actor and critic of the final refinement step are $b:=b_H$, $\Tcls_{\rm ref}:=\Tcls_H$, $P_{\rm start}:=P_H=\Gibbs{W_H}{\rho}$ and $\That_{\rm ref}:=\That_H$; the query cube of the final refinement step is $D_{\rm ref}(m,y):=\Query_H(m,y)=D_{q(m,y)+S+r}(y)$; and the clipped score is $v:=W_{H+1}$, so that $\Phat=\Gibbs{v}{\rho}$.
In the front phase, the query map is written $\Psi_h:=\Query_h$ of \eqref{eq:query-rule}, and $\Delta_h=2$, so that line~\ref{ln:score} of \cref{alg:actor-critic} reads $W_{h+1}(m,y)=\That_h(m,D_{h+S}(y))$ on the selected cubes and $W_{h+1}=W_h$ elsewhere.

The quantities listed in \cref{tab:notation-large} are specific to this proof; the schedules that both upper bounds share are those of \cref{tab:notation-upper}, whose values are fixed below for the large-budget setting.

\begin{table}[h]
\caption{Quantities used only in the proof of the large-budget bound of \cref{thm:upper}.}
\label{tab:notation-large}
\begin{center}
\begin{tabular}{ll}
\hline
Symbol & Meaning \\
\hline
$H$ & number of front-phase iterations \\
$R$, $S=R+1$ & refinement depth of the lookahead; the lookahead itself \\
$\eta$, $r_{\rm mom}$ & birth-accuracy level; moment order used in the regret transfer \\
$\delta_F$, $\delta_h$ & total and per-phase failure budgets \\
$\Tin$ & input length of the front-phase critic \\
$b$, $r$, $N_{\rm ref}$ & batch, query depth and approximation size of the final refinement step \\
$\Lf$ & logarithmic factor $\log(e+n+M+\rho^{-1})$ \\
\hline
\end{tabular}
\end{center}
\end{table}

Given $(n,M,\rho)$, the algorithm fixes the following quantities before the first query.
The structural constants $C_H,C_H',C_{\rm mom},c_\eta,t_0,C_J,C_N,C_b,D_N,D_b$ and $C_{\rm cal}$ (the constant of \eqref{eq:calibration}) are chosen in the proofs below, in the order indicated there; in particular $c_\eta\le\eta_0$ with $\eta_0$ the constant of \cref{lem:mass}.
\begin{align}
    H&=\left\lceil C_HK+C_H'\Lf\right\rceil,
    &r_{\rm mom}&=\lceil C_{\rm mom}\Lf\rceil,
    \notag\\
    \eta&=\frac{c_\eta}{r_{\rm mom}},
    &\delta_F&=\frac{\rho^2}{(n+1)^2},\qquad \delta_h=\frac{\delta_F}H,
    \notag\\
    R&=\left\lceil\frac1{\beta}\log_2\frac{2C_{\rm cal}}\eta\right\rceil,\quad S=R+1,
    &\varepsilon_h&=\frac\eta4\,a_{h+S}\quad(0\le h<H),
    \notag\\
    \vartheta_h&=t_0\,c_{1+\lfloor h/\lstar\rfloor}\quad(0\le h\le H-1),
    &\kappa_0&=\frac{p_\ast}2,
    \label{eq:params-1}
\end{align}
where $C_H,C_H'$ are large enough that $H\ge\lbar\lceil K\rceil+1$ and $S\le H$ (recall $S=O(\log\Lf)$), and $p_\ast$ is the constant of \cref{lem:mass}.
The accuracy $\varepsilon_h$ required of the critic of step $h$ is tied to the depth $h+S$ of the cubes that this step queries and writes into the score (\emph{query-depth accuracy}); no common absolute accuracy is required of all steps.
For the front phase, the window length, the window size, the input length, the approximation size and the batch sizes are
\begin{align}
    J&=\lceil C_J\Lf\rceil,
    &K_w&=C_d(J+S+2),
    \label{eq:params-2}\\
    \Tin&=d(H+S)+3,
    &N&=\left\lceil C_NM^{\frac d{2\alpha}}\Lf^{D_N}\right\rceil,
    \notag\\
    b_h&=\left\lceil C_b(1+\sigma^2)\,M^{\zeta_M}\,a_{h+S}^{-2}\,\Lf^{D_b}\right\rceil\quad(0\le h<H),
    \notag
\end{align}
and $\Tcls_h$ is the class $\Tcls_{\rm win}(K_w,N,\Tin)$ of \cref{lem:window} (whose input-length argument is written $\Tin$ there), the same for all $h<H$; the window parameter $J$ enters only the analysis (through the comparator of \cref{lem:window}), not the algorithm.
For the final refinement step, the batch size is $b=n-\sum_{h<H}b_h$, and the integer approximation size $N_{\rm ref}\ge1$ is chosen with
\begin{align}
    N_{\rm ref}\asymp\max\left\{1,\left(\frac{b\rho^2}{(1+\sigma^2)M(H+1)\Lf^{D_{\rm E}}}\right)^{\frac d{2\alpha+d}}\right\},
    \label{eq:params-3}
\end{align}
the query depth increment is $r=\lceil\alpha\log_2 N_{\rm ref}/(\beta d)\rceil$, the input length is $T_{\rm ref}=d(H+S+r)+3$, and $\Tcls_{\rm ref}$ is the class of \cref{lem:second-half} with these $H,S,N_{\rm ref}$.
The constants $C_{\rm F},D_{\rm F}$ of the large-budget bound of \cref{thm:upper} are such that $\sum_{h<H}b_h\le n/2$ whenever $n\ge C_{\rm F}(1+\sigma^2)M^{\zeta_M}\rho^{-(2+1/\gamma)}\Lf^{D_{\rm F}}$ (\cref{prop:front}), so that $b\ge n/2$.
The second condition of \cref{thm:upper} is this requirement written in terms of $\rho$, with $c_{\log}=\gamma D_{\rm F}/(2\gamma+1)$ and implicit constant $[C_{\rm F}(1+\sigma^2)]^{\gamma/(2\gamma+1)}$.
If the budget is below this threshold, the algorithm outputs the uniform policy; this case is outside the large-budget bound of \cref{thm:upper}.

\subsection{The Front Phase}
\label{app:front}

The goal of this subsection is the following guarantee for the exploration iterations $h<H$.
It is stated here, before the lemmas that prove it, because all of them are ingredients of this one statement.

\begin{proposition}[Front-phase guarantee and cost]\label{prop:front}
    With the parameters of \cref{app:parameters},
    \begin{align}
        \Pr(\scE)&\ge1-\delta_F,
        \notag\\
        \sum_{h<H}b_h&\le C(1+\sigma^2)M^{\zeta_M}H^{2\gamma+1}\Lf^{D_b}
        \le\frac{C_{\rm F}}2(1+\sigma^2)M^{\zeta_M}\rho^{-(2+1/\gamma)}\Lf^{D_{\rm F}},
        \label{eq:front}
    \end{align}
    where $\scE$ is the event \eqref{eq:event-E} that every critic of the front phase is accurate on its update inputs.
\end{proposition}

The proof is an induction over the phases, and the lemmas below are its steps.
It is not a single chain: after the bookkeeping of \cref{lem:mask} the argument branches into \cref{lem:progress} and \cref{lem:fine}, and the two branches rejoin at \cref{lem:joint}:
\begin{align*}
    \scE_h\ \longrightarrow\
    \begin{cases}
        \text{\cref{lem:progress}: true-path progress}\\
        \quad\to\ \text{\cref{lem:localize}: localization}\\
        \quad\to\ \text{\cref{lem:window,lem:oracle}: joint risk},\\[2mm]
        \text{\cref{lem:fine}: fine-query mass }p_S,
    \end{cases}
    \ \longrightarrow\
    \text{\cref{lem:joint}}\ \longrightarrow\ \scE_{h+1}.
\end{align*}
\Cref{prop:front} iterates this implication over the $H$ phases, sums the conditional failure probabilities and adds up the batch sizes.

The lemmas of this subsection are stated for a general lookahead $S\ge1$, a general accuracy sequence $\varepsilon_h=\frac\eta4a_{h+S}$ and general batch sizes $b_h$; they are used with the parameters of \cref{app:parameters} here and with $S=1$ in \cref{app:small}.
We analyze the scores as functions on the \emph{analysis partition} generated by the updates: initially $\scD_0(m)=\{\Omega\}$, and at step $h$ the selected cubes are replaced by their depth-$(h+S)$ sub-cubes.
The algorithm never forms this partition.

\begin{lemma}[Analysis partition with lookahead]\label{lem:mask}
    For every history and every $h$, the following hold.
    \begin{enumerate}[label=(\roman*),leftmargin=*,itemsep=0mm]
        \item $W_h(m,\cdot)$ is constant on the cubes of a partition $\scD_h(m)$ of $\Omega$ into dyadic cubes.
        \item Every cube $D\in\scD_h(m)$ except the initial root was created at a step $r(D)<h$ and has depth $r(D)+S$.
        \item The set $\{y:P_h(D_h(y)\mid m)\ge\kappa_0\}$ of responses is a union of cubes of $\scD_h(m)$ and of depth-$h$ cubes, of which there are at most $1/\kappa_0$; these are the \emph{selected} cubes.
        \item $\scD_{h+1}(m)$ is obtained from $\scD_h(m)$ by splitting the selected cubes into their depth-$(h+S)$ sub-cubes, the \emph{fine queries} of step $h$; the other cubes are left untouched and keep their score.
    \end{enumerate}
    In particular the partitions are refined and never coarsened, so the birth step $r(D)$ and the birth depth $r(D)+S$ of every current cube are well defined.
\end{lemma}
\begin{proof}
    If a cube $D'\in\scD_h$ has depth at most $h$, the actor is uniform on $D'$, so all depth-$h$ cubes inside $D'$ have the same mass and are selected together; if $D'$ is deeper than $h$, it lies inside one depth-$h$ cube.
    Hence the mask $I_h$ is constant on the cubes of $\scD_h$.
    The update of line~\ref{ln:score} with $\Delta_h=2$ replaces the score on the selected cubes by the values $\That_h(m,D)$ on their depth-$(h+S)$ sub-cubes $D$, which have depth $h+S>h-1+S\ge\max_{D'\in\scD_h}\mathrm{depth}(D')$, and keeps it elsewhere.
    The count follows since the masses of disjoint cubes sum to at most one.
\end{proof}

Let $\scE_h$ be the event that all critics of the steps $h'<h$ are accurate on their update inputs:
\begin{align}
    \scE_h=\Bigl\{\abs{\That_{h'}(m,D)-Q(m,D)}\le\varepsilon_{h'}\ \text{ for all }h'<h,\ m\in[M],\ \text{and fine queries }D\text{ of step }h'\Bigr\},
    \label{eq:event-E}
\end{align}
and $\scE=\scE_H$; $\scE_h$ is determined by the history at the beginning of step $h$.
On $\scE_h$, every cube $D\in\scD_h(m)$ created at step $r(D)$ satisfies $\abs{W_h(m,D)-Q(m,D)}\le\varepsilon_{r(D)}=\frac\eta4a_{\mathrm{depth}(D)}$, i.e., $W_h(m,\cdot)$ is $\eta$-birth-accurate on $\scD_h(m)$ in the sense of \cref{def:birth} (the root is updated at step $0$, so no cube keeps the initial value $0$ for $h\ge1$).

If the scores written so far are accurate, the correct cube carries enough actor mass to be refined, and the mass of the responses that left the path much earlier is exponentially small.
\begin{lemma}[The correct cube is refined at every step]\label{lem:progress}
    Let $\eta\le\eta_0$.
    On $\scE_h$, for every $m$ and $h\ge1$, the correct cube $P^\ast_h(m)$ is a union of cubes of $\scD_h(m)$, all created at step $h-1$,
    \begin{align}
        P_h\qty(P^\ast_h(m)\relmiddle|m)\ge p_\ast,\qquad
        P_h\qty(\qty(P^\ast_{h-j}(m))^c\relmiddle|m)\le Ce^{-cj}\quad(0\le j\le h),
        \label{eq:progress}
    \end{align}
    and, with $\kappa_0=p_\ast/2$, the correct cube $P^\ast_h(m)$ and all its ancestors are selected at step $h$.
    For $h=0$ the root is selected.
\end{lemma}
\begin{proof}
    By induction: if the correct parent $P^\ast_{h-1}(m)$ was selected at step $h-1$, it was split into depth-$(h-1+S)$ cubes, so $P^\ast_h(m)$ (of depth $h\le h-1+S$) is a union of these cubes.
    On $\scE_h$ the score is $\eta$-birth-accurate on $\scD_h(m)$, and the temperature satisfies \eqref{eq:temp-cond} by \eqref{eq:temp-range}, so \cref{lem:mass} gives \eqref{eq:progress}.
    Every ancestor of $P^\ast_h(m)$ has mass at least $p_\ast\ge\kappa_0$, hence passes the mass check.
\end{proof}

The previous lemma says which cube is refined; the next says how much probability each cube created by that refinement receives.
\begin{lemma}[Every fine query has a positive probability]\label{lem:fine}
    On $\scE_h$, for every $m$, every selected depth-$h$ cube $A'$ and every depth-$(h+S)$ cube $D\subset A'$,
    \begin{align}
        P_h(D\mid m)\ge p_S:=\kappa_0e^{-C_SS}2^{-dS},
        \label{eq:p_S}
    \end{align}
    and every response that enters $D$ is queried at $D$.
    In particular $p_S^{-1}\le C\Lf^{D_p}$ for a fixed $D_p$ when $S=O(1+\log\Lf)$, and $p_S=\kappa_02^{-d}$ when $S=1$.
\end{lemma}
\begin{proof}
    If $A'$ is contained in one cube of $\scD_h(m)$, the actor is uniform on $A'$ and $P_h(D\mid m)=2^{-dS}P_h(A'\mid m)\ge\kappa_02^{-dS}$.
    Otherwise the cubes of $\scD_h(m)$ inside $A'$ have depths between $h$ and $h+S-1$, their cube averages differ by at most $CSc_{k(h)}$ by \cref{lem:oscillation} (with $R+1=S$), and their scores are within $\frac\eta4a_h\le C\eta c_{k(h)}$ of the averages by birth accuracy; as $\vartheta_h\asymp c_{k(h)}$, the density of $P_h$ on $A'$ is within a factor $e^{C_SS}$ of uniform, which gives \eqref{eq:p_S}.
    A response entering $D$ passes the mass checks of all ancestors of $D$ up to depth $h$ (their masses are at least that of $A'$), so it is not stopped and is queried at depth $h+S$, i.e., at $D$.
    The bound holds for every selected cube, not only for the correct one.
\end{proof}

Up to an exponentially small leakage, the queried inputs of step $h$ fall in the window $\scW_h$, which is the finite region on which \cref{lem:window-target,lem:window} operate.
\begin{lemma}[Window localization]\label{lem:localize}
    Let $\Pi_h$ be the joint law of the queried input $(X_t,\Psi_h(X_t,Y_t))$ of step $h$ (a fresh prompt and one on-policy response).
    On $\scE_h$,
    \begin{align}
        \Lambda_h:=\Pi_h(\scW_h^c)\le Ce^{-cJ},
        \label{eq:localize}
    \end{align}
    and $\Lambda_h=0$ if $s_h=0$.
\end{lemma}
\begin{proof}
    Every query has depth at most $h+S$.
    If the response lies in $P^\ast_{s_h}(m)$, all its prefixes up to depth $s_h$ are ancestors of $P^\ast_h(m)$, which pass the mass check by \cref{lem:progress}, so the response is not stopped before depth $s_h$ and its queried cube, of depth at least $s_h$, is contained in $P^\ast_{s_h}(m)$, i.e., the input lies in $\scW_h$.
    Hence $\scW_h^c\subset\{Y\notin P^\ast_{s_h}(X)\}$, whose probability is at most $Ce^{-c(h-s_h)}\le Ce^{-cJ}$ by \eqref{eq:progress}.
\end{proof}

The oracle inequality bounds the risk averaged over the query law, whereas the induction needs accuracy at \emph{every} cube written into the score; each of them has probability at least $p_S$, so a joint risk of order $p_S\varepsilon_h^2/M$ suffices, and this is where the factor $M^{\zeta_M}$ is paid.
\begin{lemma}[Joint risk of the full-class ERM and accuracy on the update inputs]\label{lem:joint}
    Let $\Tcls_h=\Tcls_{\rm win}(K_w,N,\Tin)$ with $K_w\ge C_d(J+S+2)$ and $\Tin\ge d(h+S)+3$, and let $\That_h$ be the ERM of line~\ref{ln:erm} on the $b_h$ records of step $h$.
    Conditionally on a history in $\scE_h$, with probability at least $1-\delta_h$,
    \begin{align}
        \mathsf R_h:=\norm{\That_h-Q}^2_{L^2(\Pi_h)}
        \le C\qty[K_w^{2\gamma+2}a_{h+S}^2N^{-2\alpha/d}+\Lambda_h+(1+\sigma^2)\frac{V_N(b_h^{-2})+\log(2b_h/\delta_h)}{b_h}],
        \label{eq:joint-risk}
    \end{align}
    and, deterministically, for every $m$ and every fine query $D$ of step $h$,
    \begin{align}
        \abs{\That_h(m,D)-Q(m,D)}^2\le\frac M{p_S}\mathsf R_h.
        \label{eq:joint-to-node}
    \end{align}
\end{lemma}
\begin{proof}
    Conditionally on the history at the beginning of step $h$, the actor, the query rule and the class are fixed, the inputs $(X_t,\Psi_h(X_t,Y_t))$ are i.i.d.\ with law $\Pi_h$ (a fresh exogenous prompt and an independent rollout), the targets are $Z_t=Q(X_t,\Psi_h(X_t,Y_t))+\xi_t$ with conditionally $\sigma$-sub-Gaussian noise, and all records, including the stopped ones, enter the unweighted loss.
    \Cref{lem:oracle} with $\scF=\Tcls_h$ and the comparator of \cref{lem:window} gives \eqref{eq:joint-risk}, using \eqref{eq:window-bias-L2} with $\Lambda=\Lambda_h$ and \eqref{eq:window-covering}.
    For \eqref{eq:joint-to-node}: the joint probability of the input $(m,D)$ is $\Pi_h(\{(m,D)\})=P_h(D\mid m)/M\ge p_S/M$ by \cref{lem:fine}, and $\mathsf R_h\ge\Pi_h(\{(m,D)\})\abs{\That_h(m,D)-Q(m,D)}^2$.
\end{proof}
The factor $M$ in \eqref{eq:joint-to-node} is paid in full: the joint $L^2$ risk is converted into the accuracy on every update input of every prompt, and this is what makes the batch sizes grow like $M^{\zeta_M}$ rather than $M$; no prompt-wise estimation guarantee is used.

\begin{proof}[Proof of \cref{prop:front}]
    The accuracy statement is proved by induction over the phases: conditionally on a history in $\scE_h$, the parameters are chosen so that the two error terms of the joint risk bound \eqref{eq:joint-risk} are below the level that \eqref{eq:joint-to-node} needs in order to give the required accuracy at every node updated at phase $h$, which is exactly $\scE_{h+1}$; summing the conditional failure probabilities over the phases gives the first half of \eqref{eq:front}.
    The cost statement then only adds up the batch sizes and uses the bound on $H$.

    \emph{Accuracy.}
    Condition on a history in $\scE_h$.
    By \eqref{eq:joint-to-node} it suffices that the right-hand side of the joint risk bound \eqref{eq:joint-risk} be at most
    \begin{align}
        \mathsf r_h:=c\,\frac{p_S\eta^2a_{h+S}^2}M,
        \label{eq:target-risk}
    \end{align}
    and we check that each of its three terms is at most $\mathsf r_h$ separately.
    \begin{itemize}[leftmargin=*,itemsep=0mm]
        \item \emph{Leakage.} By \cref{lem:localize} and $H\le n$, choosing $C_J$ large gives
        \begin{align*}
            \Lambda_h\le Ce^{-cJ}\le c\,\frac{p_S\eta^2a_{H+S}^2}M\le\mathsf r_h
            \qquad(h<H),
        \end{align*}
        where we used $\eta^{-1}=O(\Lf)$, $p_S^{-1}=O(\Lf^{D_p})$, $a_{H+S}^{-2}=O(H^{2\gamma})$ and $M\le n$ on the range of the theorem.
        \item \emph{Approximation.} With $K_w=O(\Lf)$ and $N=\lceil C_NM^{d/(2\alpha)}\Lf^{D_N}\rceil$ for $D_N$ large, the approximation term $CK_w^{2\gamma+2}a_{h+S}^2N^{-2\alpha/d}$ of \eqref{eq:joint-risk} is at most $\mathsf r_h$.
        \item \emph{Estimation.} Since $b_h\le n$, \eqref{eq:window-covering} and the choice of $\delta_h$ give
        \begin{align*}
            V_N(b_h^{-2})\le CM^{1+\frac d{2\alpha}}\Lf^{D_N+9},
            \qquad
            \log\frac{2b_h}{\delta_h}\le C\Lf,
        \end{align*}
        so that with $D_b\ge D_N+D_p+11$ and $C_b$ large the estimation term of \eqref{eq:joint-risk} is at most $\mathsf r_h$ as well.
    \end{itemize}
    Hence, by \eqref{eq:joint-to-node}, with conditional probability at least $1-\delta_h$ every fine query of step $h$ is $\varepsilon_h=\frac\eta4a_{h+S}$-accurate, i.e.\ $\scE_{h+1}$ holds.
    Summing the conditional failure probabilities up to the first failure over the $H$ steps,
    \begin{align*}
        \Pr(\scE^c)\le\sum_{h<H}\delta_h=\delta_F.
    \end{align*}
    No independence between the noise and the future occurrence of inputs is assumed, and no conditioning on future successes is used.

    \emph{Cost.}
    Since $S\le H$ and $a_j^{-2}\asymp(j+1)^{2\gamma}$,
    \begin{align*}
        \sum_{h<H}b_h
        \le C(1+\sigma^2)M^{\zeta_M}\Lf^{D_b}\sum_{h<H}a_{h+S}^{-2}
        \le C(1+\sigma^2)M^{\zeta_M}H^{2\gamma+1}\Lf^{D_b},
    \end{align*}
    which is the first inequality of \eqref{eq:front}, the rounding in \eqref{eq:params-2} being absorbed into $C$; the second follows from $H\le CK\Lf$ with $C_{\rm F},D_{\rm F}$ large.
    All records are counted: stopped responses, queries outside the window and the batches after a failure.
\end{proof}

\subsection{The Final Refinement Step}
\label{app:final}

Condition on the successful front-phase event $\scE$ of \cref{prop:front}.
No query, critic fit or score update takes place between the two stages: the front-phase score $W_H$ is kept fixed and the Gibbs temperature is set to the target regularization level, so that the responses of the last iteration are drawn from $P_{\rm start}=\Gibbs{W_H}{\rho}$.
The goal of this subsection is the following conditional guarantee, expressed in terms of the resolution $N_{\rm ref}$ chosen in \eqref{eq:params-3}, to display the approximation and estimation errors separately.

\begin{proposition}[Final-refinement guarantee]\label{prop:final}
    With the parameters of \cref{app:parameters} on the large-budget range of \cref{thm:upper}, including the refinement resolution $N_{\rm ref}$ chosen in \eqref{eq:params-3}, the output $\Phat=\Gibbs{v}{\rho}$ of \cref{alg:actor-critic} satisfies
    \begin{align}
        \bE\qty[\Reg_\rho(\Phat)\relmiddle|\scE]
        \le C\rho N_{\rm ref}^{-\frac{2\alpha}d}
        +\frac{C(1+\sigma^2)M(H+1)(N_{\rm ref}+1)}{b\rho}\Lf^{D_{\rm E}}
        +\frac{C\delta_C}\rho,
        \label{eq:final-prop}
    \end{align}
    where $b$ is the batch of the final iteration and $\delta_C$ is the mass of the deep core \eqref{eq:core-mass}.
\end{proposition}

The three terms are the approximation error of a resolution-$N_{\rm ref}$ fit, the estimation error of one critic fitted on $b$ records, and the mass of the deep core.
The proof is the chain
\begin{align*}
\text{\cref{lem:calibration}: pointwise calibration}
&\to
\text{\cref{lem:exp-moment}: Gibbs mass and moments}\\
&\to
\text{\cref{lem:second-half}: final critic estimation}\\
&\to
\text{\cref{lem:projection}: continuous reward and clipping}\\
&\to
\text{\cref{lem:transfer}: regret transfer}.
\end{align*}
\paragraph{Calibration and Gibbs control.}
For the rest of the proof we fix a good front-phase history (the event $\scE$) and a prompt $m$; all conditional expectations $\bE[\,\cdot\mid\scE]$ below are taken given that history.
We write
\begin{align*}
    P_j=P^\ast_j(m)\ \ (0\le j\le H),\qquad
    A_j=P_j\setminus P_{j+1}\ \ (0\le j<H),\\
    C_{\rm core}=P_H,\qquad
    P_{\rm start}=\Gibbs{W_H}{\rho}.
\end{align*}
Only for the proof, consider the analysis partition $\scD_H$ of the final score $W_H(m,\cdot)$.
By \cref{lem:mask}, each update refines the previous partition and never coarsens it, so the last update index $q$ is constant on each final cube $A'\in\scD_H$, and, on $\scE$,
\begin{align}
    \mathrm{depth}(A')=q(A')+S,\quad
    W_H(A')=\That_{q(A')}(m,A'),\quad
    \abs{W_H(A')-Q(m,A')}\le\tfrac\eta4a_{q(A')+S}.
    \label{eq:holding}
\end{align}
Since the root is updated at $h=0$, no cube keeps the initial value $0$, and $W_H\le G_m+\frac\eta4a_{q+S}$ everywhere.
The algorithm never enumerates $\scD_H$; the index $q(m,y)$ of \eqref{eq:final-step} equals $q(A')$ for the cube $A'\ni y$, and is evaluated for each sample from the saved actors.

A response that leaves the correct path at depth $j$ lies in a cube last updated at some step $q\ge j$, so the error of the held score there is measured at the scale $a_q$ of that update rather than at a common scale.
\begin{lemma}[Calibration of the held score]\label{lem:calibration}
    On $\scE$, for every final cube $A'\in\scD_H$ contained in the shell $A_j$ ($j<H$),
    \begin{align}
        q(A')\ge j,
        \label{eq:q-ge-j}
    \end{align}
    and for every $y\notin C_{\rm core}$ with last update index $q=q(m,y)$,
    \begin{align}
        \abs{f(y;m)-W_H(m,y)}\le\frac\eta4a_{q+S}+C_{\rm cal}a_q2^{-\beta S}\le\eta a_q.
        \label{eq:calibration}
    \end{align}
\end{lemma}
\begin{proof}
    By \cref{lem:progress} the correct parent $P_j$ is selected at step $j$, so the whole shell $A_j\subset P_j$ satisfies the mass condition at step $j$ and is updated at depth $j+S$; the holding of later steps never moves a query depth back, which gives \eqref{eq:q-ge-j}, and a final cube of depth $q+S\ge j+1$ is contained in one shell by dyadic consistency.
    Let $i(j)=\min\{i:L_i>j\}$, so that $f=f^{[i(j)]}$ on $A_j$.
    By \eqref{eq:holder-shell} (the cube $A'$ has side $2^{-(q+S-j)}$ in the local coordinates of $P_j$),
    \begin{align*}
        \sup_{y,z\in A'}\abs{f(y;m)-f(z;m)}\le Cc_{i(j)}2^{-\beta(q+S-j)}.
    \end{align*}
    The known scales satisfy
    \begin{align*}
        \frac{c_{i(j)}}{a_q}\le C(1+q-j)^\gamma,\qquad
        a_q\le a_j\le C_ac_{i(j)},\qquad
        \sup_{u\ge0}(1+u)^\gamma2^{-\beta u}<\infty,
    \end{align*}
    so the oscillation of $f$ on $A'$ is at most $C_{\rm cal}a_q2^{-\beta S}$, which is at most $\eta a_q/2$ by the choice of $R$ in \eqref{eq:params-1} ($2^{-\beta S}\le2^{-\beta R}\le\eta/(2C_{\rm cal})$).
    Combining with \eqref{eq:holding} and $\frac\eta4a_{q+S}\le\frac\eta4a_q$ gives \eqref{eq:calibration}.
    The accuracy of the held score is the one of its own birth, at the scale of its query depth; no common final accuracy of all steps is used.
    This is a guarantee for the held score itself; no re-selection of the critic of step $j$ is needed.
\end{proof}

\paragraph{Deep core mass and exponential moments.}

The calibration scale $a_q$ varies over the response space, and the next lemma controls its distribution under $P_{\rm start}$; the exponential moment is what bounds the density ratio in the regret transfer.
\begin{lemma}[Normalization, deep-core mass and moments of the scale]\label{lem:exp-moment}
    Let $K=\rho^{-1/\gamma}$, $k_0=\lceil K\rceil$, and let $C_H$ be large enough that $H\ge\lbar k_0+1$.
    On $\scE$, there are constants $U,C_{\exp}$, independent of $n,M,\rho,r_{\rm mom}$ and $S$, such that
    \begin{align}
        Z_H:=\int e^{W_H/\rho}\dd{\mu}\ge\exp\qty(\frac{G_m}\rho-UK-C\eta),
        \label{eq:ZB}
    \end{align}
    \begin{align}
        P_{\rm start}(C_{\rm core}\mid m)+P_{f,\rho}(C_{\rm core}\mid m)
        \le C\exp\qty(-dH\log2+UK+C\eta)
        \le\delta_C:=\frac{\rho^2}{(n+1)^2},
        \label{eq:core-mass}
    \end{align}
    and, with $q=q(m,y)$ and for every fixed $\varpi\ge0$,
    \begin{align}
        &\int\exp\qty(\frac{r_{\rm mom}\eta a_q}{\rho})\dd{P_{\rm start}}\le C_{\exp},
        \qquad
        \int\qty(\frac{a_q}\rho)^\varpi\dd{P_{\rm start}}\le C_\varpi,
        \label{eq:exp-moment}\\
        &\sum_{j<H}P_{\rm start}(A_j\mid m)a_j^2\le C\rho^2.\notag
    \end{align}
\end{lemma}
\begin{proof}
    The three statements are proved in the order in which they depend on each other.
    We first lower-bound the partition function $Z_H$ by restricting the integral to a near-optimal region on which the held score is accurate; this denominator bound, together with the tiny volume of the deepest region, shows that the deep core has negligible mass under both $P_{\rm start}$ and the true Gibbs policy.
    For the moments we split the remaining shells into two regimes: shallow shells have a large reward gap and hence exponentially small Gibbs weight, while deep shells may carry mass but there $a_q/\rho$ is uniformly bounded.
    Combining the two regimes gives the exponential moment, the polynomial moments and the weighted quadratic bound.

    \emph{Normalization.}
    Let $\ell=L_{k_0}\le\lbar k_0\le H-1$.
    Step $\ell$ is executed, and at this step the correct parent $P^\ast_\ell(m)=C_{k_0}(m)$ is selected (\cref{lem:progress}) and updated as a whole at depth $\ell+S$; since later updates only refine, $C_{k_0}(m)$ is a union of final cubes with last update index $q\ge\ell$, whose birth errors are at most $\frac\eta4a_{q+S}\le\frac\eta4a_\ell\le C\eta c_{k_0}\le C\eta\rho$ by \eqref{eq:holding}.
    By \eqref{eq:geo}, $f\ge G_m-C_{\rm near}T_{k_0}$ on $C_{k_0}(m)$, and so are the cube averages of the final cubes inside; hence
    \begin{align*}
        W_H\ge G_m-C_{\rm near}T_{k_0}-C\eta\rho\ \ \text{on }C_{k_0}(m),
    \end{align*}
    and \eqref{eq:ZB} follows from $\mu(C_{k_0})\ge2^{-d\lbar k_0}$, $k_0\le2K$ and $T_{k_0}/\rho\le C_Tk_0^{-(\gamma-1)}K^\gamma\le2C_TK$, with $U=4d\lbar\log2+2C_{\rm near}C_T$.
    No constancy of the score on the core is assumed, and the core $C_{\rm core}=P^\ast_H(m)$ is also a union of final cubes, all created at step $H-1$.

    \emph{Deep-core mass.}
    The core has volume $2^{-dH}$, and on it $q=H-1$ and $W_H\le Q+\frac\eta4a_{H-1+S}\le G_m+C\eta\rho$ (as $a_{H-1+S}\le a_H\le C\rho$); together with \eqref{eq:ZB} this gives the first inequality of \eqref{eq:core-mass} for $P_{\rm start}$, and the same bound holds for $P_{f,\rho}$ by the lower bound on its normalization from \eqref{eq:geo}.
    The second inequality holds once the constants $C_H,C_H'$ in \eqref{eq:params-1} are large.

    \emph{Moments.}
    Choose $c_\eta$ with $2c_\eta C_aC_{cT}\le\kappa/2$, where $c_i\le C_{cT}T_i$.
    For a final cube $A'\subset A_j\subset E_i$ (exit level $i$), by \eqref{eq:geo-lower}, \eqref{eq:calibration} and $a_q\le a_j\le C_ac_i$,
    \begin{align*}
        W_H+r_{\rm mom}\eta a_q\le G_m-\kappa T_i+(r_{\rm mom}+1)\eta C_ac_i\le G_m-\frac\kappa2T_i,
    \end{align*}
    and the same holds with $a_q$ replaced by $a_j$.
    Fix $\theta_0>0$ with $\kappa c_T\theta_0^{1-\gamma}/2\ge U+2$.
    For $i\le\theta_0K$, \eqref{eq:ZB} and \eqref{eq:tail} give a contribution at most $C\mu(A')e^{-2K}$ (the factor $e^{C\eta}\le C$ is absorbed); for $i>\theta_0K$, $a_q/\rho\le a_j/\rho\le C\theta_0^{-\gamma}$ and $r_{\rm mom}\eta=c_\eta$, so the exponential factor is bounded by a constant.
    Inside the core, $q=H-1$ and $a_q\le CH^{-\gamma}\le C\rho$, so the core can be included in the first two bounds.
    Summing gives \eqref{eq:exp-moment}; the decomposition into shells is a device of the proof and does not correspond to any special treatment of the core by the algorithm.
\end{proof}

\paragraph{The final critic fit.}
\label{app:second-half}

Let $\Pi$ be the joint law of $(X,D_{\rm ref}(X,Y))$ under $P_{\rm start}$ and $T_{\rm ref}=d(H+S+r)+3$; in the rest of this subsection, $N$ denotes the approximation size $N_{\rm ref}$ of \eqref{eq:params-3}, with $r=\lceil\alpha\log_2 N/(\beta d)\rceil$.

The final critic approximates a different target, the reward on each shell in the local coordinates of that shell, by the same pattern as in \cref{app:transformer}.
\begin{lemma}[Approximation and capacity for the final refinement]\label{lem:second-half}
    There is a class $\Tcls_{\rm ref}=\Tcls(\Nblk,D_{\rm emb},L_{\rm FFN},W_{\rm FFN},S_{\rm FFN},B_{\rm par})$ with $\Tin=T_{\rm ref}$ and
    \begin{align}
        D_{\rm emb}&=C_{\alpha,d},\qquad
        \Nblk\le C_{\alpha,d}\qty(R+2+\log(N+1)),
        \notag\\
        L_{\rm FFN}&\le C_{\alpha,d}\qty(1+\log(N+1)),
        \label{eq:capacity-2}\\
        W_{\rm FFN},S_{\rm FFN}&\le C_{\alpha,d}M(H+1)(N+1)\qty(1+\log(N+1)),
        \notag\\
        B_{\rm par}&\le C_{\alpha,d}\qty[M+\Tin^2\log(8M\Tin)+1],
        \notag
    \end{align}
    whose covering entropy satisfies
    \begin{align}
        \log\scN_\infty(\varepsilon,\Tcls_{\rm ref})\le CM(H+1)(N+1)\qty(R+2+\log(N+1))^5\log\frac{CM\Tin(N+1)}{\varepsilon},
        \label{eq:covering-2}
    \end{align}
    and which, on $\scE$ (with the front-phase history fixed as above), contains a comparator $F_0$ (depending on the history) with
    \begin{align}
        \norm{F_0-Q}^2_{L^2(\Pi)}\le C\rho^2N^{-2\alpha/d}+4\delta_C.
        \label{eq:approximation}
    \end{align}
    Consequently, the ERM $\That_{\rm ref}$ over $\Tcls_{\rm ref}$ on $b$ fresh records satisfies, for a constant $D_{\rm E}$,
    \begin{align}
        \bE\qty[\norm{\That_{\rm ref}-Q}^2_{L^2(\Pi)}\relmiddle|\scE]
        \le C\rho^2N^{-2\alpha/d}+4\delta_C+\frac{C(1+\sigma^2)M(H+1)(N+1)}b\Lf^{D_{\rm E}}.
        \label{eq:second-half}
    \end{align}
\end{lemma}
\begin{proof}
    The proof repeats the approximation--realization--covering--ERM pattern of \cref{app:transformer} for the shell-wise target, in four steps.
    We first approximate the local H\"older function of each shell at resolution $N$ by a ReLU network, keeping the side length of the shell out of the approximation problem; then we realize these networks inside the Transformer class with the fixed positional encoding; then we read off the capacity and the covering entropy; and finally we collect the approximation error, the contribution of the deep core and the estimation error of \cref{lem:oracle}.

    \emph{Local approximation without treating the side length as a variable.}
    By \cref{lem:holder-shell}, $h_{m,j}(z)=(g_{m,j}(z)-g_{m,j}(0))/a_j$ belongs to a fixed $\alpha$-H\"older ball on $[0,1]^d$.
    By \citet[Theorem 5]{schmidt2020nonparametric} (see also \citet{yarotsky2017error}), there is a ReLU network $\tilde h_{m,j}$ of depth $C\log(N+1)$, width $CN$ and $CN\log(N+1)$ nonzero weights of size $O(1)$ with
    \begin{align*}
        \norm{\tilde h_{m,j}-h_{m,j}}_\infty\le CN^{-\alpha/d}
    \end{align*}
    (for all $\alpha>0$, including $C^{k,1}$ at integer $\alpha$); the comparator uses $\tilde g_{m,j}:=g_{m,j}(0)+a_j\tilde h_{m,j}$, so that $\norm{\tilde g_{m,j}-g_{m,j}}_\infty\le Ca_jN^{-\alpha/d}$, with $MH$ local networks in parallel and $CMHN\log(N+1)$ nonzero weights in total.
    A query cube $D$ of a response leaving the path at depth $j$ has depth $q+S+r=q+R+1+r$ with $q\ge j$ by \eqref{eq:q-ge-j}.
    Let $F=D_{j+R+1+r}(y)$ be its ancestor of depth $j+R+1+r$, and let $z_F$ be the local coordinates of the center of $F$ in $P_j$.
    Since $D\subset F\subset A_j$, \cref{lem:holder-shell} gives
    \begin{align}
        \max\left\{\abs{Q(m,D)-g_{m,j}(z_F)},\ \sup_{y\in F}\abs{f(y;m)-g_{m,j}(z_F)}\right\}\le Ca_j2^{-\beta(R+1+r)}\le Ca_jN^{-\alpha/d}.
        \label{eq:ancestor-approx}
    \end{align}
    Here $2^{-\beta r}\le N^{-\alpha/d}$ by the choice of $r$. Thus only the first $d(R+1+r)$ bits after the exit from the path have to be read, whatever the query depth; no approximation in $d+1$ variables and no separate approximator per $(m,j,q)$ is needed. Reading these $O_{\alpha,d}(R+1+\log(N+1))$ bits uses that many blocks, and does not require a table indexed by the $2^{dr}$ fine cubes; each cube query still costs one reward observation.

    \emph{Realization with the fixed positional encoding.}
    The comparator is built from the following blocks; $D_{\rm emb}$ is fixed throughout.
    \begin{enumerate}[label=(\arabic*),leftmargin=*,itemsep=0mm]
        \item \emph{Reading the input.} The first block recovers from $\varphi(t)$, by \cref{lem:position} and plateaus, the token category and the integers $t,t^2$ at every position, with sparsity $O(\Tin+M)$ and weights $O(\Tin^2+M)$; a second block delivers the prompt index $m$ to all positions by an attention whose logit is $\log(16M\Tin)$ on the prompt token, followed by a plateau on $[1,M]$.
        \item \emph{Locating the exit from the correct path.} The comparator stores the $dH$ bits of the correct path of each prompt in an FFN table of size $O(MH)$ --- a property of the comparator, not information given to the learner --- and flags as an \emph{event} every one of the first $dH$ response positions at which the input bit differs from the path bit, as well as the $\STOP$ position; later bits raise no event.
        With the event flag $e_t\in\{0,1\}$, the readout attends with the logit
        \begin{align*}
            b_0e_t-a_0t,\qquad a_0=\log(16\Tin^2),\quad b_0=a_0(\Tin+1),
        \end{align*}
        so that the first event position $s_{\rm first}$ receives weight at least $7/8$; its position and $\STOP$ flag are recovered exactly, and the exit depth
        \begin{align*}
            j=\left\lfloor\frac{s_{\rm first}-2}d\right\rfloor
        \end{align*}
        is computed by $O(\Tin)$ ReLUs.
        \item \emph{Reading the local coordinates.} If the first event is a mismatch, $d(R+1+r)$ further blocks read the bits at the positions
        \begin{align*}
            dj+2,\ \ldots,\ dj+d(R+1+r)+1
        \end{align*}
        by an attention with logit $2a'st-a't^2$, $a'=\log(8\Tin)$, and keys $(t,t^2)$; accumulating them into $d$ coordinates with weights $2^{-k'}$ and adding $2^{-(R+r+2)}$ yields $z_F\in[0,1]^d$.
        All these bits are present in the input, since $q\ge j$.
        \item \emph{Selecting the local network.} The ReLU gate
        \begin{align*}
            \qty(1-\abs{m-m_0}-\abs{j-j_0})_+
        \end{align*}
        selects the local network $\tilde g_{m_0,j_0}(z_F)$, so that the approximation error is not multiplied by $MH$.
        \item \emph{The core.} If the first event is $\STOP$, i.e.\ there is no mismatch within the first $dH$ bits, the comparator returns $0$; on the actual query support of a good history these inputs are exactly the queries inside $C_{\rm core}$.
    \end{enumerate}

    \emph{Capacity and covering.}
    The construction fits \eqref{eq:capacity-2}: the sparsity is dominated by $O(M\Tin)$ for the position/path tables and $CMHN\log(N+1)$ for the local networks, the number of blocks is $O(R+1+r)=O(R+2+\log(N+1))$, and the weights are polynomial in $M,\Tin$; no address of the final holding cubes has to be memorized, so no extra factor $H$ enters the capacity.
    \Cref{lem:covering} gives \eqref{eq:covering-2} (the exponent $5$ is not optimized).

    \emph{Errors.}
    On the queries outside the core, by \eqref{eq:ancestor-approx} and the third bound of \eqref{eq:exp-moment} applied to every prompt (recall that $\Pi(m,\cdot)$ carries the factor $1/M$),
    \begin{align*}
        \sum_{m=1}^M\sum_{j<H}\Pi(m,A_j)\,C a_j^2N^{-2\alpha/d}\le C\rho^2N^{-2\alpha/d};
    \end{align*}
    on the core, $Q\in[-1,1]$ and the comparator returns $0$, so the contribution is at most $4\Pi(\text{core})\le4\delta_C$ by \eqref{eq:core-mass}.
    This gives \eqref{eq:approximation}; the value $0$ on the core is a property of the comparator, not a removal of these records from the ERM.
    For the resolution \eqref{eq:params-3}, $\log(N+1)+\log T_{\rm ref}\le C\Lf$.
    \Cref{lem:oracle} with $n_b=b$ and $V(b^{-2})$ bounded by \eqref{eq:covering-2} gives \eqref{eq:second-half}.
\end{proof}

The previous lemma bounds the error at the queried cubes, while the output policy is built from the clipped score on the whole space; the transfer uses \cref{lem:oscillation} and the clipping widths of \eqref{eq:final-step}.
\begin{lemma}[Projection error and the clipped score]\label{lem:projection}
    On $\scE$, the clipped score $v=W_{H+1}$ of line~\ref{ln:score} of the final refinement step satisfies $v\in[-1,1]$, and, with $f^\dagger=f$ off the core and $f^\dagger=W_H$ on the core,
    \begin{align}
        \bE\qty[\int\abs{v-f^\dagger}^2\dd{P_{\rm start}}\relmiddle|\scE]
        \le C\rho^2N^{-\frac{2\alpha}d}+\frac{C(1+\sigma^2)M(H+1)(N+1)}b\Lf^{D_{\rm E}}+C\delta_C.
        \label{eq:continuous-risk}
    \end{align}
\end{lemma}
\begin{proof}
    The index $q$ is constant on each final cube, and the query cubes are the $r$-fold subdivisions of the final cubes.
    Since $P_{\rm start}$ has constant density on each final cube, it is uniform on each query cube $D$, so that
    \begin{align*}
        &\int_D\qty(\That_{\rm ref}(m,D)-f)^2\dd{P_{\rm start}}
        =P_{\rm start}(D)\qty(\That_{\rm ref}(m,D)-Q(m,D))^2\\
        &\qquad+\int_D\qty(Q(m,D)-f)^2\dd{P_{\rm start}}.
    \end{align*}
    Off the core, the second term summed over the query cubes and the prompts is at most $C\rho^2N^{-2\alpha/d}$ by \eqref{eq:ancestor-approx} and \eqref{eq:exp-moment}.
    Off the core, \eqref{eq:calibration} shows that $f$ and $Q(m,D)$ lie in the clipping interval $[W_H-\Delta,W_H+\Delta]$, so the clipping is a projection and does not increase the distance to $f^\dagger=f$.
    On the core, the same update is applied, $q=H-1$, and $\abs{v-W_H}\le\rho+\eta a_{H-1}\le C\rho$, so the squared error is at most $C\rho^2P_{\rm start}(\text{core})\le C\rho^2\delta_C$.
    Combining with \eqref{eq:second-half} gives \eqref{eq:continuous-risk}.
    The clipping interval contains $W_H$ and both $W_H$ and $\That_{\rm ref}$ lie in $[-1,1]$, so $v\in[-1,1]$.
\end{proof}

\paragraph{From squared error to KL-regularized regret.}
\label{app:clipping}
A small squared error does not by itself give a small regret, and it is worth seeing why before the transfer is carried out.
Let $\Phi_\rho(w)=\rho\log\int e^{w/\rho}\dd{\mu}$.
For a Gibbs policy $P_v=\Gibbs{v}{\rho}$, the regret against a reward $f$ is the Bregman divergence of $\Phi_\rho$,
\begin{align*}
    \Reg_{\rho,f}(P_v)=\Phi_\rho(f)-\Phi_\rho(v)-D\Phi_\rho(v)[f-v],
    \qquad
    D^2\Phi_\rho(w)[e,e]=\frac1\rho\Var_{P_w}(e),
\end{align*}
so that, by Taylor's formula with integral remainder,
\begin{align}
    \Reg_{\rho,f}(P_v)=\frac1\rho\int_0^1(1-s)\Var_{P_{w_s}}(f-v)\dd{s},\qquad w_s=(1-s)v+sf.
    \label{eq:bregman}
\end{align}
The variance is taken under the interpolated Gibbs policies, not under the collecting policy, and in general a small squared error under the collecting policy does not control \eqref{eq:bregman}.
Indeed, take $f=0$, $P_{\rm start}=\mu$, and $v_\rho=\ind{A_\rho}$ for a set $A_\rho$ of measure $e^{-1/(2\rho)}$; then
\begin{align*}
    \norm{v_\rho-f}^2_{L^2(P_{\rm start})}=e^{-1/(2\rho)}\ll\rho,
    \qquad
    P_{v_\rho}(A_\rho)\xrightarrow[\rho\to0]{}1,
    \qquad
    \Reg_{\rho,0}(P_{v_\rho})\xrightarrow[\rho\to0]{}\frac12.
\end{align*}
A rare overestimate on a region of tiny mass concentrates the policy there, and bounding the density ratio by the range of the scores alone would cost a factor $e^{2/\rho}$.
The calibration of \cref{lem:calibration} and the clipping \eqref{eq:final-step} avoid both.

The conversion uses the calibration and the exponential moment to bound the density ratio along the interpolation.
\begin{lemma}[Regret of the clipped update]\label{lem:transfer}
    Let $\bar u$ be a known upper bound on the right-hand side of \eqref{eq:continuous-risk} with $\bar u\le1$, and suppose $r_{\rm mom}\ge\max\{2,\log(4/\bar u)\}$. Then, on $\scE$,
    \begin{align}
        \bE\qty[\Reg_{\rho,f}(P_v)\relmiddle|\scE]
        \le C\frac{\bar u}\rho+C\delta_C.
        \label{eq:transfer}
    \end{align}
\end{lemma}
\begin{proof}
    The proof proceeds in three steps.
    By \eqref{eq:bregman} the regret is an integral of variances under the Gibbs policies interpolating between the clipped score and the reward, so the argument has to compare those policies with the collecting policy $P_{\rm start}$.
    We first bound the density ratio along the interpolation by $Ce^{\eta a_q/\rho}$, using the calibration of the held score; then H\"older's inequality with the exponential moment of \cref{lem:exp-moment} converts the squared error under $P_{\rm start}$ into a regret bound against the modified reward $f^\dagger$; and finally the core is restored, at the price of its mass $\delta_C$.

    \emph{Density ratios along the interpolation.}
    For any interpolation $w$ between $f^\dagger$ and $v$, we have $\abs{w-W_H}\le\rho+\eta a_q$ everywhere (on the core because $q=H-1$).
    By \eqref{eq:exp-moment} and Jensen's inequality,
    \begin{align*}
        \frac{Z_w}{Z_H}\ge\exp\qty(-1-\eta\int\frac{a_q}\rho\dd{P_{\rm start}})\ge c,
        \qquad
        \frac{\dd{P_w}}{\dd{P_{\rm start}}}\le C\exp\qty(\frac{\eta a_q}\rho),
    \end{align*}
    and the same bound holds on the core since $a_{H-1}\le C\rho$; no constancy of the score on the core is used.

    \emph{Regret against $f^\dagger$.}
    By \eqref{eq:bregman} and the density bound, with $E:=\abs{v-f^\dagger}\le2$,
    \begin{align*}
        \Reg_{\rho,f^\dagger}(P_v)\le\frac C\rho\int e^{\eta a_q/\rho}E^2\dd{P_{\rm start}}.
    \end{align*}
    H\"older's inequality with exponents $r_{\rm mom}$ and $r_{\rm mom}/(r_{\rm mom}-1)$ and \eqref{eq:exp-moment} give
    \begin{align*}
        \int e^{\eta a_q/\rho}E^2\dd{P_{\rm start}}
        &\le\qty(\int e^{r_{\rm mom}\eta a_q/\rho}\dd{P_{\rm start}})^{1/r_{\rm mom}}
        \qty(\int E^{2r_{\rm mom}/(r_{\rm mom}-1)}\dd{P_{\rm start}})^{1-1/r_{\rm mom}}\\
        &\le4\,C_{\exp}^{1/r_{\rm mom}}\qty(\int E^2\dd{P_{\rm start}})^{1-1/r_{\rm mom}},
    \end{align*}
    using $E^{2r_{\rm mom}/(r_{\rm mom}-1)}\le4^{1/(r_{\rm mom}-1)}E^2$.
    Taking expectations, using concavity and \eqref{eq:continuous-risk},
    \begin{align*}
        \bE\,\Reg_{\rho,f^\dagger}(P_v)\le\frac C\rho\bar u^{1-1/r_{\rm mom}}\le\frac{Ce}\rho\bar u,
    \end{align*}
    where the last step uses $\bar u^{-1/r_{\rm mom}}\le e$ by the choice of $r_{\rm mom}$; the constant of the exponential moment is uniform in $r_{\rm mom}$, so no polynomial factor is lost.

    \emph{Restoring the true reward on the core.}
    By the density bound and \eqref{eq:core-mass}, $P_v(C_{\rm core})+P_{f^\dagger,\rho}(C_{\rm core})\le C\delta_C$.
    Since $f$ and $f^\dagger$ agree off the core,
    \begin{align*}
        \Phi_\rho(f)-\Phi_\rho(f^\dagger)=\rho\log\frac{1-P_{f^\dagger,\rho}(C_{\rm core})}{1-P_{f,\rho}(C_{\rm core})},
    \end{align*}
    whose absolute value is at most $C\rho\delta_C$, while the expected rewards of $P_v$ under $f$ and $f^\dagger$ differ by at most $2P_v(C_{\rm core})\le C\delta_C$.
    This gives \eqref{eq:transfer}.
\end{proof}

\begin{proof}[Proof of \cref{prop:final}]
    On $\scE$, \cref{lem:calibration,lem:exp-moment,lem:second-half,lem:projection} hold for every prompt, and \eqref{eq:continuous-risk} bounds the squared error of the clipped score $v$ under $P_{\rm start}$ by
    \begin{align*}
        \bar u_{N_{\rm ref}}:=C\rho^2N_{\rm ref}^{-\frac{2\alpha}d}+\frac{C(1+\sigma^2)M(H+1)(N_{\rm ref}+1)}b\Lf^{D_{\rm E}}+C\delta_C.
    \end{align*}
    For the choice \eqref{eq:params-3}, put $\mathsf A=b\rho^2/[(1+\sigma^2)M(H+1)\Lf^{D_{\rm E}}]$.
    The large-budget condition, $b\ge n/2$ and $H+1\le C\rho^{-1/\gamma}\Lf$ give $\mathsf A\ge cC_{\rm F}M^{\zeta_M-1}\Lf^{D_{\rm F}-D_{\rm E}-1}$.
    Taking $C_{\rm F},D_{\rm F}$ sufficiently large makes $\mathsf A\ge1$, and $N_{\rm ref}\asymp\mathsf A^{d/(2\alpha+d)}$ yields
    $\bar u_{N_{\rm ref}}\le C\rho^2[\mathsf A^{-p}+(n+1)^{-2}]\le\rho^2\le1$.
    The known bound may be enlarged to at least $c\rho^2/(n+1)^2$, so that $\log(4/\bar u_{N_{\rm ref}})\le C\Lf\le r_{\rm mom}$ for $C_{\rm mom}$ large and the clipping widths use only $\rho$, $\eta$ and $a_q$.
    \Cref{lem:transfer} then gives $\bE[\Reg_\rho(\Phat)\mid\scE]\le C\bar u_{N_{\rm ref}}/\rho$, which is \eqref{eq:final-prop}.
\end{proof}

\subsection{Proof of the Large-Budget Bound}
\label{app:proof-upper}

By \cref{prop:front}, the front phase uses $\sum_{h<H}b_h\le n/2$ queries on the range $n\ge C_{\rm F}(1+\sigma^2)M^{\zeta_M}\rho^{-(2+1/\gamma)}\Lf^{D_{\rm F}}$ and the event $\scE$ has probability at least $1-\delta_F$; the final iteration uses the remaining $b\ge n/2$ queries, so the total is at most $n$, and $H+1\le C\rho^{-1/\gamma}\Lf$.
The choice \eqref{eq:params-3} of $N_{\rm ref}$ balances the first two terms of \eqref{eq:final-prop}, so that \cref{prop:final} gives, on $\scE$,
\begin{align*}
    \bE\qty[\Reg_\rho(\Phat)\relmiddle|\scE]\le C\rho\qty[\frac{(1+\sigma^2)M(H+1)}{b\rho^2}]^p\Lf^{D_{\rm E}}+C\frac{\delta_C}\rho,
\end{align*}
while on $\scE^c$ the regret is at most $4$, since all scores lie in $[-1,1]$.
Since $\Pr(\scE^c)\le\delta_F=\rho^2/(n+1)^2$, $\delta_C/\rho=\rho/(n+1)^2$, $H+1\le C\rho^{-1/\gamma}\Lf$ and $b\ge n/2$, and since $p<1$ and $n\rho^{2+1/\gamma}/M\le n$, the failure terms are absorbed into the main term, and
\begin{align*}
    \bE\qty[\Reg_\rho(\Phat)]
    \le C\rho\qty(\frac{M\rho^{-1/\gamma}}{n\rho^2})^p\Lf^{D}
    =C\rho\qty(\frac{M}{n\rho^{2+1/\gamma}})^p\Lf^{D},
\end{align*}
which is \eqref{eq:upper}; the factor $(1+\sigma^2)^p$ is absorbed into $C$.
The dependence on $M$ enters the starting condition through $M^{\zeta_M}$, i.e., through the front-phase cost of \cref{prop:front}, and the rate itself through the joint estimation error of the final fit, which carries the factor $M$ of \eqref{eq:second-half} only. \qed

\section{Proof of the Small-Regularization Bound of \cref{thm:upper}}
\label{app:small}

In this section, \cref{alg:actor-critic} runs with the small-regularization hyperparameters of \cref{subsec:algorithm}: $\ell_{\rm look}=S=1$, $\Hall=H$ exploration iterations with the query rule \eqref{eq:query-rule} and $\Delta_h=2$, no refinement step, and the output is the last actor $\Phat=P_H=\Gibbs{W_H}{\vartheta_H}$.
The conventions of \cref{app:upper} are in force, with
\begin{align*}
    t=\frac nM,\qquad
    L_n=\log(e+n+M),\qquad
    p_0=\frac{\gamma-1}{2\gamma+1},\qquad
    \mathsf T_h=\sum_{j\ge h}a_j\asymp(h+1)^{-(\gamma-1)},
\end{align*}
and $\Lf$ replaced by $L_n$ in the lemmas of \cref{app:front}.
The regularization $\rho$ enters neither the algorithm nor the analysis until the last step of the proof; in particular $\rho=0$ is allowed throughout.
The lemmas of \cref{app:geometry,app:gibbs-mass,app:transformer} are used as they stand, and the lemmas of \cref{app:front} are used with $S=1$, so that a cube created at step $r$ has depth $r+1$, the accuracy of step $h$ is $\varepsilon_h=\frac\eta4a_{h+1}$, and $p_S=\kappa_02^{-d}=:p_{\rm fine}$.

\paragraph{Proof strategy.}
The proof has the same two stages as that of the large-budget bound, and differs only in the second.
The front phase is the same loop, run with lookahead $S=1$, and \cref{app:small-front} shows that its updates stay birth-accurate and that the budget fixes the depth $H$ it reaches.
There is no final refinement step: the output is the last actor, and \cref{app:small-terminal-sec} converts the reached depth directly into a reward gap through the Gibbs variational principle.
\Cref{app:small-proof} combines the two and absorbs the regularization penalty.

\subsection{Hyperparameters of the Small-Regularization Setting}
\label{app:small-params}

Let $\eta\le\eta_0$ be a fixed structural constant (small enough for \cref{lem:mass} and \cref{lem:envelope}), $\kappa_0=p_\ast/2$ with $p_\ast$ from \cref{lem:mass}, $\vartheta_h=t_0c_{1+\lfloor h/\lstar\rfloor}$ for all $h\ge0$ with $t_0$ as fixed after \cref{lem:mass}, and
\begin{align}
    J&=\lceil C_JL_n\rceil,
    &K_w&=C_d(J+3),
    &N&=\left\lceil C_NM^{\frac d{2\alpha}}L_n^{D_N}\right\rceil,
    \label{eq:small-params}\\
    b_h&=\left\lceil C_b(1+\sigma^2)M^{\zeta_M}a_{h+1}^{-2}L_n^{D_b}\right\rceil,
    &\bar\delta_{H'}&=(H'+1)^{-\gamma},
    &\delta_h&=\frac{\bar\delta_{H'}}{H'+1}\quad(h<H'),
    \notag
\end{align}
for a candidate number of iterations $H'\ge1$; the structural constants $C_J,C_N,C_b,D_N,D_b$ are chosen as in \cref{prop:front}.
The number of iterations is the largest one that the budget allows,
\begin{align}
    H=\max\qty{H'\in\{1,\ldots,n\}:F_n(H'):=\sum_{h<H'}b_h\le n},
    \label{eq:small-H}
\end{align}
and $H=0$ if the set is empty (in particular if $n=0$), in which case the uniform policy is returned without any query.
The input length is $\Tin=dH+3$, and $\Tcls_h=\Tcls_{\rm win}(K_w,N,\Tin)$ of \cref{lem:window}, the same for all $h<H$.
The total number of queries is $F_n(H)\le n$ deterministically, and the choice uses neither the reward, nor the zoom sequence, nor any per-prompt quantity.

\subsection{The Front Phase}
\label{app:small-front}

The front phase is the loop of \cref{app:front} run with $S=1$: \cref{lem:small-progress} shows that its updates stay birth-accurate, and \cref{lem:small-cost} converts the budget into the depth it reaches.

\begin{lemma}[Progress]\label{lem:small-progress}
    With the parameters of \cref{app:small-params}, $\Pr(\scE)\ge1-\bar\delta_H$.
    On $\scE$, for $H\ge1$: the correct cube $P^\ast_H(m)$ is a union of cubes of $\scD_H(m)$ for every $m$, all cubes of $\scD_H(m)$ have depth at most $H$, and $W_H(m,\cdot)$ is $\eta$-birth-accurate on $\scD_H(m)$.
\end{lemma}
\begin{proof}
    With $S=1$, \cref{lem:progress,lem:fine,lem:localize,lem:joint} apply at every step $h<H$ on $\scE_h$, with $p_S=p_{\rm fine}$ and $\Lambda_h\le Ce^{-cJ}$.
    The factors $\eta^{-1}$, $p_{\rm fine}^{-1}$, $K_w$, $\ell_N$ and $\log(2b_h/\delta_h)$ are bounded by fixed powers of $L_n$, while the factor $a_{h+1}^{-2}$ is paid explicitly by $b_h$ in \eqref{eq:small-params}.
    For the leakage, $H\le n$ gives $a_{H+1}^{-2}\le C(n+1)^{2\gamma}$, so choosing $C_J$ large in $J=\lceil C_JL_n\rceil$ yields
    \begin{align*}
        \Lambda_h\le Ce^{-cJ}\le c\,\frac{p_{\rm fine}\eta^2a_{H+1}^2}{M}
        \le c\,\frac{p_{\rm fine}\eta^2a_{h+1}^2}{M}.
    \end{align*}
    The approximation and estimation terms are controlled as in \cref{prop:front}, with $L_n$ in place of $\Lf$.
    Hence, conditionally on $\scE_h$ the step succeeds with probability at least $1-\delta_h$, and $\Pr(\scE^c)\le\sum_{h<H}\delta_h\le\bar\delta_H$.
    The remaining claims are \cref{lem:mask,lem:progress} at $h=H$.
\end{proof}

\begin{lemma}[Cost and the chosen depth]\label{lem:small-cost}
    For $1\le H'\le n$,
    \begin{align}
        c(1+\sigma^2)M^{\zeta_M}H'^{2\gamma+1}\le F_n(H')\le C(1+\sigma^2)M^{\zeta_M}H'^{2\gamma+1}L_n^{D_b},
        \label{eq:small-cost}
    \end{align}
    and the depth \eqref{eq:small-H} satisfies
    \begin{align}
        c\qty(1+\frac{n}{(1+\sigma^2)M^{\zeta_M}L_n^{D_b}})^{\frac1{2\gamma+1}}\le H+1\le C\qty(1+\frac nM)^{\frac1{2\gamma+1}}.
        \label{eq:small-depth}
    \end{align}
\end{lemma}
\begin{proof}
    Both inequalities of \eqref{eq:small-cost} follow from $a_{h+1}^{-2}\asymp(h+2)^{2\gamma}$ and $\sum_{h<H'}(h+2)^{2\gamma}\asymp H'^{2\gamma+1}$, the rounding being absorbed by the constants.
    \emph{Lower bound on $H$.}
    Let $A_n=(1+\sigma^2)M^{\zeta_M}L_n^{D_b}$.
    If $n\ge C_1A_n$ with $C_1$ large, then $H'=\lfloor a_0(n/A_n)^{1/(2\gamma+1)}\rfloor$ with $a_0$ small satisfies $1\le H'\le n$ and $F_n(H')\le n$ by \eqref{eq:small-cost}, so $H\ge H'$; if $n<C_1A_n$, the left inequality of \eqref{eq:small-depth} holds trivially with $c\le(1+C_1)^{-1/(2\gamma+1)}$, including $H=0$.
    \emph{Upper bound.}
    If $H\ge1$, then $n\ge F_n(H)\ge cM^{\zeta_M}H^{2\gamma+1}\ge cMH^{2\gamma+1}$, so $H\le C(n/M)^{1/(2\gamma+1)}$; the case $H=0$ is trivial.
\end{proof}
The cost counts every record: stopped responses, responses of prompts that were never selected, and the batches after a failure.
The factor $M^{\zeta_M}$, rather than $M$, is the price of converting the joint risk of the critic into the accuracy on the update inputs of every prompt (\cref{lem:joint}); the depth reached by the algorithm therefore grows like $(n/M^{\zeta_M})^{1/(2\gamma+1)}$ up to logarithms, and not like $(n/M)^{1/(2\gamma+1)}$.

\subsection{The Terminal Bound}
\label{app:small-terminal-sec}

Unlike the large-budget setting, no further critic is fitted: the output is the last actor and its regret is read off the depth it reached.

\Cref{lem:envelope} turns the heterogeneous birth errors into a one-sided envelope of the score, and \cref{lem:small-terminal} combines that envelope with the Gibbs variational principle and the deterministic KL bound.

Recall that the output is $\Phat=\Gibbs{W_H}{\vartheta_H}$ with $\vartheta_H=t_0c_{1+\lfloor H/\lstar\rfloor}\asymp a_H\asymp(H+1)^{-\gamma}$.

\begin{lemma}[Envelope of a birth-accurate score]\label{lem:envelope}
    Let $\eta\le\eta_0$, fix $m$ and $H\ge1$, and let $W$ be $\eta$-birth-accurate on a partition $\scD$ of $\Omega$ into dyadic cubes of depth at most $H$ such that $P^\ast_H(m)$ is a union of cubes of $\scD$.
    Then, for all $y\in\Omega$,
    \begin{align}
        G_m-f(y;m)&\le C\qty[\qty(G_m-W(y))_++\mathsf T_H],
        \label{eq:envelope-1}\\
        W(y)&\le G_m+Ca_H,
        \label{eq:envelope-2}\\
        W(y)&\ge G_m-C\mathsf T_H\qquad(y\in P^\ast_H(m)).
        \label{eq:envelope-3}
    \end{align}
\end{lemma}
\begin{proof}
    Let $k=k(H)$ and $y\in D\in\scD$.
    By \cref{lem:gap} (applied as in \cref{lem:mass}), either $D\subset E_a(m)$ for some $a<k$, or $D\subset C_{k-1}(m)$.
    \emph{Shells.}
    If $D\subset E_a(m)$, then by \eqref{eq:geo-lower} and \eqref{eq:geo}, $\kappa T_a\le G_m-f(y';m)\le C_{\rm near}T_{a-1}\le CT_a$ for all $y'\in E_a(m)$, so that $G_m-Q(m,D)\ge\kappa T_a$; since $\mathrm{depth}(D)\ge L_{a-1}+1$, the birth error is at most $\frac\eta4a_{\mathrm{depth}(D)}\le C\eta c_a\le C\eta T_a$, and for $\eta$ small
    \begin{align*}
        G_m-W(D)\ge\frac\kappa2T_a>0,\qquad
        G_m-f(y;m)\le CT_a\le C(G_m-W(D))=C(G_m-W(y))_+,
    \end{align*}
    which gives \eqref{eq:envelope-1} and \eqref{eq:envelope-2} on the shells.
    \emph{Inside.}
    If $D\subset C_{k-1}(m)$, then $\mathrm{depth}(D)\ge L_{k-1}+1$ and the birth error is at most $C\eta c_k\le Ca_H$, while $0\le G_m-f(y';m)\le C_{\rm near}T_{k-1}\le C\mathsf T_H$ for $y'\in C_{k-1}(m)$ by \eqref{eq:geo} and $T_{k-1}\asymp\mathsf T_H$; hence $G_m-f(y;m)\le C\mathsf T_H$, $W(D)\le Q(m,D)+Ca_H\le G_m+Ca_H$, and, if $D\subset P^\ast_H(m)\subset C_{k-1}(m)$, $W(D)\ge Q(m,D)-Ca_H\ge G_m-C\mathsf T_H-Ca_H\ge G_m-C\mathsf T_H$.
\end{proof}

\begin{lemma}[Terminal reward and KL]\label{lem:small-terminal}
    For every history and every $m$,
    \begin{align}
        \KL{\Phat(\cdot\mid m)}{\mu}\le dH\log2.
        \label{eq:small-kl}
    \end{align}
    On $\scE$, for every $m$,
    \begin{align}
        G_m-\bE_{\Phat(\cdot\mid m)}\qty[f(\cdot;m)]\le C\mathsf T_H\le C(H+1)^{-(\gamma-1)}.
        \label{eq:small-terminal}
    \end{align}
\end{lemma}
\begin{proof}
    \emph{KL.}
    By \cref{lem:mask}, $\Phat(\cdot\mid m)$ is uniform inside every cube of $\scD_H(m)$, whose cubes have depth at most $H$ (with $S=1$); refining $\scD_H(m)$ to the uniform grid of depth $H$ keeps $\Phat$ uniform on each cell, so with $p_D=\Phat(D\mid m)$ for the $2^{dH}$ cells,
    \begin{align*}
        \KL{\Phat(\cdot\mid m)}{\mu}=\sum_Dp_D\log\frac{p_D}{2^{-dH}}=dH\log2+\sum_Dp_D\log p_D\le dH\log2.
    \end{align*}
    This holds on every history; for $H=0$ the divergence is $0$.

    \emph{Reward.}
    For $H=0$ the claim is trivial since $G_m-\bE_\mu f\le2$; let $H\ge1$ and write $W=W_H(m,\cdot)$, which is $\eta$-birth-accurate on $\scD_H(m)$ on $\scE$ by \cref{lem:small-progress}.
    Let $U_\ast=\mu(\cdot\mid P^\ast_H(m))$ be the uniform distribution on the correct cube, a union of cubes of $\scD_H(m)$.
    The Gibbs variational principle for $\Phat=\Gibbs{W}{\vartheta_H}$, \eqref{eq:envelope-3} and $\KL{U_\ast}{\mu}=dH\log2$ give
    \begin{align*}
        \bE_{\Phat}[W]-\vartheta_H\KL{\Phat}{\mu}\ge\bE_{U_\ast}[W]-\vartheta_H\KL{U_\ast}{\mu}\ge G_m-C\mathsf T_H-\vartheta_HdH\log2,
    \end{align*}
    hence, by $\KL{\Phat}{\mu}\ge0$ and $\vartheta_HH\le Ca_HH\le C\mathsf T_H$,
    \begin{align*}
        \bE_{\Phat}\qty[G_m-W]\le C\mathsf T_H.
    \end{align*}
    By \eqref{eq:envelope-2}, $(G_m-W)_+\le(G_m-W)+Ca_H$, so $\bE_{\Phat}[(G_m-W)_+]\le C\mathsf T_H$, and \eqref{eq:envelope-1} yields \eqref{eq:small-terminal}.
    No sup-norm accuracy of the final score is used: the old cubes carry only the accuracy of their own birth, and the envelope of \cref{lem:envelope} is what the variational comparison needs.
\end{proof}

\subsection{Proof of the Small-Regularization Bound}
\label{app:small-proof}

The number of queries is $F_n(H)\le n$ by \eqref{eq:small-H}.
For every $m$, $J^\ast_\rho(m)=\rho\log\int e^{f(y;m)/\rho}\dd{\mu}(y)\le G_m$ (and $J^\ast_0(m)=G_m$), so, by \eqref{eq:small-kl},
\begin{align*}
    \Reg_\rho(\Phat)
    &=\frac1M\sum_{m}\qty[J^\ast_\rho(m)-\bE_{\Phat(\cdot\mid m)}[f(\cdot;m)]+\rho\KL{\Phat(\cdot\mid m)}{\mu}]\\
    &\le\frac1M\sum_m\qty[G_m-\bE_{\Phat(\cdot\mid m)}[f(\cdot;m)]]+\rho\,dH\log2.
\end{align*}
On $\scE$ the bracket is at most $C(H+1)^{-(\gamma-1)}$ by \eqref{eq:small-terminal}, and on $\scE^c$ it is at most $2$ since $f\in[-1,1]$; $\Pr(\scE^c)\le\bar\delta_H=(H+1)^{-\gamma}$ by \cref{lem:small-progress}.
Therefore
\begin{align}
    \bE\qty[\Reg_\rho(\Phat)]\le C(H+1)^{-(\gamma-1)}+\rho\,dH\log2,
    \label{eq:small-regret-H}
\end{align}
and for $\rho=0$ the last term vanishes.
For $0<\rho\le c_{\rm S}(1+t)^{-\gamma/(2\gamma+1)}$, the upper bound in \eqref{eq:small-depth} gives
\begin{align*}
    \rho\,dH\log2
    &\le C\rho(H+1)^\gamma(H+1)^{-(\gamma-1)}\\
    &\le Cc_{\rm S}(1+t)^{-\frac\gamma{2\gamma+1}}(1+t)^{\frac\gamma{2\gamma+1}}(H+1)^{-(\gamma-1)}
    \le C(H+1)^{-(\gamma-1)}.
\end{align*}
In both cases, the lower bound in \eqref{eq:small-depth} yields
\begin{align*}
    \bE\qty[\Reg_\rho(\Phat)]\le C(H+1)^{-(\gamma-1)}\le C\qty(1+\frac{n}{M^{\zeta_M}L_n^{D_b}})^{-p_0},
\end{align*}
where the factor $(1+\sigma^2)$ is absorbed into $C$; this is the first case of \eqref{eq:upper}, and the case $H=0$ (uniform output, regret at most $2$) is included by adjusting the constant.
The constant $c_{\rm S}$ is a structural constant that may be taken as small as we like; on the range of the small-regularization bound we have $(1+t)\rho^{2+1/\gamma}\le c_{\rm S}^{2+1/\gamma}$, so the first lower bound of \eqref{eq:lower} applies on that whole range as soon as $c_{\rm S}^{2+1/\gamma}\le c_0$, where $c_0$ is the constant of the small-regularization bound of \cref{thm:lower}. \qed

\section{Proof of \cref{thm:lower}}
\label{app:lower}

Throughout this section, $\sigma>0$ is fixed and $\lambda\ge\Lambda_{\alpha,d}$.
In the lower bound we use independent $N(0,\sigma^2)$ noise, which belongs to $\scN_\sigma$, and we fix the zoom depths $\ell_i\equiv\lstar$; we may even reveal them to the learner, which only strengthens the statement.
We write
\begin{align*}
    t=\frac nM,\qquad p=\frac{2\alpha}{2\alpha+d},\qquad p_0=\frac{\gamma-1}{2\gamma+1},\qquad q=2^{-\lstar}.
\end{align*}
Both bounds are first proved for a single prompt and then transferred to $M$ prompts; the transfer is \cref{lem:bayes-transfer} of \cref{app:lower-reduction} and is common to the two.
The rest of the section proves them one after the other, each in three steps: a hard family, the information bound for that family, and the assembly.
The small-regularization bound comes first.
Its family, built in \cref{app:two-choice}, places the bump of each level in one of two cubes; \cref{app:lower-small} shows that a cube query at any depth carries little information about those placements, so that a constant fraction of them stays undecodable within the budget, and \cref{app:lower-small-conclusion} assembles the bound.
The large-budget bound follows.
Its family, built in \cref{app:lower-family}, has a Gibbs policy that spreads a comparable mass over a band of levels around $\rho^{-1/\gamma}$ and perturbs each level by bumps with random signs; \cref{app:lower-decode} bounds the information that the queries carry about those signs, and \cref{app:lower-conclusion} assembles the bound.

\subsection{Reduction to a Single Prompt}
\label{app:lower-reduction}
For a fixed prompt, the problem is that of a single sequence $(g_i)_{i\ge1}$ with regret $J^\ast_\rho-J_\rho(P)$.
For a prior $\pi$ on single-prompt instances, let $B_\pi(k,\rho)$ be the Bayes risk of the best algorithm that uses $k$ adaptive cube queries; it is nonincreasing in $k$.
\begin{lemma}[Bayes transfer]\label{lem:bayes-transfer}
    Let $k=\max\{1,\lceil2t\rceil\}$, let $\pi$ be any prior on single-prompt instances that may depend on $(n,M,\rho)$ but not on the data, and put the product prior $\pi^{\otimes M}$ on the $M$ sequences. Then
    \begin{align}
        \Rmin(n,M,\rho)\ge\frac12B_\pi(k,\rho).
        \label{eq:bayes-transfer}
    \end{align}
\end{lemma}
\begin{proof}
    Consider the strengthened experiment in which the learner is told the whole prompt sequence $X_1,\ldots,X_n$ in advance and, when evaluated on prompt $m$, the true functions of all other prompts.
    Since the sequences are independent under the product prior and the noise is independent across rounds, the answers to queries on other prompts can be simulated from the known functions and independent noise; hence the only source of information about prompt $m$ is its own
    \begin{align*}
        N_m=\sum_{t'=1}^n\ind{X_{t'}=m}\sim\Bin(n,1/M)
    \end{align*}
    queries, and adaptivity across prompts does not help.
    The Bayes regret on prompt $m$ is therefore at least
    \begin{align*}
        \bE_{N_m}B_\pi(N_m,\rho)\ge\Pr(N_m\le k)B_\pi(k,\rho)\ge\frac12B_\pi(k,\rho),
    \end{align*}
    where the last step is Markov's inequality with $\bE N_m=t\le k/2$.
    Averaging over $m$ and using that the minimax risk dominates the Bayes risk gives \eqref{eq:bayes-transfer}.
    The prior is fixed before the data and is not re-chosen according to the realized $N_m$.
\end{proof}
\begin{lemma}[KL divergence of adaptive cube queries]\label{lem:kl-adaptive}
    Let $f,f'$ be two instances and let $\bar P_f,\bar P_{f'}$ be the laws of the transcript (internal randomness, prompts, queried cubes, observations) of an algorithm with $k$ queries under Gaussian noise. Then
    \begin{align}
        \KL{\bar P_f}{\bar P_{f'}}=\frac1{2\sigma^2}\bE_f\sum_{r'=1}^k\qty[Q_f(X_{r'},D_{r'})-Q_{f'}(X_{r'},D_{r'})]^2.
        \label{eq:kl-adaptive}
    \end{align}
\end{lemma}
\begin{proof}
    By the chain rule, the kernels of the internal randomness, of the prompt and of the choice of $D_{r'}$ given the past are identical under both laws and contribute zero; the observation kernel given $(X_{r'},D_{r'})$ is Gaussian with the same variance, and the KL divergence of two such Gaussians is the squared mean difference divided by $2\sigma^2$.
\end{proof}

\subsection{A Two-Choice Family}
\label{app:two-choice}
For $r_c=1/8$ and $m_\alpha=\lceil\alpha\rceil+1$, define on $\bR^d$
\begin{align}
    \psi_0(u)&=\begin{cases}
        [1-(\norm{u}_\infty/r_c)^\alpha]_+,&0<\alpha\le1,\\
        [1-\norm{u}_2^2/r_c^2]_+^{m_\alpha},&\alpha>1,
    \end{cases}
    \qquad \psi_z(u)=\psi_0(u-z),\notag\\
    B_{\alpha,d}&=\max\{8^\beta,\,|\psi_0|_\alpha\},\qquad
    \Lambda_{\alpha,d}=2B_{\alpha,d}.
    \label{eq:richness-threshold}
\end{align}
Here $|\cdot|_\alpha$ is the zero-extension seminorm in \assmp{A1}.
For $\alpha\le1$, $B_{\alpha,d}=8^\alpha$, so $\Lambda_{\alpha,d}=2\cdot8^\alpha$ as before.
For $\alpha>1$, the truncated polynomial is $C^{m_\alpha-1}$ with bounded derivatives, hence $B_{\alpha,d}<\infty$, including integer $\alpha$.
In both cases $\psi_0$ has height one, a unique maximum at zero, and support inside $[-r_c,r_c]^d$; its value is strictly decreasing with radius inside its support.
We use the following elementary gluing bound: disjoint translates of $\psi_0$, with either sign, have seminorm at most $2B_{\alpha,d}$.
Indeed, at each point only one component contributes to any derivative; for the highest derivative at points in different supports, insert a zero value at each support boundary and bound each distance by the distance between the two points.
The same argument applies to disjoint copies $A\psi_0((u-v)/h)$ with $h\le1$ and $|A|h^{-\alpha}\le1$: every derivative of order $j\le\lceil\alpha\rceil-1$ is bounded by the corresponding derivative of $\psi_0$, since $|A|h^{-j}\le1$, and the highest derivative has the same H\"older bound before gluing.
Let $r_c=1/8$, $q=2^{-\lstar}$, $\ell_i\equiv\lstar$, and
\begin{align*}
    z_0&=(1/4+q/2,\ 1/2+q/2,\ldots,1/2+q/2),\\
    z_1&=(3/4+q/2,\ 1/2+q/2,\ldots,1/2+q/2),\\
    z_-&=(1/2+q/2,\ 1/2+q/2,\ldots,1/2+q/2).
\end{align*}
For bits $\omega_{i,m}\in\{0,1\}$ let
\begin{align}
    g_{1,m}=\psi_{z_{\omega_{1,m}}},\qquad
    g_{i,m}=\psi_{z_{\omega_{i,m}}}-\psi_{z_-}\quad(i\ge2).
    \label{eq:two-choice}
\end{align}
Let $Q_b$ be the depth-$\lstar$ dyadic cube centered at $z_b$, for $b\in\{0,1\}$.
Explicitly,
\begin{align*}
    Q_0&=[1/4,1/4+q)\times[1/2,1/2+q)^{d-1},\\
    Q_1&=[3/4,3/4+q)\times[1/2,1/2+q)^{d-1}.
\end{align*}
The bits $\omega$ index reward instances, not response bits: $\omega_{i,m}=b$ selects a rescaled copy of $Q_b$ inside $C_{i-1}(m)$ as the next active cell.
The same disjoint local pair is used at every level and does not cover $\Omega$.
By \assmp{A3}, $q<r_c$.
The bumps have disjoint supports and equal integrals, so the deeper primitives have zero mean; their zero extensions have $|g_i^0|_\alpha\le\Lambda_{\alpha,d}\le\lambda$ by the gluing bound.
Each primitive has maximum one at $z_{\omega_{i,m}}$, so the family lies in $\scF_{M,\theta}$ whenever $\lambda\ge\Lambda_{\alpha,d}$.
The adaptive lower bound of \cref{app:lower-small} and the fixed-sampling bounds of \cref{app:fixed} both use this family.

\subsection{Small Regularization: The Adaptive Exploration Lower Bound}
\label{app:lower-small}

By \cref{lem:bayes-transfer} and $1+k\le3(1+t)$, it suffices to exhibit, for a single prompt, a prior $\pi_{k,0}$ with
\begin{align}
    B_{\pi_{k,0}}(k,0)\ge c(1+k)^{-p_0},
    \label{eq:one-prompt-small}
\end{align}
which \cref{lem:reg-transfer} then transfers from $\rho=0$ to $\rho>0$.

\begin{lemma}[The regularization changes the optimal value by $O(\rho^{1-1/\gamma})$]\label{lem:reg-transfer}
    There is a constant $U$ depending only on $\theta$ such that for all $f\in\scF_{M,\theta}$, $m\in[M]$ and $\rho\in(0,1]$,
    \begin{align}
        0\le G_m-J^\ast_\rho(m)\le U\rho^{1-1/\gamma}.
        \label{eq:reg-gap}
    \end{align}
    Consequently, for every policy $P$, $\Reg_\rho(P)\ge\Reg_0(P)-U\rho^{1-1/\gamma}$, and
    \begin{align}
        \Rmin(n,M,\rho)\ge\qty[\Rmin(n,M,0)-U\rho^{1-1/\gamma}]_+.
        \label{eq:reg-transfer}
    \end{align}
\end{lemma}
\begin{proof}
    Since $f\le G_m$, $J^\ast_\rho(m)=\rho\log\int e^{f/\rho}\dd{\mu}\le G_m$.
    Let $K'=\lceil\rho^{-1/\gamma}\rceil$ and let $U_{K'}=\mu(\cdot\mid C_{K'}(m))$; by the variational formula, \eqref{eq:geo}, \eqref{eq:tail} and $L_{K'}\le K'\lbar\le2\lbar\rho^{-1/\gamma}$,
    \begin{align*}
        J^\ast_\rho(m)&\ge\bE_{U_{K'}}f-\rho\KL{U_{K'}}{\mu}
        \ge G_m-C_{\rm near}T_{K'}-\rho\,dL_{K'}\log2\\
        &\ge G_m-\qty[C_{\rm near}C_T+2d\lbar\log2]\rho^{1-1/\gamma}.
    \end{align*}
    For a policy $P$, $\Reg_\rho(P)=\frac1M\sum_m[J^\ast_\rho(m)-\bE_Pf+\rho\KL{P}{\mu}]\ge\frac1M\sum_m[G_m-\bE_Pf]-U\rho^{1-1/\gamma}=\Reg_0(P)-U\rho^{1-1/\gamma}$.
    Taking the supremum over instances and the infimum over algorithms (an algorithm for the problem with parameter $\rho$ is an algorithm for $\rho=0$ with the same queries) gives \eqref{eq:reg-transfer}, together with $\Rmin\ge0$.
\end{proof}

\begin{lemma}[Bayes lower bound for adaptive exploration of one prompt]\label{lem:adaptive-lower}
    Let $M=1$ and $\rho=0$.
    For every integer $s\ge2$ there is a prior $\pi_s$ on the two-choice family \eqref{eq:two-choice} such that every algorithm with $k\ge0$ adaptive cube queries has Bayes regret
    \begin{align}
        B_{\pi_s}(k,0)\ge\frac{\kappa T_{2s-1}}4\qty[1-\frac{C_{\rm adv}\,k\,c_s^2}{8\sigma^2s}]_+^2,
        \qquad C_{\rm adv}=8\sum_{r\ge1}r^22^{-r/2}.
        \label{eq:adaptive-finite}
    \end{align}
    Consequently, with $s_k=\max\{2,\lceil(C_{\rm adv}k/(4\sigma^2\zeta(\gamma)^2))^{1/(2\gamma+1)}\rceil\}$,
    \begin{align}
        B_{\pi_{s_k}}(k,0)\ge c(1+k)^{-p_0}.
        \label{eq:one-prompt-lower-0}
    \end{align}
\end{lemma}
\begin{proof}
    We first prove the finite-$s$ bound by placing a prior on the unknown placements and pairing instances whose words to be decoded differ.
    For each fixed response, the reward has a spine structure that controls the information in any cube query.
    A Hellinger-affinity potential carries this control through adaptive histories; decoding the output policy then turns the remaining uncertainty into regret.
    The stated choice of $s$ yields the rate.

    \emph{Prior and comparison kernel.}
    Let $\pi=\pi_s$ be the law of the reward-instance bits $\omega=(\omega_i)_{i\ge1}$, where
    \begin{align*}
        \omega_i=0\quad(i<s),\qquad
        (\omega_i)_{i\ge s}\text{ are independent fair bits}.
    \end{align*}
    The corresponding reward is $f_\omega$ from \eqref{eq:two-choice}; by \cref{app:two-choice}, every $f_\omega$ lies in the class.
    The word to decode and its index set are
    \begin{align*}
        \varpi(\omega)=(\omega_s,\ldots,\omega_{2s-1}),\qquad
        \scI_s=\{s,\ldots,2s-1\}.
    \end{align*}
    For $i\in\scI_s$, define $\Gamma_i$ on $L^2(\pi)$ by keeping the bits before $i$, flipping $\omega_i$, and resampling the bits after $i$.
    This is an analytical comparison between reward instances, not an operation performed by the learner.
    Its kernel density with respect to $\pi$ is
    \begin{align*}
        \scK_i(\omega,\omega')=2^{i-s+1}\ind{\omega_{s:i-1}=\omega'_{s:i-1},\ \omega_i\ne\omega'_i}.
    \end{align*}
    It is symmetric, nonnegative, and has unit row integrals, so $\norm{\Gamma_i}_{L^2(\pi)\to L^2(\pi)}\le1$.
    Set
    \begin{align*}
        \Gamma=\sum_{i\in\scI_s}\Gamma_i,\qquad
        \scK=\sum_{i\in\scI_s}\scK_i.
    \end{align*}
    Then $\Gamma\bfone=s\bfone$ and $\norm{\Gamma}\le s$.
    Moreover, $\scK(\omega,\omega')>0$ implies $\varpi(\omega)\ne\varpi(\omega')$.

    \emph{The reward at a fixed point.}
    Fix $y\in\Omega$.
    We claim that there are bits $\mathsf a=(\mathsf a_s,\mathsf a_{s+1},\ldots)$ and reals $(F_j)_{j\ge s}$ such that, with
    \begin{align}
        \scB_j=\{\omega:\omega_{s:j-1}=\mathsf a_{s:j-1},\ \omega_j\ne \mathsf a_j\},\qquad \pi(\scB_j)=2^{-(j-s+1)},\qquad j\ge s,
        \label{eq:spine-sets}
    \end{align}
    which partition $\pi$-almost all of the space,
    \begin{align}
        f_\omega(y)=F_j\ \ (\omega\in\scB_j),\qquad \abs{F_{j+1}-F_j}\le2c_j,\qquad\text{so}\quad\abs{F_j-F_i}\le2c_s(j-i)\ \ (s\le i<j).
        \label{eq:spine}
    \end{align}
    The levels before $s$ and the cell $C_{s-1}$ do not depend on $\omega$.
    If $y\notin C_{s-1}$, all later components vanish at $y$.
    Thus $f_\omega(y)=\sum_{i<s}c_ih_i(y)$ for every $\omega$; choose any $\mathsf a$ and set every $F_j$ to this common value.

    Suppose now that $y\in C_{s-1}$.
    Choose $\mathsf a$ recursively: given $\mathsf a_{s:i-1}$ with $y\in C_{i-1}(\mathsf a_{<i})$, let $w$ be the local coordinate of $y$ in that cell.
    There are two cases because $Q_0$ and $Q_1$ do not cover $\Omega$.

    If $w\in Q_b$ for $b\in\{0,1\}$, set $\mathsf a_i=b$.
    For $\omega\in\scB_i$ we have $\omega_i=1-b$, so $y\notin C_i(\omega)$ and all components after level $i$ vanish at $y$.
    Also, $w$ is at $\norm{\cdot}_\infty$-distance greater than $r_c$ from both $z_{1-b}$ and $z_-$, so the level-$i$ component vanishes in this two-choice family.
    Therefore
    \begin{align*}
        f_\omega(y)=\sum_{l<i}c_lh_l(y)=:F_i\qquad(\omega\in\scB_i).
    \end{align*}
    The earlier reward components in this sum need not vanish.

    If $w\notin Q_0\cup Q_1$, then $y\notin C_i(\omega)$ for either choice of $\omega_i$.
    All components after level $i$ vanish, while the level-$i$ component may be nonzero:
    \begin{align*}
        f_\omega(y)=\sum_{l<i}c_lh_l(y)
        +c_i\bigl[\psi_{z_{\omega_i}}(w)-\psi_{z_-}(w)\bigr].
    \end{align*}
    Set $\mathsf a_i=0$ and let $F_i$ be the value for $\omega_i=1$.
    For $j>i$, set $\mathsf a_j=0$ and let $F_j$ be the value for $\omega_i=0$; the recursion stops because later bits have no effect on the reward at $y$.

    When the recursion continues through level $i$, moving the first mismatch from $i$ to $i+1$ changes the level-$i$ contribution by at most $c_i$, and may add a level-$(i+1)$ contribution of at most $c_{i+1}$ if the recursion stops there.
    At the stopping level the difference is at most
    \begin{align*}
        c_i\abs{\psi_{z_0}(w)-\psi_{z_1}(w)}\le c_i;
    \end{align*}
    thereafter the values $F_j$ are constant.
    The bumps have disjoint supports and height at most one, so in all cases $\abs{F_{j+1}-F_j}\le c_j+c_{j+1}\le2c_j$.
    Summing these differences and using $c_l\le c_s$ for $l\ge s$ proves the last bound in \eqref{eq:spine}.

    \emph{Information of one cube query.}
    Let $\mathsf T_y$ be the operator on $L^2(\pi)$ with kernel
    \begin{align*}
        \scK(\omega,\omega')\bigl(f_\omega(y)-f_{\omega'}(y)\bigr)^2,
    \end{align*}
    and normalize the indicators of the sets in \eqref{eq:spine-sets} by
    \begin{align*}
        \mathsf b_j=\frac{\ind{\scB_j}}{\sqrt{\pi(\scB_j)}}.
    \end{align*}
    Take $\omega\in\scB_i$ and $\omega'\in\scB_j$ with $i<j$.
    Their bits agree before $i$ and differ at $i$, so only $\scK_i$ can be nonzero; this requires $i\in\scI_s$.
    On $\scB_j\times\scB_j$ the reward difference is zero.
    Thus the kernel of $\mathsf T_y$ is constant on each block $\scB_i\times\scB_j$, and the operator vanishes outside the closed span of $(\mathsf b_j)_{j\ge s}$.
    For $i<j$ its matrix entries are
    \begin{align}
        \ev{\mathsf b_i,\mathsf T_y\mathsf b_j}=\ind{i\in\scI_s}\,2^{i-s+1}(F_j-F_i)^2\sqrt{\pi(\scB_i)\pi(\scB_j)}
        =\ind{i\in\scI_s}\,2^{-(j-i)/2}(F_j-F_i)^2.
        \label{eq:adaptive-matrix}
    \end{align}
    By \eqref{eq:spine},
    \begin{align*}
        (F_j-F_i)^2\le4c_s^2(j-i)^2.
    \end{align*}
    Each row of this symmetric matrix therefore has absolute sum at most
    \begin{align*}
        8c_s^2\sum_{r\ge1}r^22^{-r/2}
        =C_{\rm adv}c_s^2.
    \end{align*}
    The Schur test gives $\norm{\mathsf T_y}\le C_{\rm adv}c_s^2$.
    For a dyadic cube $D$ and a nonnegative $v\in L^2(\pi)$, Jensen's inequality applied to the cube average and $\scK\ge0$ give
    \begin{align}
        \iint\scK(\omega,\omega')v(\omega)v(\omega')\qty(Q_\omega(D)-Q_{\omega'}(D))^2\dd{\pi}(\omega)\dd{\pi}(\omega')
        &\le\frac1{\mu(D)}\int_D\ev{v,\mathsf T_yv}\dd{\mu}(y)\notag\\
        &\le C_{\rm adv}c_s^2\norm{v}_{L^2(\pi)}^2,
        \label{eq:adaptive-cube}
    \end{align}
    uniformly in the depth of $D$.

    \emph{A potential along the adaptive history.}
    Fix an algorithm and use $N(0,\sigma^2)$ observation noise.
    Let $P_{\omega,t}$ be the law under $f_\omega$ of the history $\mathsf h_t$: the internal randomness, queried cubes, and observations through query $t$.
    Let $P_{{\rm ref},t}$ be the law of the same algorithm when every observation is $N(0,\sigma^2)$ regardless of the query.
    This is a comparison law and need not be a reward instance.
    The Gaussian observation densities are positive, so the likelihood ratio exists and satisfies
    \begin{align*}
        \mathsf L_t(\omega,\mathsf h)
        =\frac{\dd{P_{\omega,t}}}{\dd{P_{{\rm ref},t}}}(\mathsf h),
        \qquad
        \int\mathsf L_t(\omega,\mathsf h)\dd{P_{{\rm ref},t}}(\mathsf h)=1.
    \end{align*}
    For probability laws $P,P'$ dominated by a common law $R$, write their Hellinger affinity as
    \begin{align*}
        \operatorname{Aff}(P,P')
        :=\int\sqrt{\frac{\dd{P}}{\dd{R}}\frac{\dd{P'}}{\dd{R}}}\,\dd{R}.
    \end{align*}
    Put $v_{\mathsf h}(\omega)=\sqrt{\mathsf L_t(\omega,\mathsf h)}$ and define
    \begin{align*}
        \Phi_t
        &=\int\ev{v_{\mathsf h},\Gamma v_{\mathsf h}}
            \dd{P_{{\rm ref},t}}(\mathsf h),\\
        \Phi_0&=\ev{\bfone,\Gamma\bfone}=s,
        \qquad
        \int\norm{v_{\mathsf h}}^2_{L^2(\pi)}
            \dd{P_{{\rm ref},t}}(\mathsf h)=1.
    \end{align*}
    Expanding $\Gamma$ with $R=P_{{\rm ref},t}$ shows what this potential measures:
    \begin{align*}
        \Phi_t
        &=\iint\scK(\omega,\omega')
        \left[
            \int\sqrt{\mathsf L_t(\omega,\mathsf h)\mathsf L_t(\omega',\mathsf h)}
            \dd{P_{{\rm ref},t}}(\mathsf h)
        \right]
        \dd{\pi}(\omega)\dd{\pi}(\omega')\\
        &=\iint\scK(\omega,\omega')\operatorname{Aff}(P_{\omega,t},P_{\omega',t})
        \dd{\pi}(\omega)\dd{\pi}(\omega').
    \end{align*}
    The affinity equals one for identical laws and becomes smaller as they separate; $\Phi_t$ sums these affinities over pairs selected by $\scK$.

    Condition on the current history and any fresh algorithmic randomness used to choose the next cube.
    The resulting cube $D_{t+1}$ is then the same under every reward instance.
    For arbitrary $u,u'\in\bR$, completing the square gives
    \begin{align*}
        \operatorname{Aff}\bigl(N(u,\sigma^2),N(u',\sigma^2)\bigr)
        &=\int_{\bR}\frac{\exp\bigl(-[(z-u)^2+(z-u')^2]/(4\sigma^2)\bigr)}{\sqrt{2\pi\sigma^2}}\,\dd{z}\\
        &=\exp\left[-\frac{(u-u')^2}{8\sigma^2}\right]
        \int_{\bR}\frac{\exp\bigl(-(z-(u+u')/2)^2/(2\sigma^2)\bigr)}{\sqrt{2\pi\sigma^2}}\,\dd{z}\\
        &=\exp\left[-\frac{(u-u')^2}{8\sigma^2}\right].
    \end{align*}
    Thus, for the next query,
    \begin{align*}
        \operatorname{Aff}\bigl(N(Q_\omega(D_{t+1}),\sigma^2),N(Q_{\omega'}(D_{t+1}),\sigma^2)\bigr)
        =\exp\left[-\frac{\bigl(Q_\omega(D_{t+1})-Q_{\omega'}(D_{t+1})\bigr)^2}{8\sigma^2}\right].
    \end{align*}
    Therefore $1-e^{-z}\le z$ and \eqref{eq:adaptive-cube} yield
    \begin{align}
        0\le\Phi_t-\Phi_{t+1}
        &=\int\iint\scK(\omega,\omega')v_{\mathsf h}(\omega)v_{\mathsf h}(\omega')\qty[1-e^{-(Q_\omega(D_{t+1})-Q_{\omega'}(D_{t+1}))^2/(8\sigma^2)}]\dd{\pi}\dd{\pi}\dd{P_{\rm ref}}\notag\\
        &\le\frac{C_{\rm adv}c_s^2}{8\sigma^2},
        \label{eq:adaptive-progress}
    \end{align}
    so that $\Phi_k\ge s-C_{\rm adv}kc_s^2/(8\sigma^2)$.
    The bound is applied to the likelihood vector of each history, so it covers adaptive queries and their dependence on the unknown bits; if the algorithm stops before $k$ queries, the history is completed by uninformative symbols.

    \emph{Decoding from the output.}
    After the $k$ queries, the algorithm outputs a policy $\Phat(\mathsf h_k)$.
    Draw $Y\sim\Phat(\mathsf h_k)$ and append it to the history.
    This conditional sampling rule is common to all $\omega$, so $\mathsf L_k$ and $\Phi_k$ are unchanged.

    The word $\varpi(\omega)$ determines $C_{2s-1}(\omega)$, and different words give disjoint cells.
    Decode $\hat\varpi$ as the word whose cell contains $Y$, choosing an arbitrary word if $Y$ lies in none.
    Then
    \begin{align*}
        \{\hat\varpi\ne\varpi(\omega)\}
        \subseteq \{Y\notin C_{2s-1}(\omega)\}.
    \end{align*}
    For an augmented history $\mathsf h$, define the error set and its Bayes probability by
    \begin{align*}
        \scM_{\mathsf h}
        &=\{\omega:\hat\varpi(\mathsf h)\ne\varpi(\omega)\},\\
        \epsilon
        &=\int\norm{v_{\mathsf h}\ind{\scM_{\mathsf h}}}^2_{L^2(\pi)}
            \dd{P_{\rm ref}}(\mathsf h).
    \end{align*}
    If $\scK(\omega,\omega')>0$, the two words differ, so the decoder must err on at least one of them.
    Symmetry of $\scK$, $\norm{\Gamma}\le s$, and Cauchy--Schwarz give
    \begin{align*}
        \Phi_k\le2\int\ev{v_{\mathsf h}\ind{\scM_{\mathsf h}},\Gamma v_{\mathsf h}}\dd{P_{\rm ref}}
        \le2s\qty(\int\norm{v_{\mathsf h}\ind{\scM_{\mathsf h}}}^2\dd{P_{\rm ref}})^{1/2}\qty(\int\norm{v_{\mathsf h}}^2\dd{P_{\rm ref}})^{1/2}
        =2s\sqrt\epsilon.
    \end{align*}
    Combining this with \eqref{eq:adaptive-progress} gives
    \begin{align*}
        \sqrt\epsilon
        \ge\frac12\left[1-\frac{C_{\rm adv}kc_s^2}{8\sigma^2s}\right]_+.
    \end{align*}

    \emph{Regret.}
    If $Y\notin C_{2s-1}(\omega)$, then $Y\in E_i(\omega)$ for some $i\le2s-1$.
    By \eqref{eq:geo-lower},
    \begin{align*}
        G-f_\omega(Y)\ge\kappa T_i\ge\kappa T_{2s-1}.
    \end{align*}
    The Bayes regret at $\rho=0$ is therefore at least
    \begin{align*}
        \kappa T_{2s-1}\Pr\bigl(Y\notin C_{2s-1}(\omega)\bigr)
        \ge\kappa T_{2s-1}\epsilon,
    \end{align*}
    which proves \eqref{eq:adaptive-finite}.

    Finally, the stated choice $s=s_k$ makes
    \begin{align*}
        \frac{C_{\rm adv}kc_s^2}{8\sigma^2s}
        =\frac{C_{\rm adv}k}
               {8\sigma^2\zeta(\gamma)^2s^{2\gamma+1}}
        \le\frac12.
    \end{align*}
    Hence $B_{\pi_{s_k}}(k,0)\ge\kappa T_{2s_k-1}/16
    \ge c(2s_k)^{-(\gamma-1)}$ by \eqref{eq:tail}.
    Since $s_k\le C(1+k)^{1/(2\gamma+1)}$, this proves \eqref{eq:one-prompt-lower-0}.
\end{proof}

\subsection{Small Regularization: Conclusion}
\label{app:lower-small-conclusion}

For $\rho=0$, \cref{lem:bayes-transfer} with $k=\max\{1,\lceil2t\rceil\}$ and the product prior $\pi_{s_k}^{\otimes M}$ (the primitives of different prompts are independent copies of the two-choice family, and the prompts are exogenous), together with \eqref{eq:one-prompt-lower-0} and $1+k\le3(1+t)$, gives
\begin{align*}
    \Rmin(n,M,0)\ge\frac12B_{\pi_{s_k}}(k,0)\ge c(1+t)^{-p_0}.
\end{align*}
For $\rho\in(0,1]$, \eqref{eq:reg-transfer} and the identity
\begin{align}
    \rho^{1-1/\gamma}=(1+t)^{-p_0}u^{p_0},\qquad u=(1+t)\rho^{2+1/\gamma}
    \label{eq:scale-identity}
\end{align}
(note that $(2+1/\gamma)p_0=1-1/\gamma$) give
\begin{align*}
    \Rmin(n,M,\rho)\ge\qty[c-Uu^{p_0}]_+(1+t)^{-p_0}\ge\frac c2(1+t)^{-p_0}
\end{align*}
whenever $u\le c_0$ with $Uc_0^{p_0}\le c/2$, which is the first case of \eqref{eq:lower}. \qed

\subsection{Large Budget: The Flat-Band Family}
\label{app:lower-family}

By \cref{lem:bayes-transfer} and $1+k\le3(1+t)$, it now suffices to exhibit, for a single prompt and every $k\ge1$, a prior $\pi_{k,\rho}$ with
\begin{align}
    B_{\pi_{k,\rho}}(k,\rho)\ge c\rho\min\qty{1,\qty[(1+k)\rho^{2+1/\gamma}]^{-p}}.
    \label{eq:one-prompt-lower}
\end{align}
Let $r_c=1/8$, $a=d\log(1/q)$ and let $Q_\flat$ be the cube in which every primitive of this family places its maximizer:
\begin{align*}
    Q_\flat=[1/4,1/4+q]^d,\qquad
    \mathsf T(u)=\tfrac14\bfone+qu,\qquad
    z=\frac1{4(1-q)}\bfone,\qquad
    w=\tfrac34\bfone,
\end{align*}
so that $z$ is the fixed point of $\mathsf T$.
Write $e_1=(1,0,\ldots,0)\in\bR^d$ for the first coordinate vector.
Use the bump $\psi_v$ from \eqref{eq:richness-threshold}, and for $v_i=z+t_ie_1$ with $q/16\le t_i\le q/8$ define the base family
\begin{align}
    g_1=\psi_{v_1},\qquad
    g_i=\psi_{v_i}-\psi_w\quad(i\ge2).
    \label{eq:base-family}
\end{align}
By \assmp{A3} and $\lambda\ge\Lambda_{\alpha,d}\ge2\cdot8^\beta$,
\begin{align*}
    \qty(\frac q{r_c})^{\beta}\le\frac{\cmin}{16\cdot3^\gamma},\qquad\text{so}\qquad q\le\frac{r_c}{48},
\end{align*}
and $v_i$ is at distance at least $q/4$ from the boundary of $Q_\flat$.
Hence the maximizer of $g_i$ is $v_i\in Q_\flat$, the maximum is $1$, the placement cube is $Q_\flat$ at every level, $|g_i^0|_\alpha\le\lambda$, $\int g_i=0$ for $i\ge2$, and the ranges are as required.
The value at the fixed point,
\begin{align*}
    a_i:=g_i(z)=\psi_0(t_ie_1)\in[a_-,a_+],\qquad
    a_-=\psi_0((q/8)e_1),\quad a_+=\psi_0((q/16)e_1),
\end{align*}
can be chosen continuously in $[a_-,a_+]\subset(3/4,1)$, and $g_i(u)\le a_i$ for $u\notin Q_\flat$.
Indeed, radial monotonicity gives a strictly positive width $a_+-a_->0$ depending only on the structural parameters, and $1-a_-\le B_{\alpha,d}(q/8)^\beta<1/4$ by \assmp{A3}.
For $u\notin Q_\flat$, the distance from $u$ to $v_i$ in either relevant radial norm is at least $q/4>t_i$; radial monotonicity and $\psi_w\ge0$ therefore give $g_i(u)\le a_i$.

\paragraph{Flat band.}
Let
\begin{align*}
    \tau=\rho^{-1/\gamma},\qquad
    u_0=\qty(\frac{a_-}{a\zeta(\gamma)})^{1/\gamma},\qquad
    v_0=\qty(\frac{a_+}{a\zeta(\gamma)})^{1/\gamma},\qquad
    \eta_0=v_0-u_0>0,
\end{align*}
$\tau_0=\max\{1,2/u_0,4/\eta_0\}$ and $\rho_{\rm flat}=\tau_0^{-\gamma}$.
Assume first $0<\rho\le\rho_{\rm flat}$, and let
\begin{align*}
    I=\lceil u_0\tau\rceil\ge2,\qquad J=\lfloor v_0\tau\rfloor,\qquad L=J-I+1\ge\frac{\eta_0\tau}2.
\end{align*}
Choose the values at the fixed point as
\begin{align}
    a_i=a_-\ \ (i<I),\qquad
    a_i=a\rho\zeta(\gamma)i^\gamma\ \ (I\le i\le J),\qquad
    a_i=a_+\ \ (i>J),
    \label{eq:flat-values}
\end{align}
so that $c_ia_i=a\rho$ inside the band; all values lie in $[a_-,a_+]$.
With $H_i=\sum_{j\le i}c_ja_j$ and
\begin{align*}
    B_i=q^{d(i-1)}e^{H_{i-1}/\rho},\qquad
    \frac{B_{i+1}}{B_i}=\exp\qty(-a+\frac{c_ia_i}\rho),
\end{align*}
$B_i$ is nondecreasing up to $I$, constant $=:B$ on $I,\ldots,J+1$, and nonincreasing afterwards.
On the shell $E_i=\mathsf T^{i-1}(\Omega\setminus Q_\flat)$, writing $y=\mathsf T^{i-1}(u)$,
\begin{align}
    f(y)=H_{i-1}+c_ig_i(u)+e_i(u),\qquad
    \abs{e_i(u)}\le\lambda\sum_{j<i}c_jq^{\beta(i-j)}\le\epsilon_{\rm anc}\,c_{i+1},\qquad
    \epsilon_{\rm anc}=\frac{\cmin}4,
    \label{eq:flat-ancestor}
\end{align}
by the geometric series in the proof of \cref{lem:geometry}.
With $D_{\rm b}=(1+\epsilon_{\rm anc})a/a_-$, the shell masses $w_i=\int_{E_i}e^{f/\rho}\dd{\mu}$ satisfy
\begin{align*}
    (1-q^d)e^{-D_{\rm b}}B\le w_i\le(1-q^d)e^{D_{\rm b}}B\ \ \text{inside the band},\qquad
    w_i\le e^{2a}B_{i+2}\ \ \text{in general},
\end{align*}
and since $B_{i+1}/B_i\le e^{-a/2}$ for $i\ge\lceil2v_0\tau\rceil$, the normalizing constant satisfies
\begin{align}
    \frac12\eta_0(1-q^d)e^{-D_{\rm b}}\tau B\le Z=\sum_{i\ge1}w_i\le C_Z\tau B.
    \label{eq:flat-Z}
\end{align}
Thus the Gibbs policy $P^\ast_\rho$ spreads a mass of order $1/\tau$ over each of the $L\asymp\tau$ shells of the band.

\paragraph{Perturbations.}
Let $r_U=\frac1{32}(3/4)^{1/\alpha}$ and $U=\frac12\bfone+[-r_U,r_U]^d$, a cube disjoint from $Q_\flat$ and from the supports of the positive and negative bumps.
Partition $U$ into $N^d$ cells of side $h=2r_U/N$, and let $\phi_j$ be the copy on the $j$-th cell of the balanced bump on the unit cell,
\begin{align*}
    \phi=b_{v_+}-b_{v_-},\qquad
    b_v(u)=\psi_v(u),\qquad
    v_\pm=(1/4\text{ or }3/4,\,1/2,\ldots,1/2);
\end{align*}
thus $\abs{\phi_j}\le1$, $\int\phi_j=0$, $\int\phi_j^2=h^dv_\phi$ with
\begin{align*}
    v_\phi=2\int_{\bR^d}\psi_0(u)^2\dd{u}
    =\begin{cases}
        \displaystyle\frac{4\alpha^2}{4^d(d+\alpha)(d+2\alpha)},&0<\alpha\le1,\\[1mm]
        \displaystyle\frac{2r_c^d\pi^{d/2}\Gamma(2m_\alpha+1)}{\Gamma(2m_\alpha+1+d/2)},&\alpha>1,
    \end{cases}
\end{align*}
which is strictly positive. The zero extensions obey
$|\phi_j^0|_\alpha\le2B_{\alpha,d}h^{-\alpha}$ by the scaling and gluing bounds above.
Let
\begin{align*}
    \kappa_b=\min\qty{1,\ \frac a{4a_+},\ \frac{a(2r_U)^\alpha r_c^{-\alpha}}{4\cdot8^\alpha a_+}},\qquad
    \delta_{\rm p}=\kappa_b\rho N^{-\alpha}.
\end{align*}
For $I\le i\le J$ and independent signs $\omega_{i,j}\in\{-1,1\}$, replace $g_i$ by
\begin{align}
    g_i+\frac{\delta_{\rm p}}{c_i}\sum_{j=1}^{N^d}\omega_{i,j}\phi_j.
    \label{eq:flat-bumps}
\end{align}
Since $c_ia_i=a\rho$ on the band, the amplitude $A_i=\delta_{\rm p}/c_i$ satisfies
$A_i\le1/4$ and
$A_ih^{-\alpha}\le\kappa_ba_+/(a(2r_U)^\alpha)\le1/4$.
The original and added bumps have disjoint supports, so the scaling and gluing bounds give $|g_i^0|_\alpha\le2B_{\alpha,d}\le\lambda$ uniformly in $N$ and the signs.
The perturbation has sup norm at most $1/4<a_-$ and vanishes outside $U$; it therefore keeps the zero mean, range, unique maximizer $v_i$, maximum $1$, value $a_i$ at $z$, placement cube $Q_\flat$ and the bound $g_i\le a_i$ off $Q_\flat$.
Hence the perturbed instance belongs to $\scF_{M,\theta}$ for every sign pattern.
The ancestor perturbations vanish on $Q_\flat$, so \eqref{eq:flat-ancestor} and \eqref{eq:flat-Z} hold for all signs, and each physical cell $\mathsf T^{i-1}(\text{cell}_j)$ in the band has Gibbs mass
\begin{align}
    b\ge c_{\rm mass}\frac{h^d}{\tau},\qquad c_{\rm mass}=\frac{e^{-D_{\rm b}}}{C_Z}.
    \label{eq:flat-mass}
\end{align}

\subsection{Large Budget: Decoding and Information Bound}
\label{app:lower-decode}
For laws $P,P'$ with densities $p,p'$ relative to a common law $R$, use the squared Hellinger distance
\begin{align*}
    H^2(p,p'):=\int\qty(\sqrt p-\sqrt{p'})^2\dd{R}
    =2\qty(1-\operatorname{Aff}(P,P')).
\end{align*}

\paragraph{Separation of the two hypotheses on a cell.}
Let $w$ be the conditional density of $P^\ast_\rho$ on a cell of the band under the unperturbed family; it is within a factor $e^{\pm2D_{\rm b}}$ of the uniform density on the cell and does not depend on the signs (the ancestor perturbations vanish on $Q_\flat$ and the deeper levels vanish on the shell).
With $t_\delta=\delta_{\rm p}/\rho\le1$, the conditional densities under the two signs of bit $(i,j)$ are
\begin{align*}
    q_\pm=\frac{we^{\pm t_\delta\phi_j}}{\bE_we^{\pm t_\delta\phi_j}},
\end{align*}
and with $V,V'\sim w$ independent,
\begin{align*}
    A_t:=\qty(\bE_we^{t_\delta\phi_j})\qty(\bE_we^{-t_\delta\phi_j})
    =\bE\cosh\qty(t_\delta(\phi_j(V)-\phi_j(V')))\in\qty[1+t_\delta^2\Var_w(\phi_j),\,e^2].
\end{align*}
Hence the squared Hellinger distance satisfies
\begin{align}
    H^2(q_+,q_-)=2\qty(1-A_t^{-1/2})\ge e^{-3}t_\delta^2\Var_w(\phi_j)\ge c_{\rm sep}t_\delta^2,\qquad
    c_{\rm sep}=e^{-3-2D_{\rm b}}v_\phi.
    \label{eq:flat-sep}
\end{align}

\paragraph{Decoding.}
Given any output policy with density $\hat p$, decode the sign of bit $(i,j)$ by comparing $\ev{\sqrt{\hat p},\sqrt{q_+}}$ and $\ev{\sqrt{\hat p},\sqrt{q_-}}$ on the physical cell.
If the decoded sign is wrong, the squared $L^2$ distance between $\sqrt{\hat p}$ and $\sqrt{b\,q_{\rm true}}$ on the cell is at least $\frac b4H^2(q_+,q_-)$ (the distance to the half-space of the wrong decision).
Using $\KL{P}{Q'}\ge\int(\sqrt p-\sqrt{q'})^2$, $\Reg_\rho(P)=\rho\,\KL{P}{P^\ast_\rho}$, \eqref{eq:flat-mass} and \eqref{eq:flat-sep}, and summing over cells,
\begin{align}
    \bE\,\Reg_\rho(\Phat)\ge\frac{c_{\rm mass}c_{\rm sep}\delta_{\rm p}^2h^d}{4\rho\tau}\sum_{i=I}^J\sum_{j=1}^{N^d}\Pr\qty(\text{bit }(i,j)\text{ decoded wrongly}),
    \label{eq:flat-decode}
\end{align}
where the probability is over the prior and the transcript.
Neither the unknown normalizing constant nor the mass that $\Phat$ puts outside the cells affects this bound; if $\Phat$ is singular, the regret is infinite.

\paragraph{Information.}
Let $\Phi_{i,j}=\phi_j\circ\mathsf T^{-(i-1)}$ on $\mathsf T^{i-1}(U)$ and $0$ elsewhere.
Flipping the sign $(i,j)$ changes the reward by $2\delta_{\rm p}\Phi_{i,j}$, and since the $\Phi_{i,j}$ have disjoint supports and $\abs{\Phi_{i,j}}\le1$, Jensen's inequality gives, for any cube $D$,
\begin{align}
    \sum_{i,j}\ev{\Phi_{i,j}}_D^2\le\ev{\sum_{i,j}\Phi_{i,j}^2}_D\le1,
    \label{eq:sensitivity}
\end{align}
where $\ev{\cdot}_D$ is the average over $D$.
By \cref{lem:kl-adaptive}, for every sign pattern $\omega$,
\begin{align}
    \sum_{i,j}\KL{\bar P_\omega}{\bar P_{\omega^{(i,j)}}}\le\frac{2k\delta_{\rm p}^2}{\sigma^2}.
    \label{eq:flat-info}
\end{align}
By the two-point test and Pinsker's inequality \citep[proof of Theorem 2.2(i) and Lemma 2.5(i)]{tsybakov2008introduction} averaged over all $LN^d$ bits and all sign patterns, and Jensen's inequality on the square root,
\begin{align}
    2^{-LN^d}\sum_\omega\sum_{i,j}\Pr_\omega\qty(\text{bit }(i,j)\text{ wrong})
    \ge\frac{LN^d}2\qty[1-\sqrt{\frac{k\delta_{\rm p}^2}{LN^d\sigma^2}}].
    \label{eq:flat-wrong}
\end{align}

\paragraph{Conclusion for small $\rho$.}
Choosing
\begin{align*}
    N=\max\qty{1,\left\lceil\qty(\frac{4k\kappa_b^2\rho^2}{L\sigma^2})^{1/(2\alpha+d)}\right\rceil}
\end{align*}
makes the bracket in \eqref{eq:flat-wrong} at least $1/2$, so the expected number of wrong bits is at least $LN^d/4$.
Substituting into \eqref{eq:flat-decode} with $N^dh^d=\abs{U}$ and $L/\tau\ge\eta_0/2$,
\begin{align*}
    B_{\pi_{k,\rho}}(k,\rho)\ge C_0'\kappa_b^2\rho N^{-2\alpha},\qquad C_0'=\frac{c_{\rm mass}c_{\rm sep}\eta_0\abs{U}}{32},
\end{align*}
and since $N^{2\alpha+d}\lesssim1+k\rho^{2+1/\gamma}$ (recall $L\asymp\tau=\rho^{-1/\gamma}$), we obtain \eqref{eq:one-prompt-lower} for $0<\rho\le\rho_{\rm flat}$ with a constant $C_{\rm small}$ accounting for the rounding.

\paragraph{Large $\rho$.}
For $\rho_{\rm flat}\le\rho\le1$, use the base family with all $a_i=a_-$ and perturb only level $2$, with amplitude $b_0N^{-\alpha}$, where
\begin{align*}
    b_0=\min\qty{\frac14,\ \frac{r_c^{-\alpha}(2r_U)^\alpha}{4\cdot8^\alpha},\ \frac{\rho_{\rm flat}}{c_2}},\qquad
    \delta_{{\rm p},0}=c_2b_0.
\end{align*}
The inclusion in $\scF_{M,\theta}$ is verified as before, since $b_0N^{-\alpha}h^{-\alpha}=b_0/(2r_U)^\alpha\le1/4$; $t_\delta=\delta_{{\rm p},0}N^{-\alpha}/\rho\le1$, the true mass of a cell is at least $q^de^{-2/\rho_{\rm flat}}h^d$ since $\abs{f}\le1$, and the conditional density is within a factor $e^{\pm2/\rho_{\rm flat}}$ of uniform.
The same decoding and information bounds with $N=\max\{1,\lceil(4k\delta_{{\rm p},0}^2/\sigma^2)^{1/(2\alpha+d)}\rceil\}$ give
\begin{align*}
    B_{\pi_{k,\rho}}(k,\rho)\ge C_{\rm large}\min\{1,k^{-p}\},
\end{align*}
and since $\rho\min\{1,(k\rho^{2+1/\gamma})^{-p}\}\le\max\{1,\rho_{\rm flat}^{1-(2+1/\gamma)p}\}\min\{1,k^{-p}\}$ on this range, \eqref{eq:one-prompt-lower} holds with
\begin{align*}
    c=\min\qty{C_{\rm small},\ \frac{C_{\rm large}}{\max(1,\rho_{\rm flat}^{1-(2+1/\gamma)p})}};
\end{align*}
$\rho_{\rm flat}$ depends only on the structural parameters.

\subsection{Large Budget: Conclusion}
\label{app:lower-conclusion}
Combining \eqref{eq:one-prompt-lower} with \cref{lem:bayes-transfer} and $1+k\le3(1+t)$ yields \eqref{eq:lower}.
The second statement of the large-budget bound of \cref{thm:lower} follows since $n\ge M\rho^{-(2+1/\gamma)}$ implies $(n+M)\rho^{2+1/\gamma}/M\le2n\rho^{2+1/\gamma}/M$. \qed

\section{Proofs for \cref{sec:fixed}}
\label{app:fixed}

Recall the protocol of $\scAunif$: at round $r'$, $X_{r'}\sim\nu$ and $Y_{r'}\sim\mu$ are generated independently of the past; the learner chooses a depth $H_{r'}\in\bN_0$ from the history, $(X_{r'},Y_{r'})$ and internal randomness, queries $D_{r'}=D_{H_{r'}}(Y_{r'})$ and observes $Z_{r'}=Q(X_{r'},D_{r'})+\xi_{r'}$; after $n$ rounds it outputs an arbitrary policy.
We write
\begin{align*}
    t=\frac nM,\qquad g=\gamma-1,\qquad s=2+\frac1\gamma,\qquad L(t)=\log(e+t).
\end{align*}

The two bounds of \cref{thm:fixed} are proved in \cref{app:fixed-log} and \cref{app:fixed-exp}, and \cref{app:fixed-grid,app:fixed-rho} show that neither is an artefact of the method of proof: the first bound is attained by a fixed query depth, and for a fixed $\rho$ the exponent of $n$ is not worsened by fixed sampling.
The two bounds use different families, but for the same reason: a query made from outside the cell that carries the unknown information cannot see it, whatever its depth, and under uniform sampling that cell is reached only exponentially rarely.
The logarithmic bound uses the two-choice family of \cref{app:two-choice}, which lies in $\scF_{M,\theta}$ whenever $\lambda\ge\Lambda_{\alpha,d}$, and the following form of that property; the exponential bound reuses the flat-band family of \cref{app:lower-family} and applies the same volume factor to the information bound of \cref{app:lower-decode}.

\begin{lemma}[Cancellation]\label{lem:cancellation}
    Fix $k\ge2$, set all bits to $0$ except $\omega_{k,m}=b_m$, and let $f_0,f_1$ be the two rewards of prompt $m$ (with $b_m=0,1$).
    Then, with $C_{k-1}=C_{k-1}(m)$ common to both and $\mu(C_{k-1})=v_k:=2^{-d\lstar(k-1)}$,
    \begin{align}
        f_0=f_1\ \text{ on }\Omega\setminus C_{k-1},\qquad
        \int_{C_{k-1}}(f_0-f_1)\dd{\mu}=0,\qquad
        \int_D(f_0-f_1)\dd{\mu}=0\ \text{ if }D\not\subset C_{k-1}.
        \label{eq:cancellation}
    \end{align}
    In particular, the two oracle cube averages agree whenever $D\not\subset C_{k-1}$.
\end{lemma}
\begin{proof}
    The first $k-1$ levels are common, all deeper levels are supported in $C_{k-1}$, and each has zero integral by \assmp{A4}.
    A dyadic cube that is not contained in $C_{k-1}$ is either disjoint from $C_{k-1}$ or contains it; in both cases the cube averages coincide.
\end{proof}

\subsection{Proof of the Logarithmic Bound of \cref{thm:fixed}}
\label{app:fixed-log}
Fix $k\ge2$ and use the prior of \cref{lem:cancellation} with independent fair bits $b_m$.
Let
\begin{align*}
    E_m=\bigcap_{r'\le n}\qty{X_{r'}\ne m\ \text{ or }\ Y_{r'}\notin C_{k-1}};
\end{align*}
by the exogenous laws,
\begin{align}
    \Pr(E_m)=\qty(1-\frac{v_k}M)^n
    \label{eq:no-hit}
\end{align}
exactly.
Couple the two experiments $b_m=0,1$ by using the same other bits, internal randomness, $(X_{r'},Y_{r'})$ and noise.
On $E_m$, by induction over rounds, the histories coincide: the same depth is chosen from the same history, and the observation is the same by \cref{lem:cancellation} when $X_{r'}=m$ (the queried cube contains $Y_{r'}\notin C_{k-1}$) and trivially otherwise.
Hence, conditionally on $E_m$, the bit $b_m$ is fair and independent of the transcript and of the output policy; this holds also for $\sigma=0$.

Let $G_b=\max f_b$.
By \eqref{eq:geo-lower}, $G_b-f_b(y)\ge\kappa T_k$ for $y\notin C^b_k$, and $C^0_k,C^1_k$ are disjoint, so for every $y$ at least one of the two losses is at least $\kappa T_k$.
Integrating the common output policy and averaging the fair bit,
\begin{align*}
    \bE\qty[G_{b_m}-\bE_{\Phat(\cdot\mid m)}f_{b_m}\relmiddle|E_m]\ge\frac{\kappa T_k}2.
\end{align*}
Since the unregularized regret is nonnegative on $E_m^c$, averaging over prompts gives, for all $n,M\ge1$ and $k\ge2$,
\begin{align}
    \RC(n,M,0)\ge\frac{\kappa T_k}2\qty(1-\frac{2^{-d\lstar(k-1)}}M)^n.
    \label{eq:fixed-finite-0}
\end{align}
For $\rho>0$, \cref{lem:reg-transfer} gives $\Reg_\rho(P)\ge\Reg_0(P)-U\rho^{g/\gamma}$ for every policy, with the same constant $U$ (singular policies have infinite regularized regret).
Therefore, for $0\le\rho\le1$,
\begin{align}
    \RC(n,M,\rho)\ge\qty[\frac{\kappa T_k}2\qty(1-\frac{2^{-d\lstar(k-1)}}M)^n-U\rho^{g/\gamma}]_+.
    \label{eq:fixed-finite}
\end{align}
Choose
\begin{align*}
    k=2+\left\lceil\frac{\log(2(1+t))}{d\lstar\log2}\right\rceil,
\end{align*}
so that $tv_k\le1/2$ and, by Bernoulli's inequality, $(1-v_k/M)^n\ge1-nv_k/M\ge1/2$.
Since $k\le CL(t)$ and $T_k\ge c_Tk^{-g}$, we get $\RC\ge[c_{\rm L}L(t)^{-g}-U\rho^{g/\gamma}]_+$, and the second claim of the first case of \eqref{eq:fixed} follows with $a$ such that $Ua^{g/\gamma}\le c_{\rm L}/2$. \qed

\subsection{A Uniform-Sampling Method Attaining the Logarithmic Rate}
\label{app:fixed-grid}
\begin{proposition}\label{prop:uniform-grid}
    Let $H=\lfloor\log_2(1+t)/(2d)\rfloor$ and $B_H=2^{dH}$.
    The algorithm in $\scAunif$ that queries $D_H(Y_{r'})$ at every round, averages the observations of each $(m,D)$ (clipped to $[-1,1]$, $0$ if unobserved) and outputs the uniform distribution on the cube with the largest average satisfies, for all $n,M\ge1$, $\rho\in[0,1]$ and $\sigma\ge0$,
    \begin{align}
        \RC(n,M,\rho)\le C\qty[L(t)^{-g}+\rho L(t)].
        \label{eq:grid-upper}
    \end{align}
    Consequently, $\RC(n,M,\rho)=\Theta(L(t)^{-g})$ whenever $\lambda\ge\Lambda_{\alpha,d}$ and $\rho\le aL(t)^{-\gamma}$.
\end{proposition}
\begin{proof}
    For each $(m,D)$ the expected number of observations is
    \begin{align*}
        \lambda_H=\frac t{B_H}\ge\frac t{\sqrt{1+t}}\ge\sqrt{\frac{2t}3}\qquad(t\ge2).
    \end{align*}
    By the exponential supermartingale $\exp(\eta'S-\sigma^2\eta'^2N/2)$ of the noise sums and the lower binomial tail of the counts, for $0<x\le2$,
    \begin{align*}
        \Pr\qty(\abs{\Qhat(m,D)-Q(m,D)}>x)\le e^{-\lambda_H/8}+2\exp\qty(-\frac{\lambda_Hx^2}{4\sigma^2}).
    \end{align*}
    A union bound over the $B_H$ cubes of one prompt and integration of the tail give, for $t\ge2$,
    \begin{align*}
        \bE\max_D\abs{\Qhat(m,D)-Q(m,D)}\le2B_He^{-\lambda_H/8}+C\sigma\sqrt{\frac{\log(eB_H)}{\lambda_H}}\le C_\sigma t^{-1/4}\sqrt{L(t)}.
    \end{align*}
    Let $D^\ast_H$ be the cube of depth $H$ containing the maximizer and $K'=\lfloor H/\lbar\rfloor$; then $D^\ast_H\subset C_{K'}$ and $G_m-Q(m,D^\ast_H)\le C(H+1)^{-g}$ by \eqref{eq:geo}.
    The empirically best cube $\Dhat_m$ therefore satisfies
    \begin{align*}
        G_m-Q(m,\Dhat_m)\le C(H+1)^{-g}+2\max_D\abs{\Qhat(m,D)-Q(m,D)},
    \end{align*}
    and the KL divergence of the output is exactly $\log B_H=dH\log2$.
    Averaging over prompts (no simultaneous success is needed) and treating small $t$ by the uniform output gives \eqref{eq:grid-upper}; the matching lower bound is \cref{thm:fixed}.
\end{proof}

\subsection{Proof of the Exponential Bound of \cref{thm:fixed}}
\label{app:fixed-exp}
Fix $\sigma>0$. We reuse the flat-band family of \cref{app:lower-family}, now with independent signs $\omega_{i,m,j}$ for each prompt $m$ and with the common placement cube $Q_\flat$ at all levels (which may be revealed to the learner).
The decoding bound \eqref{eq:flat-decode} holds for each prompt, with an additional factor $1/M$ from the prompt average:
\begin{align}
    \bE\,\Reg_\rho(\Phat)\ge c\frac{\delta_{\rm p}^2h^d}{\rho\tau M}\sum_{i=I}^J\sum_{m=1}^M\sum_{j=1}^{N^d}\Pr\qty(\text{sign }(i,m,j)\text{ wrong}).
    \label{eq:exp-decode}
\end{align}

\paragraph{Information under uniform sampling.}
All perturbations are supported in the cube $C_{I-1}=\mathsf T^{I-1}(\Omega)$, of volume $v_I=q^{d(I-1)}$, and have zero mean.
If a uniform response $Y_{r'}\notin C_{I-1}$, then any dyadic cube $D\ni Y_{r'}$ is disjoint from $C_{I-1}$ or contains it, and in both cases $\ev{\Phi_{i,j}}_D=0$ for all $(i,j)$.
If $Y_{r'}\in C_{I-1}$, then $\sum_{i,j}\ev{\Phi_{i,j}}_D^2\le1$ by \eqref{eq:sensitivity}.
Hence, conditionally on the history and $X_{r'}$, the expected squared sensitivity of the $r'$-th observation to the signs is at most $v_I$, whatever depth is chosen from $(X_{r'},Y_{r'})$ and the history.
By \cref{lem:kl-adaptive} and the chain rule, for every sign pattern,
\begin{align}
    \sum_{i,m,j}\KL{\bar P_\omega}{\bar P_{\omega^{(i,m,j)}}}\le\frac{2n\delta_{\rm p}^2v_I}{\sigma^2},
    \label{eq:exp-info}
\end{align}
where only the signs of the arriving prompt contribute at each round, so no extra factor $M$ appears.

\paragraph{Conclusion.}
With $B_{\rm bits}=LMN^d$ signs, the two-point test and Pinsker's inequality give an expected number of wrong signs at least
\begin{align*}
    \frac{B_{\rm bits}}2\qty[1-\sqrt{\frac{n\delta_{\rm p}^2v_I}{B_{\rm bits}\sigma^2}}],
\end{align*}
and the choice
\begin{align*}
    N=\max\qty{1,\left\lceil\qty(\frac{4n\kappa_b^2\rho^2v_I}{LM\sigma^2})^{1/(2\alpha+d)}\right\rceil}
\end{align*}
makes it at least $B_{\rm bits}/4$.
Substituting into \eqref{eq:exp-decode} with $N^dh^d=\abs{U}$ and $L\asymp\tau$,
\begin{align*}
    \RC(n,M,\rho)\ge c\rho N^{-2\alpha}\ge c'\rho\min\qty{1,\qty[\frac1{t\rho^{2+1/\gamma}v_I}]^p}.
\end{align*}
For $\rho\le\rho_0$ small enough, $I-1\ge u_0\tau/2$, so that
\begin{align*}
    v_I\le\exp\qty(-c_0\rho^{-1/\gamma}),\qquad c_0=\frac{au_0}2,
\end{align*}
which gives the second case of \eqref{eq:fixed}; the second statement follows by selecting the second term of the minimum. \qed

\subsection{Fixed Regularization: the Exponent in $n$ Is Not Worsened}
\label{app:fixed-rho}
\begin{proposition}\label{prop:fixed-rho}
    For all $n,M\ge1$, $\rho\in(0,1]$ and $\sigma\ge0$, there is an algorithm in $\scAunif$ with
    \begin{align}
        \RC(n,M,\rho)\le\min\qty{2,\ C\frac{e^{2/\rho}}{\rho}\qty[\frac{L(t)^2}{1+t}]^{p}}.
        \label{eq:fixed-rho-upper}
    \end{align}
    Consequently, for $\sigma>0$, $\lambda\ge\Lambda_{\alpha,d}$ and $t\ge C_0\rho^{-s}$,
    \begin{align*}
        c\rho^{1-sp}(1+t)^{-p}\le\RC(n,M,\rho)\le C\frac{e^{2/\rho}}\rho\qty[\frac{L(t)^2}{1+t}]^p,
    \end{align*}
    so that, for each fixed $\rho>0$, the optimal exponent of $t$ is $p$.
\end{proposition}
\begin{proof}
    \emph{Finite approximation family.}
    For $0<\varepsilon\le1/4$ let $K'=\lceil\frac2{d\lstar}\log_2(1/\varepsilon)\rceil$, so that
    \begin{align*}
        \norm{f-f^{[K']}}_{L^2(\mu)}\le\mu(C_{K'})^{1/2}\le\varepsilon.
    \end{align*}
    On each cell of a dyadic mesh of side $h_T\asymp\varepsilon^{1/\alpha}$, use the degree-$k$ Taylor polynomial of a primitive at the lower-left corner, where $k=\lceil\alpha\rceil-1$.
    Taylor's formula and \assmp{A1} give remainder at most $C_{\alpha,d}\lambda h_T^\alpha$, also for integer $\alpha$; choose the mesh constant so this is at most $\varepsilon/4$.
    Quantize the finitely many coefficients on each cell to multiples of $c_{\alpha,d}\varepsilon$; they are bounded by $1+\lambda$, and quantization adds at most $\varepsilon/4$ on that cell.
    Clip the resulting piecewise polynomial to $[-1,1]$.
    This gives a finite sup-norm $\varepsilon/2$ net with $O(\varepsilon^{-d/\alpha})$ cells and $O_{\alpha,d}(1)$ coefficients per cell.
    Next let $J'$ be the smallest nonnegative integer with $\lambda2^{-\beta J'}\le\varepsilon/2$, and replace every candidate by its averages on the depth-$J'$ grid.
    Averaging contracts the approximation error, while the true primitive differs from its own grid averages by at most $\lambda2^{-\beta J'}$.
    Thus these piecewise constant candidates have sup error at most $\varepsilon$ and the same cardinality, with
    \begin{align*}
        \log N_{\rm prim}\le C\varepsilon^{-d/\alpha}\log(1/\varepsilon)
    \end{align*}
    candidates.
    Enumerate also all placements of the first $K'$ levels ($O(K')$ bits).
    Let $\scG_\varepsilon$ be the finite set of scores $\sum_{i\le K'}c_i\tilde h_i$ obtained from these candidates; then
    \begin{align}
        \log\abs{\scG_\varepsilon}\le C\varepsilon^{-d/\alpha}\log^2(1/\varepsilon),
        \qquad
        \inf_{v\in\scG_\varepsilon}\norm{f-v}_{L^2(\mu)}\le2\varepsilon.
        \label{eq:passive-cover}
    \end{align}

    \emph{Least squares on cube averages.}
    Fix the query depth $H=\lbar K'+J'$; every $v\in\scG_\varepsilon$ is constant on the depth-$H$ grid, and the observation is $Z_{r'}=f_H(D_H(Y_{r'}))+\xi_{r'}$, where $f_H$ is the projection of $f$ onto the grid.
    For each prompt use its first $q'=\lfloor t/2\rfloor$ observations and let $\hat v$ minimize the empirical squared loss over $\scG_\varepsilon$.
    The ERM basic inequality, the conditional sub-Gaussian bound for the noise with a union bound over $\scG_\varepsilon$, Bernstein's inequality for the i.i.d.\ uniform cells and the orthogonality $\norm{f-f_H}^2+\norm{f_H-v_0}^2=\norm{f-v_0}^2$ give
    \begin{align}
        \bE\norm{\hat v-f}^2_{L^2(\mu)}\le C\qty[\varepsilon^2+\frac{(1+\sigma^2)\varepsilon^{-d/\alpha}\log^2(1/\varepsilon)}{q'}].
        \label{eq:passive-erm}
    \end{align}
    Choosing $\varepsilon\asymp[\log^2(e+q')/(1+q')]^{\alpha/(2\alpha+d)}$ gives
    \begin{align*}
        \bE\norm{\hat v-f}^2_{L^2(\mu)}\le C\qty[\frac{\log^2(e+q')}{1+q'}]^p.
    \end{align*}

    \emph{Gibbs output.}
    Output $P_{\hat v}=\Gibbs{\hat v}{\rho}$.
    Since $f,\hat v\in[-1,1]$, the Gibbs densities along the segment between them are bounded by $e^{2/\rho}$, and \eqref{eq:bregman} gives
    \begin{align*}
        \Reg_\rho(P_{\hat v})\le\frac{e^{2/\rho}}{2\rho}\norm{f-\hat v}^2_{L^2(\mu)}.
    \end{align*}

    \emph{Prompts.}
    If $N_m<q'$, output the uniform distribution (regret at most $2$); $\Pr(N_m<q')\le e^{-t/8}$ for $t\ge4$, and conditionally on the prompt sequence the first $q'$ inputs of a prompt are i.i.d.\ uniform.
    Averaging over prompts gives \eqref{eq:fixed-rho-upper}.
    The lower bound in the second statement is the large-budget bound of \cref{thm:lower}, which applies to $\scAunif$ as a subclass of all algorithms.
\end{proof}


\begin{thebibliography}{38}
\providecommand{\natexlab}[1]{#1}
\providecommand{\url}[1]{\texttt{#1}}
\expandafter\ifx\csname urlstyle\endcsname\relax
  \providecommand{\doi}[1]{doi: #1}\else
  \providecommand{\doi}{doi: \begingroup \urlstyle{rm}\Url}\fi

\bibitem[Bengio et~al.(2009)Bengio, Louradour, Collobert, and
  Weston]{bengio2009curriculum}
Yoshua Bengio, J{\'e}r{\^o}me Louradour, Ronan Collobert, and Jason Weston.
\newblock Curriculum learning.
\newblock In \emph{International Conference on Machine Learning}, pp.\  41--48,
  2009.

\bibitem[Boucheron et~al.(2013)Boucheron, Lugosi, and
  Massart]{boucheron2013concentration}
St{\'e}phane Boucheron, G{\'a}bor Lugosi, and Pascal Massart.
\newblock \emph{Concentration Inequalities: A Nonasymptotic Theory of
  Independence}.
\newblock Oxford University Press, 2013.

\bibitem[Bu et~al.(2026)Bu, Huang, Han, Nitanda, Wong, Zhang, and
  Suzuki]{bu2026provable}
Dake Bu, Wei Huang, Andi Han, Atsushi Nitanda, Hau-San Wong, Qingfu Zhang, and
  Taiji Suzuki.
\newblock {Provable Sample Efficiency of Curriculum
  Post-Training for Transformer Reasoning}.
\newblock In \emph{International Conference on Machine Learning}, 2026.

\bibitem[Bubeck et~al.(2011)Bubeck, Munos, Stoltz, and
  Szepesv{\'a}ri]{bubeck2011x}
S{\'e}bastien Bubeck, R{\'e}mi Munos, Gilles Stoltz, and Csaba Szepesv{\'a}ri.
\newblock {X}-armed bandits.
\newblock \emph{Journal of Machine Learning Research}, 12:\penalty0 1655--1695,
  2011.

\bibitem[{DeepSeek-AI}(2026)]{deepseek2026v4}
{DeepSeek-AI}.
\newblock {DeepSeek-V4}: Towards highly efficient million-token context
  intelligence.
\newblock \emph{arXiv preprint arXiv:2606.19348}, 2026.

\bibitem[Fan et~al.(2020)Fan, Wang, Xie, and Yang]{fan2020theoretical}
Jianqing Fan, Zhaoran Wang, Yuchen Xie, and Zhuoran Yang.
\newblock A theoretical analysis of deep {Q}-learning.
\newblock In \emph{Proceedings of the 2nd Conference on Learning for Dynamics
  and Control}, volume 120 of \emph{Proceedings of Machine Learning Research},
  pp.\  486--489. PMLR, 2020.

\bibitem[Foster et~al.(2025)Foster, Mhammedi, and Rohatgi]{foster2025good}
Dylan~J Foster, Zakaria Mhammedi, and Dhruv Rohatgi.
\newblock Is a good foundation necessary for efficient reinforcement learning?
  the computational role of the base model in exploration.
\newblock In \emph{The Thirty Eighth Annual Conference on Learning Theory},
  pp.\  2026--2142. PMLR, 2025.

\bibitem[Guo et~al.(2025)Guo, Yang, Zhang, Song, Zhang, Xu, Zhu, Ma, Wang, Bi,
  et~al.]{guo2025deepseek}
Daya Guo, Dejian Yang, Haowei Zhang, Junxiao Song, Ruoyu Zhang, Runxin Xu,
  Qihao Zhu, Shirong Ma, Peiyi Wang, Xiao Bi, et~al.
\newblock Deep{S}eek-{R}1: Incentivizing reasoning capability in {LLM}s via
  reinforcement learning.
\newblock \emph{arXiv preprint arXiv:2501.12948v1}, 2025.
\newblock URL \url{https://arxiv.org/html/2501.12948v1}.

\bibitem[Hayakawa \& Suzuki(2020)Hayakawa and Suzuki]{hayakawa2020minimax}
Satoshi Hayakawa and Taiji Suzuki.
\newblock On the minimax optimality and superiority of deep neural network
  learning over sparse parameter spaces.
\newblock \emph{Neural Networks}, 123:\penalty0 343--361, 2020.

\bibitem[Kim et~al.(2025)Kim, Wu, Lee, and Suzuki]{kim2025metastable}
Juno Kim, Denny Wu, Jason~D Lee, and Taiji Suzuki.
\newblock Metastable dynamics of chain-of-thought reasoning: Provable benefits
  of search, {RL} and distillation.
\newblock In \emph{International Conference on Machine Learning}, 2025.

\bibitem[{Kimi Team}(2025)]{kimi2025k15}
{Kimi Team}.
\newblock Kimi k1.5: Scaling reinforcement learning with {LLM}s.
\newblock \emph{arXiv preprint arXiv:2501.12599}, 2025.

\bibitem[Kleinberg et~al.(2008)Kleinberg, Slivkins, and
  Upfal]{kleinberg2008multi}
Robert Kleinberg, Aleksandrs Slivkins, and Eli Upfal.
\newblock Multi-armed bandits in metric spaces.
\newblock In \emph{Proceedings of the 40th Annual ACM Symposium on Theory of
  Computing}, pp.\  681--690, 2008.

\bibitem[Li et~al.(2023)Li, Zhai, Ma, and Levine]{li2023understanding}
Qiyang Li, Yuexiang Zhai, Yi~Ma, and Sergey Levine.
\newblock Understanding the complexity gains of single-task {RL} with a
  curriculum.
\newblock In \emph{Proceedings of the 40th International Conference on Machine
  Learning}, volume 202 of \emph{Proceedings of Machine Learning Research}.
  PMLR, 2023.

\bibitem[Lightman et~al.(2024)Lightman, Kosaraju, Burda, Edwards, Baker, Lee,
  Leike, Schulman, Sutskever, and Cobbe]{lightman2023let}
Hunter Lightman, Vineet Kosaraju, Yuri Burda, Harrison Edwards, Bowen Baker,
  Teddy Lee, Jan Leike, John Schulman, Ilya Sutskever, and Karl Cobbe.
\newblock Let's verify step by step.
\newblock In \emph{International Conference on Learning Representations}, 2024.

\bibitem[Luo et~al.(2025)Luo, Tan, Wong, Shi, Tang, Roongta, Cai, Luo, Li,
  Popa, and Stoica]{luo2025deepscaler}
Michael Luo, Sijun Tan, Justin Wong, Xiaoxiang Shi, William~Y. Tang, Manan
  Roongta, Colin Cai, Jeffrey Luo, Li~Erran Li, Raluca~Ada Popa, and Ion
  Stoica.
\newblock {DeepScaleR}: Surpassing {O1}-preview with a 1.5{B} model by scaling
  {RL}.
\newblock Notion blog post, 2025.
\newblock URL
  \url{https://pretty-radio-b75.notion.site/DeepScaleR-Surpassing-O1-Preview-with-a-1-5B-Model-by-Scaling-RL-19681902c1468005bed8ca303013a4e2}.

\bibitem[Munos(2011)]{munos2011optimistic}
R{\'e}mi Munos.
\newblock Optimistic optimization of a deterministic function without the
  knowledge of its smoothness.
\newblock In \emph{Advances in Neural Information Processing Systems},
  volume~24, 2011.

\bibitem[Nishikawa \& Suzuki(2026)Nishikawa and Suzuki]{nishikawa2026inference}
Naoki Nishikawa and Taiji Suzuki.
\newblock Inference-time alignment with rewards in {B}esov spaces: Provable
  advantages of feature learning and multi-step policy updates.
\newblock In \emph{International Conference on Machine Learning}, 2026.

\bibitem[Ouyang et~al.(2022)Ouyang, Wu, Jiang, Almeida, Wainwright, Mishkin,
  Zhang, Agarwal, Slama, Ray, Schulman, Hilton, Kelton, Miller, Simens, Askell,
  Welinder, Christiano, Leike, and Lowe]{ouyang2022training}
Long Ouyang, Jeff Wu, Xu~Jiang, Diogo Almeida, Carroll~L. Wainwright, Pamela
  Mishkin, Chong Zhang, Sandhini Agarwal, Katarina Slama, Alex Ray, John
  Schulman, Jacob Hilton, Fraser Kelton, Luke Miller, Maddie Simens, Amanda
  Askell, Peter Welinder, Paul Christiano, Jan Leike, and Ryan Lowe.
\newblock Training language models to follow instructions with human feedback.
\newblock In \emph{Advances in Neural Information Processing Systems},
  volume~35, 2022.

\bibitem[Parashar et~al.(2026)Parashar, Gui, Li, Ling, Vemuri, Olson, Li,
  Zhang, Caverlee, Kalathil, and Ji]{parashar2026curriculum}
Shubham Parashar, Shurui Gui, Xiner Li, Hongyi Ling, Sushil Vemuri, Blake
  Olson, Eric Li, Yu~Zhang, James Caverlee, Dileep Kalathil, and Shuiwang Ji.
\newblock Curriculum reinforcement learning from easy to hard tasks improves
  {LLM} reasoning.
\newblock In \emph{International Conference on Learning Representations}, 2026.

\bibitem[Rajaraman et~al.(2026)Rajaraman, Huang, Dudik, Schapire, Foster, and
  Krishnamurthy]{rajaraman2026autocurriculum}
Nived Rajaraman, Audrey Huang, Miro Dudik, Rob Schapire, Dylan Foster, and
  Akshay Krishnamurthy.
\newblock Learning to reason with curriculum {I}: Provable benefits of
  autocurriculum.
\newblock In \emph{Proceedings of Thirty Ninth Conference on Learning Theory},
  volume 336 of \emph{Proceedings of Machine Learning Research}, pp.\
  5518--5555. PMLR, 2026.
\newblock URL \url{https://proceedings.mlr.press/v336/rajaraman26a.html}.

\bibitem[Schmidt-Hieber(2020)]{schmidt2020nonparametric}
Johannes Schmidt-Hieber.
\newblock Nonparametric regression using deep neural networks with {R}e{LU}
  activation function.
\newblock \emph{The Annals of Statistics}, 48\penalty0 (4):\penalty0
  1875--1897, 2020.

\bibitem[Schulman et~al.(2015)Schulman, Levine, Abbeel, Jordan, and
  Moritz]{schulman2015trust}
John Schulman, Sergey Levine, Pieter Abbeel, Michael Jordan, and Philipp
  Moritz.
\newblock Trust region policy optimization.
\newblock In \emph{International Conference on Machine Learning}, pp.\
  1889--1897, 2015.

\bibitem[Schulman et~al.(2017)Schulman, Wolski, Dhariwal, Radford, and
  Klimov]{schulman2017proximal}
John Schulman, Filip Wolski, Prafulla Dhariwal, Alec Radford, and Oleg Klimov.
\newblock Proximal policy optimization algorithms.
\newblock \emph{arXiv preprint arXiv:1707.06347}, 2017.

\bibitem[Setlur et~al.(2024)Setlur, Nagpal, Fisch, Geng, Eisenstein, Agarwal,
  Agarwal, Berant, and Kumar]{setlur2024rewarding}
Amrith Setlur, Chirag Nagpal, Adam Fisch, Xinyang Geng, Jacob Eisenstein,
  Rishabh Agarwal, Alekh Agarwal, Jonathan Berant, and Aviral Kumar.
\newblock Rewarding progress: Scaling automated process verifiers for {LLM}
  reasoning.
\newblock \emph{arXiv preprint arXiv:2410.08146}, 2024.

\bibitem[Shao et~al.(2024)Shao, Wang, Zhu, Xu, Song, Bi, Zhang, Zhang, Li, Wu,
  and Guo]{shao2024deepseekmath}
Zhihong Shao, Peiyi Wang, Qihao Zhu, Runxin Xu, Junxiao Song, Xiao Bi, Haowei
  Zhang, Mingchuan Zhang, Y.~K. Li, Y.~Wu, and Daya Guo.
\newblock {DeepSeekMath}: Pushing the limits of mathematical reasoning in open
  language models.
\newblock \emph{arXiv preprint arXiv:2402.03300}, 2024.

\bibitem[Shao et~al.(2025)Shao, Luo, Lu, Ren, Hu, Ye, Gou, Ma, and
  Zhang]{shao2025deepseekmathv2}
Zhihong Shao, Yuxiang Luo, Chengda Lu, Z.~Z. Ren, Jiewen Hu, Tian Ye, Zhibin
  Gou, Shirong Ma, and Xiaokang Zhang.
\newblock {DeepSeekMath-V2}: Towards self-verifiable mathematical reasoning.
\newblock \emph{arXiv preprint arXiv:2511.22570}, 2025.

\bibitem[Suzuki(2019)]{suzuki2019adaptivity}
Taiji Suzuki.
\newblock Adaptivity of deep {ReLU} network for learning in {B}esov and mixed
  smooth {B}esov spaces: Optimal rate and curse of dimensionality.
\newblock In \emph{International Conference on Learning Representations}, 2019.

\bibitem[Takakura \& Suzuki(2023)Takakura and
  Suzuki]{takakura2023approximation}
Shokichi Takakura and Taiji Suzuki.
\newblock Approximation and estimation ability of transformers for
  sequence-to-sequence functions with infinite dimensional input.
\newblock In \emph{Proceedings of the 40th International Conference on Machine
  Learning}, volume 202 of \emph{Proceedings of Machine Learning Research},
  pp.\  33416--33447. PMLR, 2023.

\bibitem[Tsybakov(2008)]{tsybakov2008introduction}
Alexandre~B Tsybakov.
\newblock \emph{Introduction to Nonparametric Estimation}.
\newblock Springer, 2008.

\bibitem[Valko et~al.(2013)Valko, Carpentier, and Munos]{valko2013stochastic}
Michal Valko, Alexandra Carpentier, and R{\'e}mi Munos.
\newblock Stochastic simultaneous optimistic optimization.
\newblock In \emph{Proceedings of the 30th International Conference on Machine
  Learning}, volume~28 of \emph{Proceedings of Machine Learning Research}, pp.\
   19--27. PMLR, 2013.

\bibitem[Wang et~al.(2024)Wang, Li, Shao, Xu, Dai, Li, Chen, Wu, and
  Sui]{wang2024mathshepherd}
Peiyi Wang, Lei Li, Zhihong Shao, Runxin Xu, Damai Dai, Yifei Li, Deli Chen,
  Yu~Wu, and Zhifang Sui.
\newblock {Math-Shepherd}: Verify and reinforce {LLMs} step-by-step without
  human annotations.
\newblock In \emph{Proceedings of the 62nd Annual Meeting of the Association
  for Computational Linguistics (Volume 1: Long Papers)}, pp.\  9426--9439,
  2024.

\bibitem[Wang et~al.(2026{\natexlab{a}})Wang, Shen, Sun, Feng, Teng, Dong, Hao,
  and Chen]{wang2026planning}
Siwei Wang, Yifei Shen, Haoran Sun, Shi Feng, Shang-Hua Teng, Li~Dong, Yaru
  Hao, and Wei Chen.
\newblock Benefits and pitfalls of reinforcement learning for language model
  planning: A theoretical perspective.
\newblock In \emph{International Conference on Learning Representations},
  2026{\natexlab{a}}.

\bibitem[Wang et~al.(2026{\natexlab{b}})Wang, Hong, and Wang]{wang2026greedy}
Zichen Wang, Haoyang Hong, and Huazheng Wang.
\newblock When greedy sampling explores: {KL}-regularized contextual bandits
  without eluder-dimension dependence.
\newblock \emph{arXiv preprint arXiv:2609.13564}, 2026{\natexlab{b}}.

\bibitem[Xiong et~al.(2024)Xiong, Dong, Ye, Wang, Zhong,
  {Ji}, {Jiang}, and
  Zhang]{xiong2024iterative}
Wei Xiong, Hanze Dong, Chenlu Ye, Ziqi Wang, Han Zhong,
  {Heng}~{Ji},
  {Nan}~{Jiang}, and Tong Zhang.
\newblock Iterative preference learning from human feedback: Bridging theory
  and practice for {RLHF} under {KL}-constraint.
\newblock In \emph{Proceedings of the 41st International Conference on Machine
  Learning}, volume 235 of \emph{Proceedings of Machine Learning Research}.
  PMLR, 2024.

\bibitem[Yang et~al.(2020)Yang, Jin, Wang, Wang, and Jordan]{yang2020provably}
Zhuoran Yang, Chi Jin, Zhaoran Wang, Mengdi Wang, and Michael~I. Jordan.
\newblock Provably efficient reinforcement learning with kernel and neural
  function approximations.
\newblock In \emph{Advances in Neural Information Processing Systems},
  volume~33, 2020.

\bibitem[Yarotsky(2017)]{yarotsky2017error}
Dmitry Yarotsky.
\newblock Error bounds for approximations with deep {ReLU} networks.
\newblock \emph{Neural Networks}, 94:\penalty0 103--114, 2017.

\bibitem[Zhao et~al.(2025)Zhao, Ye, Xiong, Gu, and Zhang]{zhao2025logarithmic}
Heyang Zhao, Chenlu Ye, Wei Xiong, Quanquan Gu, and Tong Zhang.
\newblock Logarithmic regret for online {KL}-regularized reinforcement
  learning.
\newblock In \emph{Proceedings of the 42nd International Conference on Machine
  Learning}, volume 267 of \emph{Proceedings of Machine Learning Research},
  pp.\  77864--77884. PMLR, 2025.

\bibitem[Ziegler et~al.(2019)Ziegler, Stiennon, Wu, Brown, Radford, Amodei,
  Christiano, and Irving]{ziegler2019fine}
Daniel~M Ziegler, Nisan Stiennon, Jeffrey Wu, Tom~B Brown, Alec Radford, Dario
  Amodei, Paul Christiano, and Geoffrey Irving.
\newblock Fine-tuning language models from human preferences.
\newblock \emph{arXiv preprint arXiv:1909.08593}, 2019.

\end{thebibliography}
\end{document}